\documentclass{article}

\PassOptionsToPackage{numbers, compress}{natbib}
\usepackage[preprint]{neurips_2026}

\makeatletter
\renewcommand{\@toptitlebar}{%
  \vspace{-20pt}
  \par\noindent
  \begin{minipage}{\textwidth}
    \includegraphics[height=1.0cm]{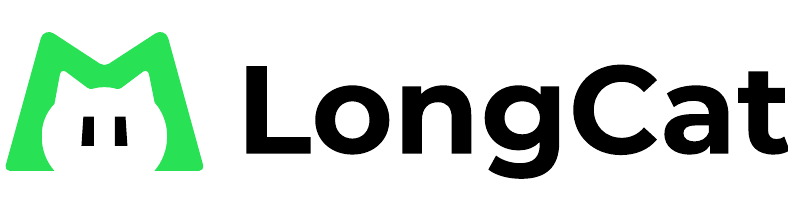}%
    \hfill
    \includegraphics[height=1.5cm]{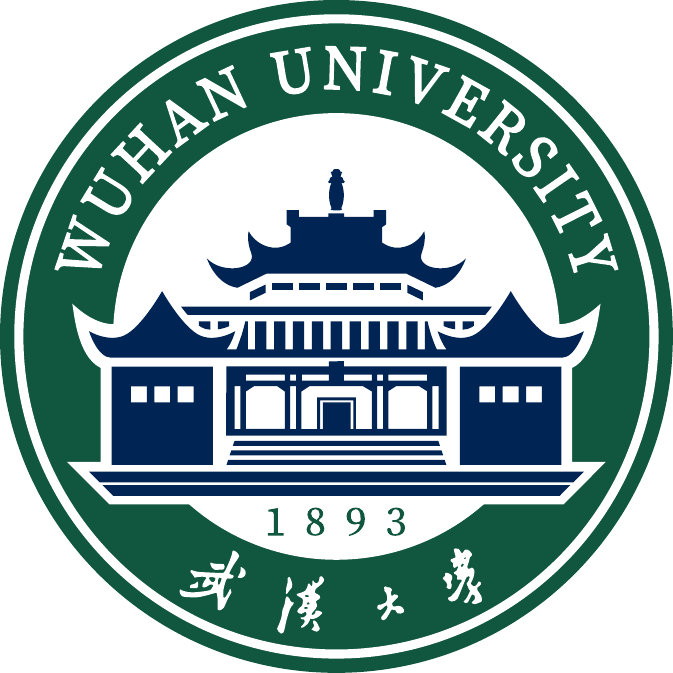}%
  \end{minipage}
  \par\vspace{5pt}
  \hrule height 1pt
  \vskip .25in
}
\makeatother

\usepackage[colorlinks=true, linkcolor=blue, citecolor=blue]{hyperref}
\usepackage[utf8]{inputenc} 
\usepackage[T1]{fontenc}    
\usepackage{hyperref}       
\usepackage{url}            
\usepackage{booktabs}       
\usepackage{amsfonts}       
\usepackage{nicefrac}       
\usepackage{microtype}      
\usepackage{xcolor}         
\usepackage {soul}
\usepackage {color, xcolor}
\definecolor{forestgreen}{RGB}{79,173,91}
\definecolor{orange}{RGB}{238,205,180}
\definecolor{purple}{RGB}{208,196,221}
\usepackage{inconsolata}
\usepackage{multirow}
\usepackage{color}
\usepackage{xcolor}
\usepackage{booktabs}
\usepackage{graphicx}
\usepackage{float}
\usepackage{colortbl}
\usepackage{url}
\usepackage{arydshln}
\usepackage{pifont}
\usepackage{enumitem}
\usepackage{wrapfig}
\usepackage{amssymb}
\usepackage{mathtools}
\usepackage{amsthm}
\usepackage{microtype}
\usepackage{subfigure}
\usepackage{booktabs}
\usepackage{tcolorbox}
\usepackage{tabularx}
\usepackage{array}
\usepackage{amsmath} 
\usepackage{algorithm} 
\usepackage{algpseudocode} 

\newtheorem{observation}{Observation}
\newtheorem{theorem}{Theorem}
\newtheorem{lemma}{Lemma}
\newtheorem{proposition}{Proposition}

\newtheorem{assumption}{Assumption}

\title{When to Stop Reusing: Dynamic Gradient Gating for Sample-Efficient RLVR}

\author{
  \textbf{Yuchun Miao$^{1,2}$}\thanks{Work done during an internship at Meituan Longcat Team.},\:\:
  \textbf{Sen Zhang$^{3}$,}\:\:
  \textbf{Yuqi Zhang$^{1}$,}\:\:
  \textbf{Yaorui Shi$^{4}$,}\:\:
  \\
  \textbf{Qi Gu$^{2\mspace{2mu}\dagger}$,}\:\:
  \textbf{Xunliang Cai$^{2}$,}\:\:
  \textbf{Lefei Zhang$^{1}$}\thanks{Correspondence to Qi Gu <guqi03@meituan.com> and Lefei Zhang <zhanglefei@whu.edu.cn>} \\
  \vspace{-2mm} \\
  \fontsize{8.6pt}{11pt}\selectfont $^{1}$National Engineering Research Center for Multimedia Software, School of Computer Science, Wuhan University, \\
  \fontsize{8.6pt}{11pt}\selectfont $^{2}$Meituan Longcat Team,
  \fontsize{8.6pt}{11pt}\selectfont $^{3}$The University of Sydney,
  \fontsize{8.6pt}{11pt}\selectfont $^{4}$University of Science and Technology of China
}

\begin{document}

\maketitle

\begin{abstract}

Reinforcement Learning with Verifiable Rewards (RLVR) has become the dominant paradigm for advanced reasoning in Large Language Models (LLMs), but rollout samples are expensive to obtain, making sample efficiency a critical bottleneck. A natural remedy is to reuse each rollout batch for multiple gradient updates, a standard practice in classical RL. Yet in RLVR, this amplifies policy shift, leading to severe performance degradation. Detecting the onset of degradation early enough to stop reuse remains an open and challenging problem. We close this gap by identifying the \textit{Disproportionate Weight Divergence (DWD)} phenomenon: performance degradation is synchronized with a sharp surge in the \texttt{lm\_head} weight change, while intermediate layers remain stable. Empirically, we verify that DWD emerges consistently across diverse LLMs and tasks. Theoretically, we prove that (i) harmful gradients concentrate at the \texttt{lm\_head} while intermediate layers are structurally attenuated, and (ii) the \texttt{lm\_head} gradient norm lower-bounds the policy divergence. These results establish the \texttt{lm\_head} gradient norm as a principled, real-time signal of catastrophic policy shift. Guided by this insight, we propose \textit{Dynamic Gradient Gating (DGG)}, a lightweight intervention that monitors the \texttt{lm\_head} gradient norm in real time and intercepts harmful gradients before they corrupt the optimizer. DGG consistently matches or exceeds the standard single-use baseline, achieving up to $2.93\times$ sample efficiency and $2.14\times$ wall-clock speedup across math, ALFWorld, WebShop, and search-augmented QA tasks.

\end{abstract}

\section{Introduction}
\label{sec:introduction}
Reinforcement Learning with Verifiable Rewards (RLVR) has become the key driver for unlocking advanced reasoning in Large Language Models (LLMs), powering recent breakthroughs in long-horizon mathematics, code, and agentic tasks~\cite{shao2024deepseekmath, guo2025deepseek, yu2025dapo}. However, this powerful paradigm introduces a severe computational bottleneck~\cite{wu2025llamarl, team2025introducing}. In modern RLVR pipelines, rollout generation dominates the training budget, often consuming over 80\% of post-training GPU hours~\cite{arnal2026efficient}. As models scale and tasks demand longer horizons, improving the sample efficiency of rollout utilization has emerged as a central prerequisite for sustainable training.

A natural approach to improving sample efficiency is \textit{sample reuse}, which performs multiple gradient updates on each fresh rollout batch. This technique is well-established in policy-gradient RL~\cite{schulman2017proximal, schulman2015trust} and has driven substantial efficiency gains in classical domains.  Unfortunately, as extensively observed in recent RLVR systems, naive sample reuse exacerbates the policy shift between the updating policy and the behavior policy, triggering catastrophic \textit{training 
collapse}~\cite{zheng2025group, yang2025sspo, zheng2026prosperity,liang2026squeeze,baroian2026prompt}—manifesting as severe \textit{performance degradation}. Consequently, production pipelines have overwhelmingly reverted to a conservative \textit{single-use rollout regime}, where each rollout is strictly discarded after one update, severely underutilizing each sample despite its steep rollout cost. Motivated by this challenge, we aim to investigate the following research question: 

\begin{tcolorbox}[
  colback=gray!3,
  colframe=black,
  arc=4pt, left=5pt, right=5pt, top=2pt, bottom=2pt,
  boxrule=0.6pt,fontupper=\itshape
]
\textit{How can we reliably detect the underlying signals of policy shift in real time—before it triggers training collapse—to unlock the full efficiency gains of sample reuse?}
\end{tcolorbox}



To answer this question, we begin with an empirical investigation into the microscopic dynamics of training collapse induced by sample reuse. Our profiling reveals a structural asymmetry that we term the \textit{Disproportionate Weight Divergence (DWD)} phenomenon: \textbf{performance degradation is synchronized with a sharp surge in the \texttt{lm\_head} weight change, while all intermediate layers remain stable}. This observation isolates the \texttt{lm\_head} as the structural locus of training collapse, providing a precise target for detecting early-warning signals. This raises a natural theoretical question: \textit{why is this weight divergence localized to the \texttt{lm\_head}, and why does its surge faithfully track the catastrophic policy shift that drives training collapse?}

Two structural properties of the per-layer RL gradient jointly explain the DWD phenomenon. First, we show that \textbf{harmful gradients under sample reuse concentrate at the \texttt{lm\_head}}, while intermediate layers are shielded by downstream Jacobian projections (Theorem~\ref{thm:asymmetry}); the resulting gradient ratio is bounded by an architectural constant empirically below $10^{-1}$ at the median across diverse widely-used LLMs (Table~\ref{tab:empirical_constants}). Second, we prove that \textbf{the \texttt{lm\_head} gradient norm lower-bounds the empirical Pearson $\chi^2$ divergence} between the updating and behavior policies (Theorem~\ref{thm:chi2_bound}), establishing its surge as a provable indicator of catastrophic policy shift. Together, these results provide a complete theoretical foundation of the DWD phenomenon.

Beyond explaining why training collapse occurs, our joint empirical and 
theoretical analyses establish the \texttt{lm\_head} gradient norm as a 
mechanistically faithful, real-time indicator of catastrophic policy shift. 
Guided by this insight, we propose \textit{Dynamic Gradient Gating (DGG)}, 
a lightweight intervention that monitors this localized signal via a 
dynamic Z-score test. Once an anomalous surge is detected, DGG 
\textbf{discards the harmful gradient} before it enters the Adam optimizer 
and \textbf{terminates the reuse loop}, preventing both the model weights 
and optimizer state from being corrupted.

We demonstrate the widespread presence of the DWD phenomenon across \textbf{six LLMs} from Qwen~\cite{yang2025qwen3} and Llama~\cite{touvron2023llama} families (1.5B--8B parameters) and \textbf{four diverse task}, covering mathematical reasoning and agentic tasks (WebShop, ALFWorld, and search-augmented QA). Across \textbf{ten benchmarks}, DGG matches or exceeds the strict single-use baseline, achieving up to \textbf{2.93$\times$} sample efficiency and \textbf{2.14$\times$} wall-clock speedup. These results confirm that the efficiency potential of sample reuse can be fully unlocked with a principled, lightweight intervention. Our main contributions are as follows:

$\bullet$ \textbf{Empirical Discovery.} We identify the DWD phenomenon, where sample reuse-induced performance degradation is structurally synchronized with a localized weight divergence at the \texttt{lm\_head}.

$\bullet$ \textbf{Theoretical Analysis.} We theoretically account for DWD by proving that harmful gradients under sample reuse concentrate at the \texttt{lm\_head}, and that its gradient norm lower-bounds policy divergence.

$\bullet$ \textbf{Algorithm.} We propose DGG, a lightweight mechanism that monitors the \texttt{lm\_head} gradient norm via a Z-score test and intercepts harmful gradients before they corrupt the optimizer.

$\bullet$ \textbf{Experiments.} We validate the broad presence of DWD across diverse LLMs and tasks, and demonstrate that DGG achieves up to 2.93$\times$ sample efficiency without sacrificing final performance.

\section{Related Work}
\textbf{Reinforcement Learning with Verifiable Rewards (RLVR).} RLVR is the dominant paradigm for unlocking LLM reasoning~\cite{guo2025deepseek,yu2025dapo,shao2024deepseekmath}, yet its rollout stage severely dominates training time~\cite{zheng2026act,zheng2026prosperity,zhang2026prorl}. Concurrently, scaling rollout exploration remains highly effective for improving performance across diverse tasks~\cite{li2025treepo,xing2026lookahead,dong2026agentic,cao2025skyrl}. Given that these rollouts are both \textit{expensive and highly valuable}, improving the sample efficiency of rollout utilization—specifically, fully extracting the latent learning value from each generated sample—has emerged as an urgent bottleneck for RLVR scalability.

\textbf{Sample-Efficient RLVR.} Existing efforts to improve sample efficiency target the \textit{rollout stage} by reshaping rollout distributions via experience replay~\cite{arnal2026efficient,sun2026improving,wang2025eframe,zhan2026exgrpo} or adaptive allocation based on prompt difficulty~\cite{liao2025enhancing,wang2025eframe,li2025knapsack} and variance signals~\cite{zheng2026act,zhang2026train,nguyen2026adaptive}. Crucially, \textit{both paradigms still operate strictly within the single-use rollout regime}, performing only one gradient update per batch.

\textbf{Our Focus.} In contrast to the aforementioned works targeting the \textit{rollout stage}, we pioneer an orthogonal perspective at the \textit{update stage}: given each batch of rollouts, how can we safely perform multiple gradient updates on it to unlock its full learning potential? The central challenge lies in \textit{detecting catastrophic policy shift before it escalates into training collapse}, which remains underexplored in existing sample-efficient methods.

\section{Preliminaries: GRPO and Sample Reuse}
\label{sec:preliminaries}
Group Relative Policy Optimization (GRPO)~\cite{shao2024deepseekmath} optimizes the policy $\pi_\theta$ by maximizing a clipped surrogate objective over $T$ active tokens:
\begin{equation}
\mathcal{L}_{GRPO}(\theta) 
= \frac{1}{T} \sum_{i=1}^{T} 
\min \! \left( 
r_i(\theta) \hat{A}_i,\; 
\text{clip}\bigl(r_i(\theta),\, 1-\epsilon_c,\, 1+\epsilon_c\bigr) 
\hat{A}_i 
\right),
\label{eqn:raw_obj}
\end{equation}
where $\hat{A}_i$ is the group-normalized advantage and $r_i(\theta) \triangleq \frac{\pi_\theta(a_i \mid \cdot)}{\pi_{old}(a_i \mid \cdot)}$ is the importance ratio against the behavior policy $\pi_{old}$ for the sampled token $a_i$. Crucially, for tokens in the unclipped region, the raw gradient propagating into the LLM is driven by:
\begin{equation}
\nabla_\theta \mathcal{L}_{GRPO}(\theta) 
= \frac{1}{T} \sum_{i=1}^{T} 
r_i(\theta) \, \hat{A}_i \, 
\nabla_\theta \log \pi_\theta(a_i \mid \cdot).
\label{eqn:raw_grad}
\end{equation}
This formulation highlights that the raw gradient signal is multiplicatively scaled by $r_i(\theta)$.

\textbf{The Vulnerability of Sample Reuse in LLMs.}
The gradient formulation in Eq.~\eqref{eqn:raw_grad} reveals why sample reuse, a cornerstone of classical RL~\cite{schulman2017proximal}, breaks down for LLMs. The vast action space and long reasoning horizons force the behavior policy $\pi_{old}$ to assign \textit{vanishingly small probabilities} to numerous tail tokens. Since $r_i(\theta)$ uses these tiny probabilities as denominators, successive reuse steps create a feedback loop in which modest policy shifts compound rapidly, inflating $r_i(\theta)$ by orders of magnitude~\cite{chen2025minimax, zheng2025group, wang2025numerical}. Because $r_i(\theta)$ enters Eq.~\eqref{eqn:raw_grad} as a multiplicative scalar, this inflation disproportionately magnifies the gradients of these tokens, \textbf{letting sparse tail tokens dominate the update and drive the model to collapse}~\cite{yang2026do, liu2026stapo}. The standard clipping mechanism in Eq.~\eqref{eqn:raw_obj} does not prevent this. Its protection is asymmetric: tokens with negative advantages and highly inflated ratios ($r_i(\theta) > 1+\epsilon_c$) remain unclipped~\cite{ye2020mastering,wang2025aspo,yu2025dapo}, so anomalously large gradients propagate freely into the model.

While this vulnerability is widely observed~\cite{zheng2026prosperity, liang2026squeeze, baroian2026prompt}, its underlying mechanism remains unclear. We investigate it empirically and theoretically in the next section.

\begin{figure*}[]
    \centering\scriptsize\renewcommand\arraystretch{0.5}
    \setlength{\tabcolsep}{1pt}
    \begin{tabular}{@{}c@{\hspace{2.5pt}}ccc}
    &\textbf{Task Performance} & \textbf{Weight Change (w/o Reuse)} & \textbf{Weight Change (w/ Naive Reuse)} \\
    \noalign{\smallskip}
    \raisebox{1.0cm}{\rotatebox{90}{\textbf{Math}}} &
    \includegraphics[width=0.32\linewidth]{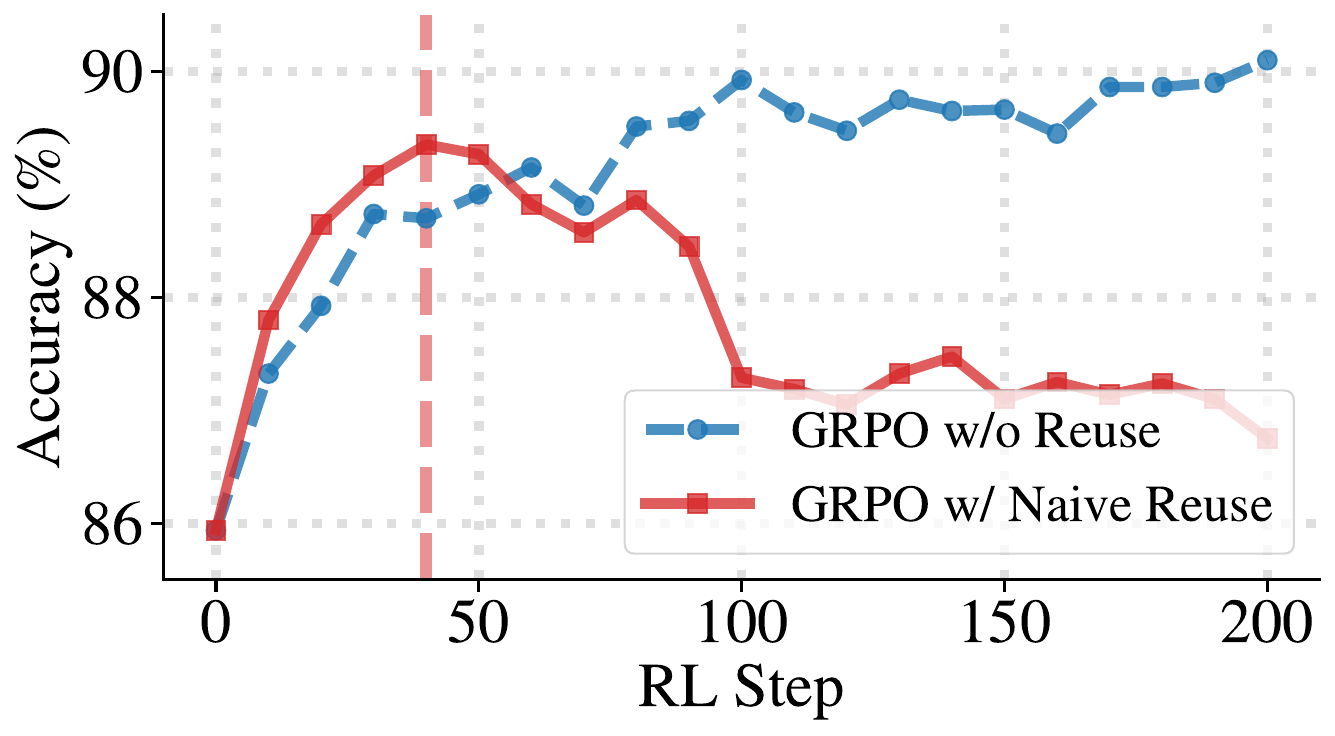}&
    \includegraphics[width=0.32\linewidth]{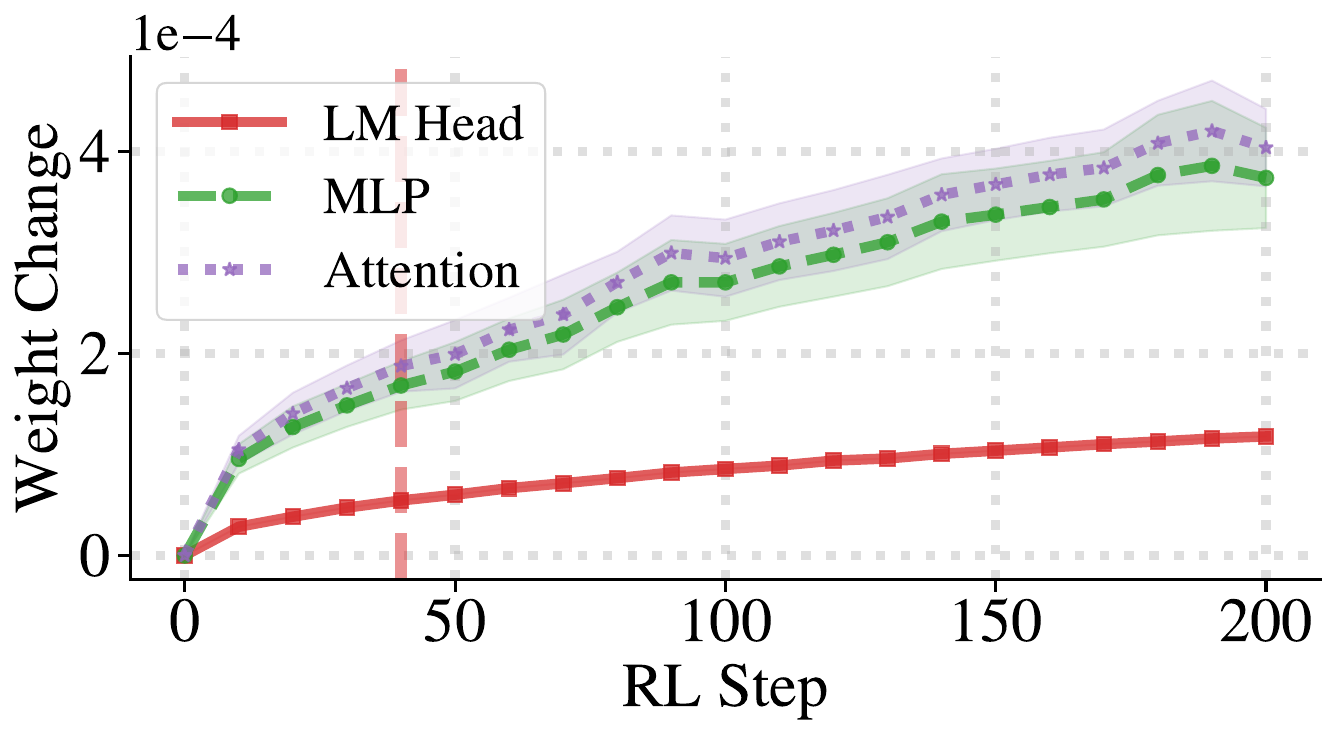}&
    \includegraphics[width=0.32\linewidth]{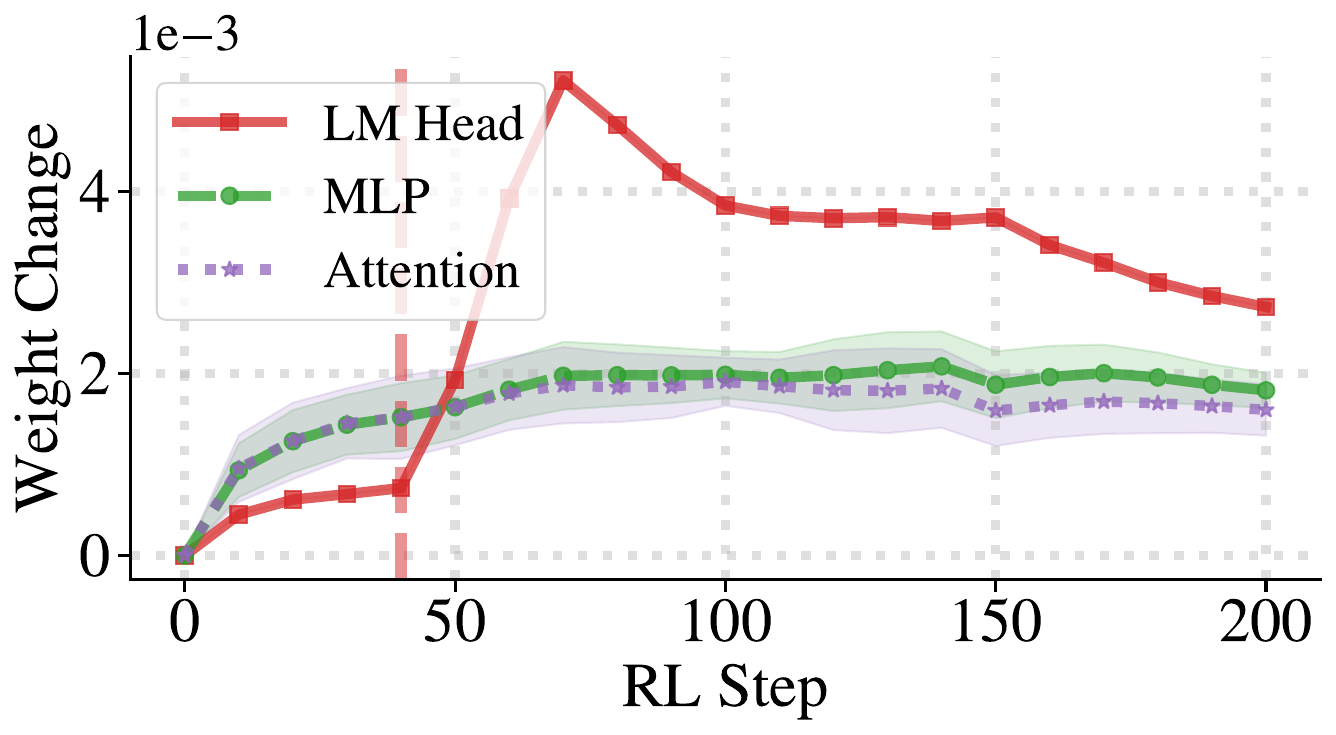}\\[-3pt]
    \raisebox{1.0cm}{\rotatebox{90}{\textbf{ALFWorld}}}&
    \includegraphics[width=0.32\linewidth]{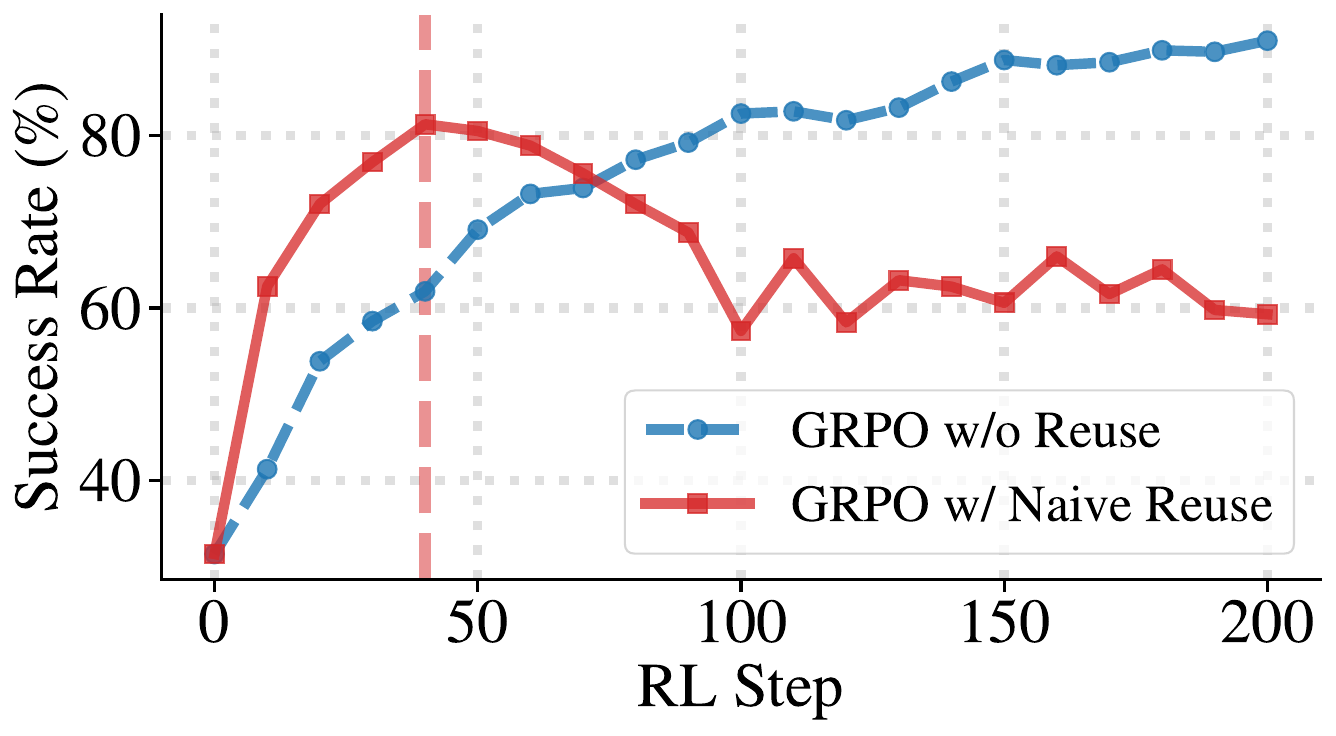}&
    \includegraphics[width=0.32\linewidth]{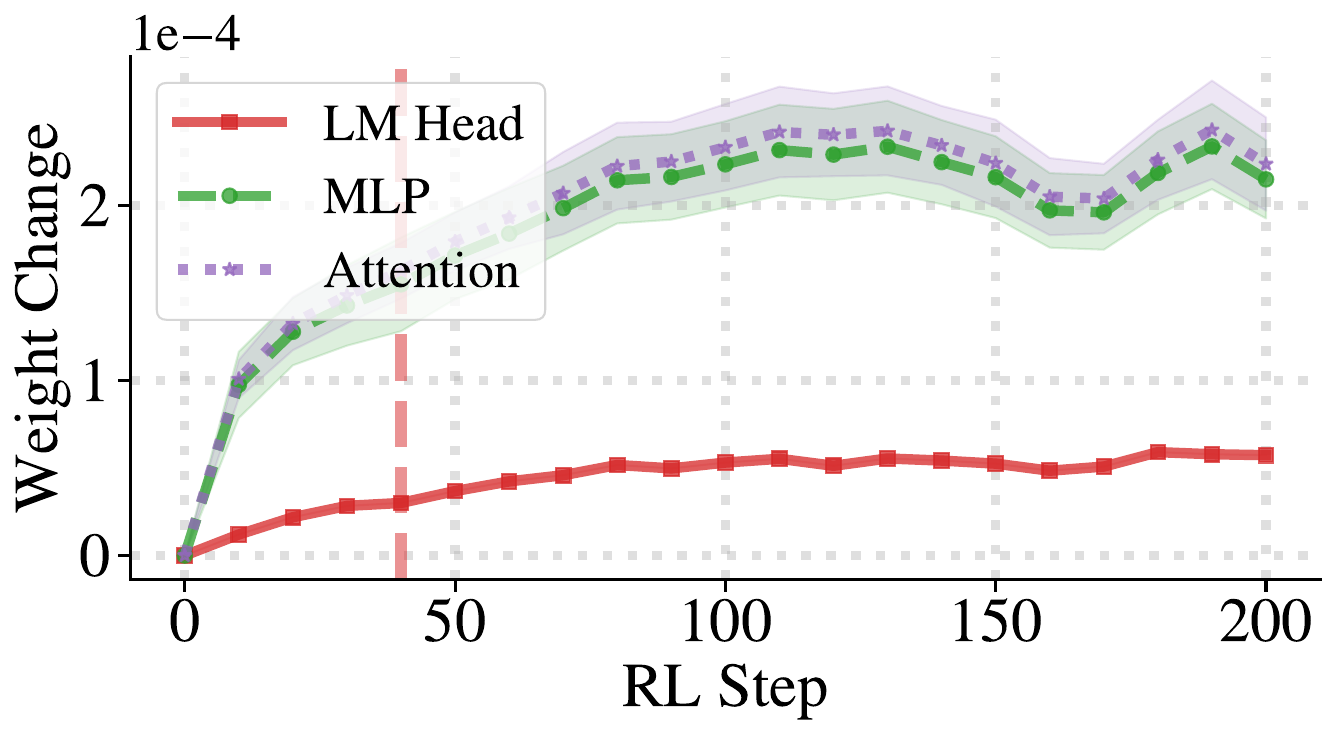}&
    \includegraphics[width=0.32\linewidth]{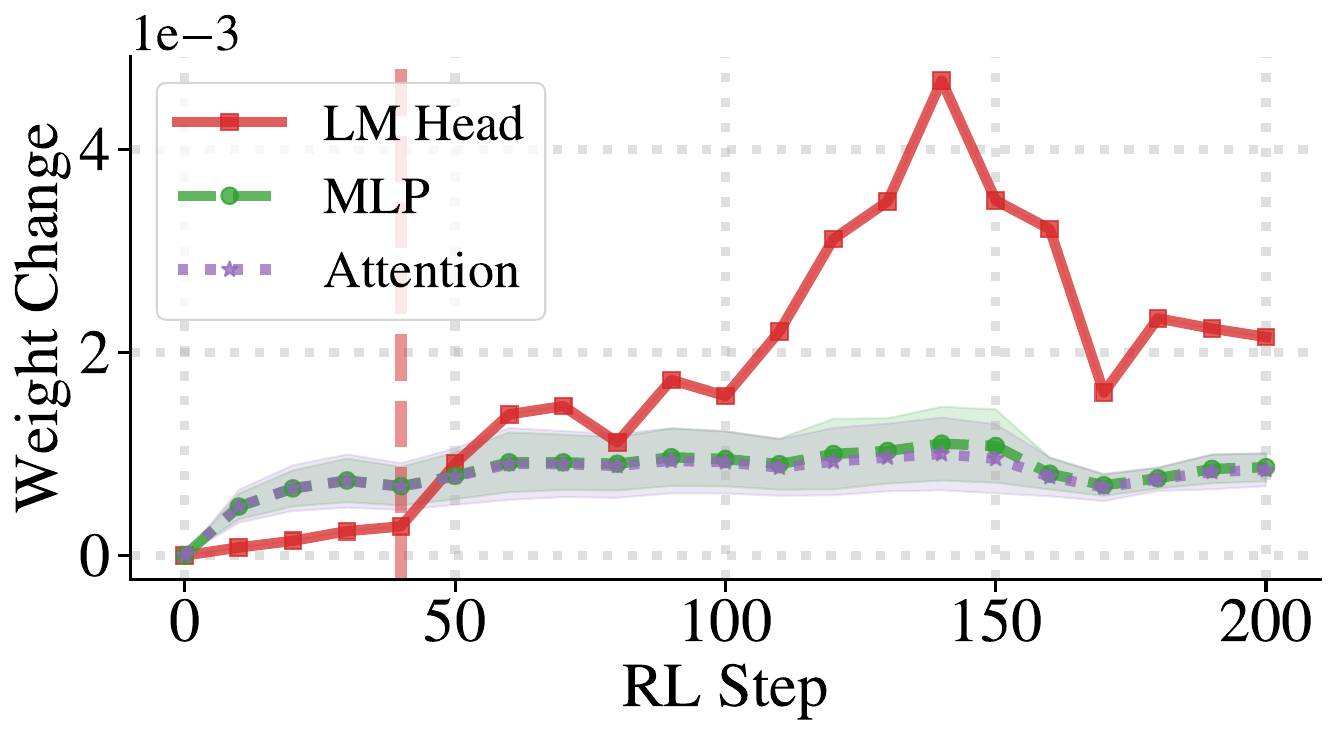}\\[-3pt]
    \raisebox{1.0cm}{\rotatebox{90}{\textbf{WebShop}}}&
    \includegraphics[width=0.32\linewidth]{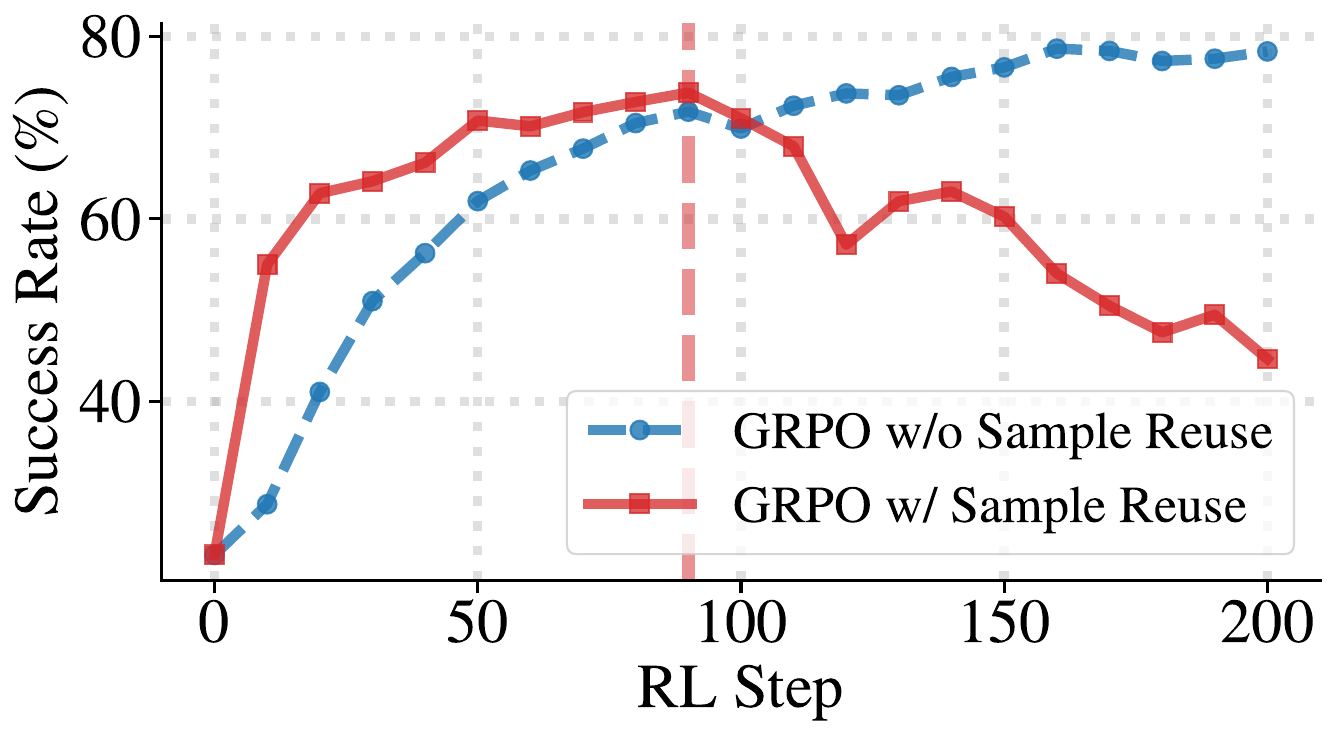}&
    \includegraphics[width=0.32\linewidth]{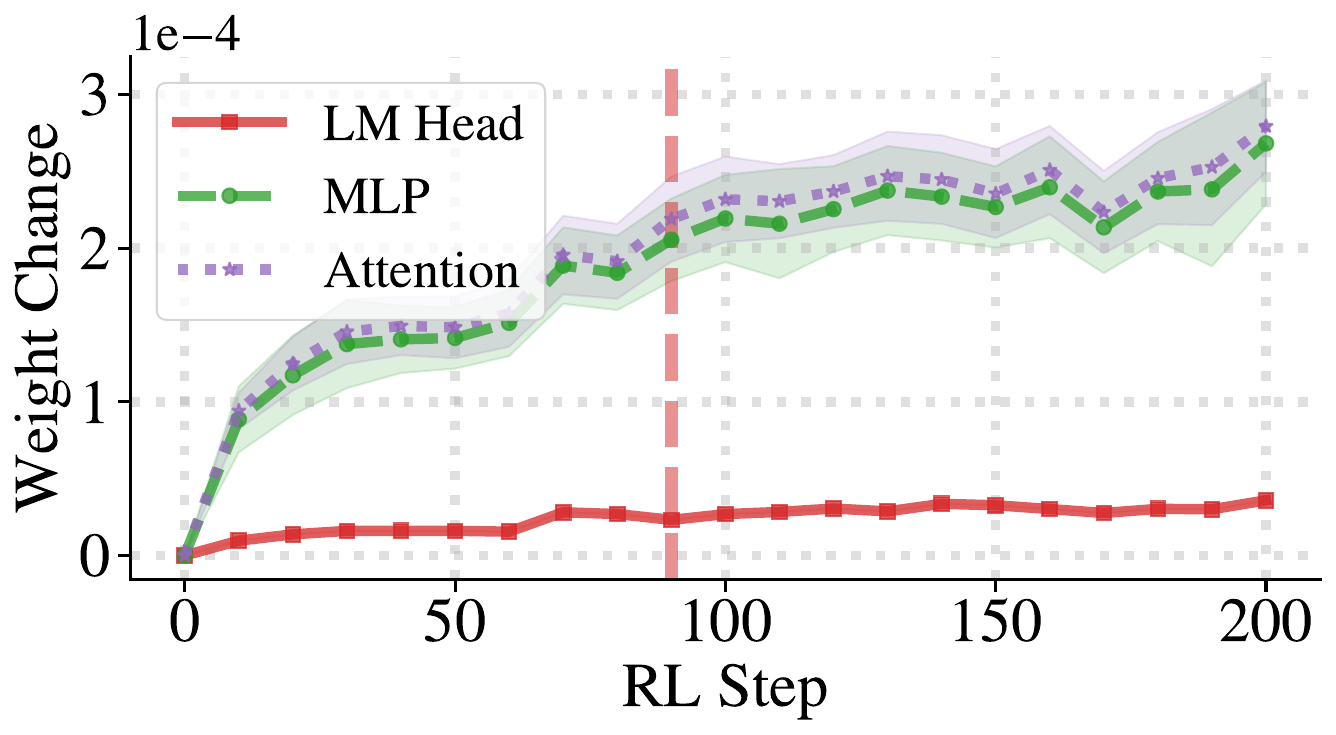}&
    \includegraphics[width=0.32\linewidth]{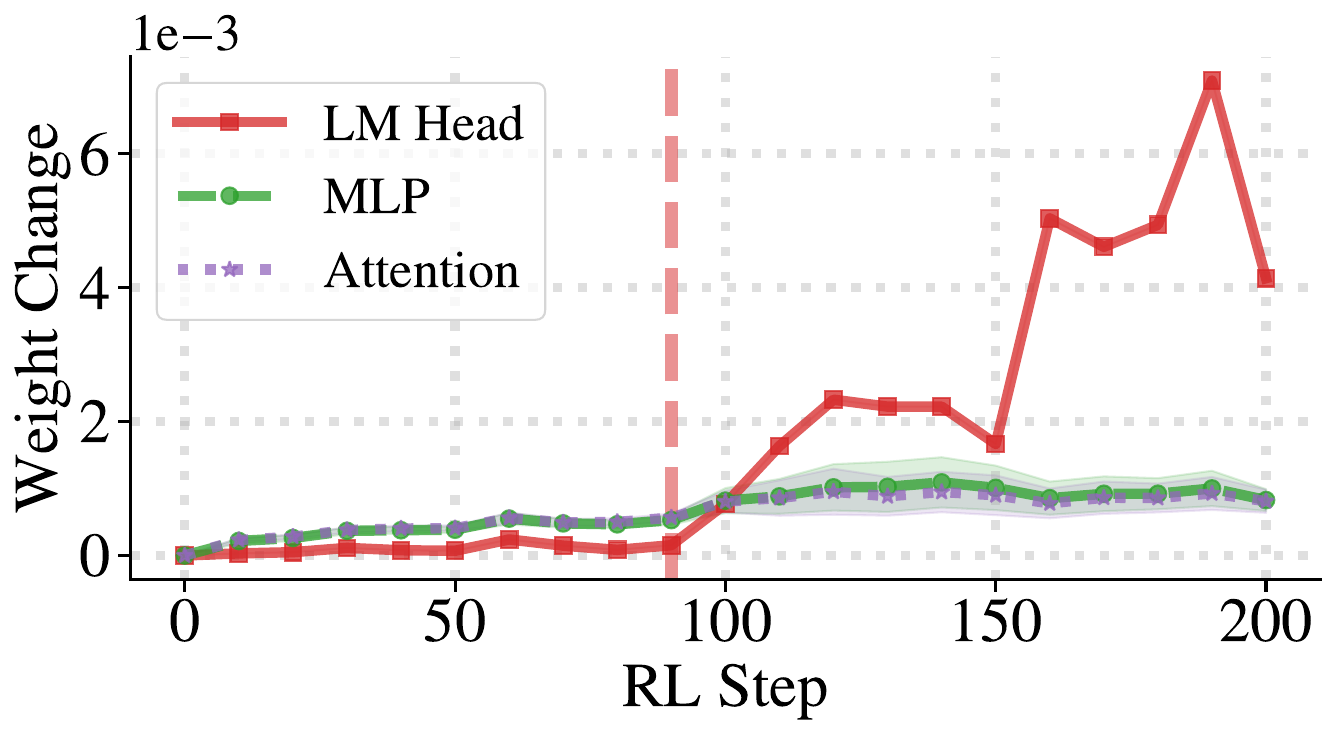}\\[-3pt]
    \raisebox{1.0cm}{\rotatebox{90}{\textbf{Search QA}}}&
    \includegraphics[width=0.32\linewidth]{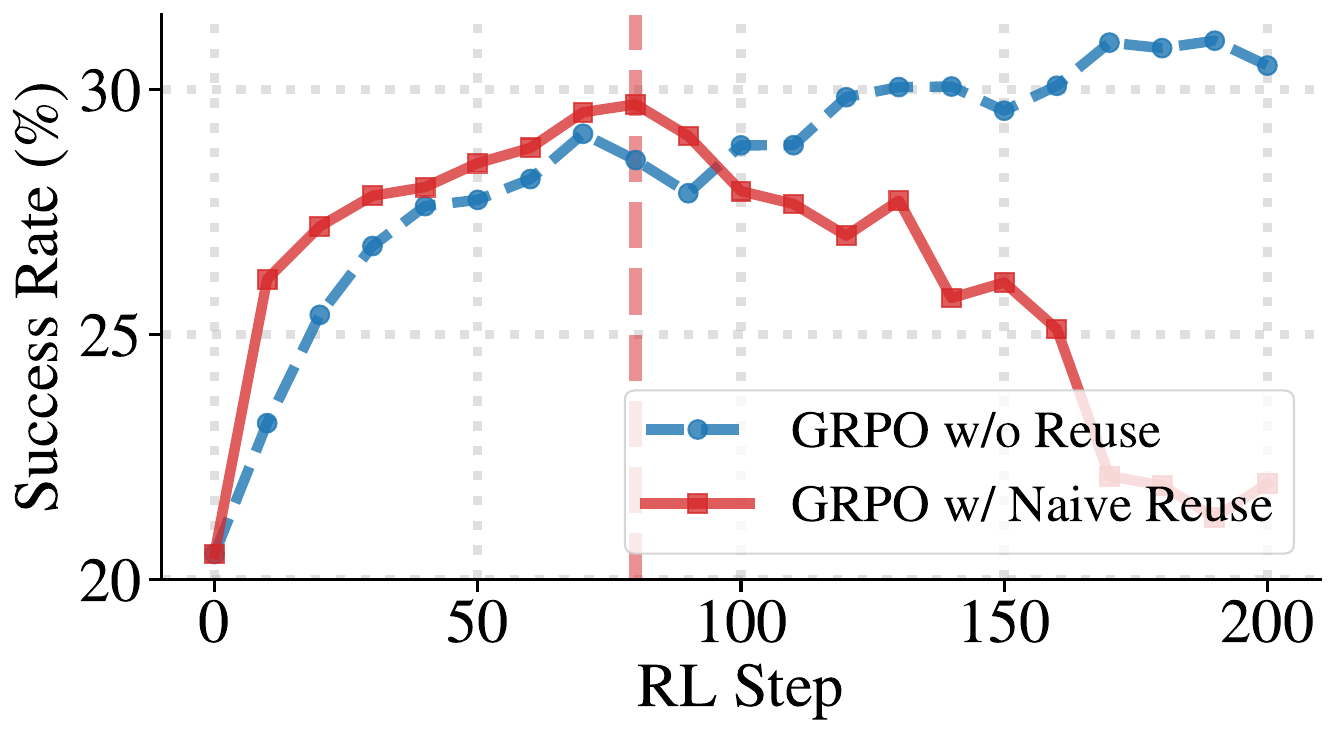}&
    \includegraphics[width=0.32\linewidth]{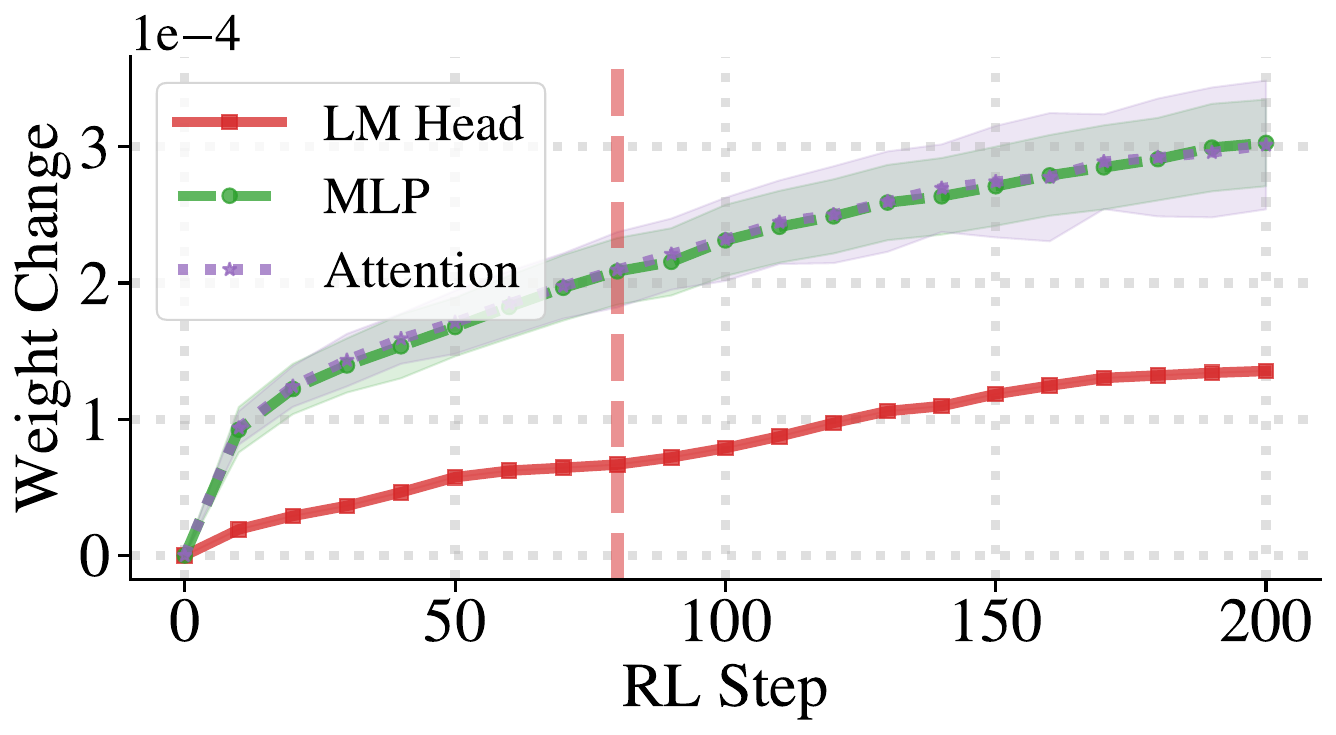}&
    \includegraphics[width=0.32\linewidth]{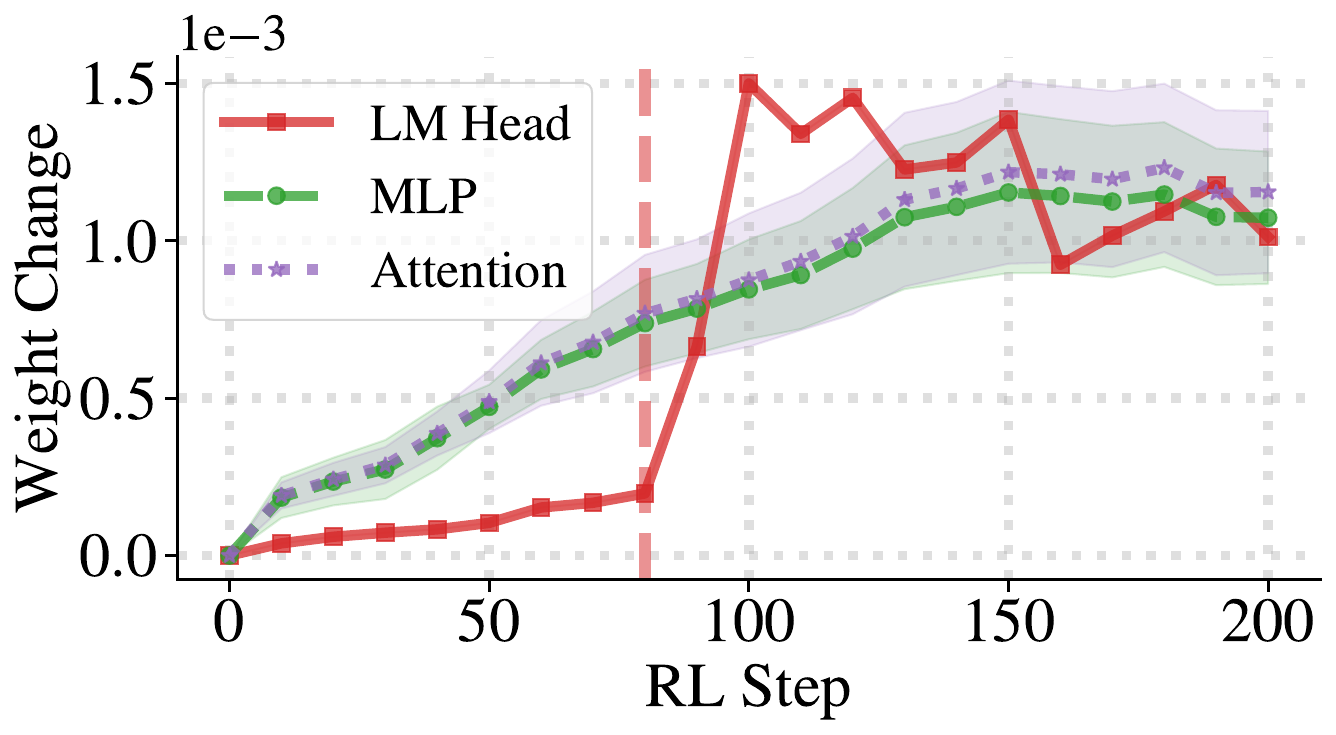}\\[-7pt]
    \end{tabular}
    \caption{\textbf{Illustration of the Disproportionate Weight Divergence (DWD) phenomenon} on Qwen3-4B-Instruct across various tasks. Relative weight change is defined as $\|W_t - W_{t-\Delta}\|_F / \|W_{\text{ref}}\|_F$ ($W_{\text{ref}}$: initial pretrained weight; $\Delta$: profiling interval, set to 10 RL steps); red dashed lines mark the onset of performance degradation for GRPO w/ Naive Reuse. Observation: \textit{Sample reuse accelerates early convergence, followed by a severe performance degradation synchronized with an abrupt, isolated surge in the LM Head, while intermediate layers (Attention and MLP) remain stable.} Consistent results across the other five LLMs are provided in Appendix~\ref{app:more_empirical_dwd}.}
    \label{fig:empirical_dwd_qwen3-4b-instruct}
\end{figure*}

\section{Disproportionate Weight Divergence (DWD)}
\label{sec:dwd}
Having established in Section~\ref{sec:preliminaries} why sample reuse destabilizes RLVR training, we now seek an indicator signal inside the network that reliably reflects this instability in real time. Section~\ref{subsec:dwd_empirical} identifies such a signal empirically through microscopic profiling—the Disproportionate Weight Divergence (DWD) phenomenon—and Section~\ref{subsec:dwd_theory} theoretically explains why DWD arises.

\subsection{Empirical Observation of the DWD Phenomenon}
\label{subsec:dwd_empirical}

\textbf{Profiling Setup.} 
To diagnose how sample-reuse-induced collapse manifests inside the network, we profile training dynamics across six LLMs from the Llama and Qwen families (1.5B--8B parameters) and four diverse tasks: mathematical reasoning, WebShop, ALFWorld, and search-augmented QA. We compare GRPO w/o Reuse (strictly single-use) against GRPO w/ Naive Reuse (a fixed number of gradient updates per rollout batch), tracking two signals: (i) {task performance}, and (ii) {relative weight change} of each layer, grouped by component type ($\texttt{lm\_head}$, $\texttt{attn}$, $\texttt{mlp}$).

\textbf{Core Finding.} 
Figure~\ref{fig:empirical_dwd_qwen3-4b-instruct} shows representative results. Without sample reuse, performance improves steadily and weight changes remain smooth across all layers. With sample reuse, convergence is substantially faster at first but soon gives way to a severe performance degradation, accompanied by a structurally asymmetric pattern in the layer-wise weight dynamics, which we formalize as follows.
\begin{tcolorbox}[colback=gray!3, colframe=black, arc=4pt, left=5pt, right=5pt, top=2pt, bottom=2pt, boxrule=0.6pt]
\begin{observation}[Disproportionate Weight Divergence]
\label{obs:dwd}
During sample-reuse RLVR training, performance degradation is synchronized with a sharp surge in the \texttt{lm\_head}'s weight change, while intermediate layers (attention and MLP) remain stable.
\end{observation}
\end{tcolorbox}


This pattern holds across all profiled settings (see Appendix~\ref{app:more_empirical_dwd}). Appendix~\ref{app:full_weight_change} further shows that without sample reuse, the \texttt{lm\_head} surge does not emerge even over the full training trajectory, confirming that sample reuse itself is the direct cause, rather than the increased number of updates.

\subsection{Theoretical Analysis of the DWD Phenomenon}
\label{subsec:dwd_theory}
The empirical findings in Section~\ref{subsec:dwd_empirical} raise two questions: \textit{why is the gradient anomaly structurally localized to the \texttt{lm\_head}, and why does its surge faithfully track the catastrophic policy shift that drives training collapse?} We answer these two questions via a \textit{Structural Gradient Asymmetry} (Theorem~\ref{thm:asymmetry}, Section~\ref{subsubsec:lm_head_sensitivity}) and a \textit{Divergence Lower Bound} (Theorem~\ref{thm:chi2_bound}, Section~\ref{subsubsec:chi2_divergence}), respectively.

\textbf{Setup and notation.}
For any active token position $i$, let $h_{L,i} \in \mathbb{R}^{d_{\text{model}}}$ denote the post-norm hidden state entering the \texttt{lm\_head}, and $x_i^{\text{int}} \in \mathbb{R}^{d_{\text{model}}}$ the input to an arbitrary intermediate linear layer. Let $a_i$ be the sampled token, $\pi_\theta(\cdot \mid h_{L,i})$ and $\pi_{\text{old}}$ the current and behavior policies, $r_i \triangleq \pi_\theta(a_i)/\pi_{\text{old}}(a_i)$ the importance sampling ratio, and $\hat{A}_i$ the group-normalized advantage. Let $J_i $ denote the composite Jacobian mapping the intermediate linear layer's output to the final logits. We measure gradient magnitude using the Frobenius norm $\|M\|_F$, and policy divergence using the Pearson $\chi^2$.

\subsubsection{Structural Gradient Asymmetry between \texttt{lm\_head} and Intermediate Layers}
\label{subsubsec:lm_head_sensitivity}
Our analysis proceeds in three steps: (1) decomposing the per-layer gradients to isolate the error signal, (2) establishing regularity conditions on the forward and backward signals, and (3) deriving a quantitative bound on the inter-layer gradient attenuation. We then verify this bound empirically across multiple LLM architectures. The following proposition formalizes the gradient decomposition.

\begin{proposition}[Gradient Decomposition]
\label{prop:grad_decomp}
Let $\mathcal{L}_i$ denote the objective for token $a_i$, and define the error signal as $E_i \triangleq r_i \hat{A}_i ( e_{a_i} - \pi_\theta(\ \cdot \ | h_{L,i}) ) \in \mathbb{R}^{d_{\text{vocab}}}$. The gradients of $\mathcal{L}_i$ with respect to the \texttt{lm\_head} weight $W_{\text{lm}}$ and any intermediate linear layer weight $W_{\text{int}}$ admit the following rank-1 forms:
\begin{equation}
G_i^{\text{lm}} \triangleq \nabla_{W_{\text{lm}}} \mathcal{L}_i = E_i \, h_{L,i}^\top, \qquad
G_i^{\text{int}} \triangleq \nabla_{W_{\text{int}}} \mathcal{L}_i = \bigl( J_i^\top E_i \bigr) \bigl( x_i^{\text{int}} \bigr)^\top,
\end{equation}
where $J_i = \partial z_i / \partial y_i $ is the Jacobian of logits $z_i$ with respect to the intermediate output $y_i = W_{\text{int}} x_i^{\text{int}}$.
\end{proposition}
Proposition~\ref{prop:grad_decomp} (proof in Appendix~\ref{app:proof_grad_decomp}) shows that the \texttt{lm\_head} absorbs the error signal $E_i$ directly, whereas every intermediate layer is shielded by the Jacobian projection $J_i^\top$. Since $\|uv^\top\|_F^2 = \|u\|_2^2 \cdot \|v\|_2^2$ for any rank-1 matrix, this structure cleanly separates the \emph{error signal} ($E_i$ or $J_i^\top E_i$) from the \emph{representation signal} ($h_{L,i}$ or $x_i^{\text{int}}$), allowing each factor to be bounded independently.

We first bound the magnitude of the representation signals feeding the respective layers.

\begin{lemma}[Bounded Activations under Pre-Norm Architectures]
\label{lem:activation_bound}
For any active token $i$, let $h_{L,i}$ denote the \texttt{lm\_head} input and $x_i^{\text{int}}$ denote the intermediate-layer input. There exist strictly positive constants $\alpha_{\min}$ and $\beta_{\max}$ such that:
\vspace{-2pt}
\begin{equation}
\| h_{L,i} \|_2^2 \;\geq\; \alpha_{\min} \cdot d_{\text{model}}, \qquad
\| x_i^{\text{int}} \|_2^2 \;\leq\; \beta_{\max} \cdot d_{\text{model}}.
\end{equation}
\end{lemma}

The upper bound $\beta_{\max}$ holds uniformly across all intermediate linear layers in standard Pre-Norm Transformers, covering both layers fed directly by RMSNorm (e.g., QKV and FFN up-projections) and those fed by internal sub-network outputs (e.g., attention output and FFN down-projections). See Appendix~\ref{app:activation_bounds} for the layer-by-layer derivation.

We next bound the Jacobian-projected error signal in the backward pass.

\begin{assumption}[Bounded Logit Sensitivity]
\label{ass:jacobian}
There exists a strictly positive constant $C$, such that for any token position $i$:
\begin{enumerate}[label=\textit{(\roman*)}, leftmargin=*, topsep=-2pt, itemsep=2pt, parsep=0pt]
    \item $\mathbb{E}_{j \sim \pi_\theta} \| (J_i)_{j,:} \|_2^2 \leq C$, bounding the expected logit sensitivity to intermediate perturbations;
    \item $\| (J_i)_{a_i,:} \|_2^2 \leq C$, extending the same bound to the sampled token $a_i$.
\end{enumerate}
\end{assumption}

This bounded-Jacobian property is well-studied in deep learning theory, both as a sufficient condition for generalization~\citep{sokolic2017robust} and as an explicit regularization target~\citep{yoshida2017spectral}. We further empirically verify that this assumption holds across diverse LLMs, and that $\mathcal{C}_{\text{struct}}$ defined in Theorem~\ref{thm:asymmetry} as a whole remains stable throughout RL training, including at the onset of training collapse (Appendix~\ref{app:empirical_table}).

Combining the gradient decomposition (Proposition~\ref{prop:grad_decomp}), forward activation bounds (Lemma~\ref{lem:activation_bound}), and Jacobian regularity (Assumption~\ref{ass:jacobian}), we formally establish the gradient asymmetry between the \texttt{lm\_head} and intermediate layers, culminating in our first main result.

\begin{tcolorbox}[colback=gray!3, colframe=black, arc=4pt, left=5pt, right=5pt, top=2pt, bottom=2pt, boxrule=0.6pt]
\begin{theorem}[Structural Gradient Asymmetry]
\label{thm:asymmetry}
Under Lemma~\ref{lem:activation_bound} and Assumption~\ref{ass:jacobian}, for any sampled token $a_i$ with policy probability $\pi_\theta(a_i) < 1$ and non-zero advantage, the ratio of the Frobenius gradient energy between any intermediate layer and the \texttt{lm\_head} is bounded by:\vspace{-2pt}
\begin{equation}
\frac{ \| G_i^{\text{int}} \|_{F}^2 }{ \| G_i^{\text{lm}} \|_{F}^2 } \;\leq\; \mathcal{C}_{\text{struct}} \cdot \frac{1}{(1-\pi_\theta(a_i))^2},
\label{eq:main_bound_F}
\end{equation}
where $\mathcal{C}_{\text{struct}} \triangleq \frac{4 \beta_{\max} C}{\alpha_{\min}}$ is a strictly positive architectural constant.
\end{theorem}
\end{tcolorbox}

Theorem~\ref{thm:asymmetry} factorizes this asymmetry into two multiplicative terms: the architectural constant $\mathcal{C}_{\text{struct}}$, which captures all architecture-dependent quantities and is empirically small across diverse LLMs (Section~\ref{app:empirical_table}); and the policy-confidence factor $(1-\pi_\theta(a_i))^{-2}$, governed by how concentrated $\pi_\theta$ is on $a_i$. The proof is deferred to Appendix~\ref{app:proof_main_theorem}.

\textbf{Quantitative instantiation of the bound.}
We measure all constants entering $\mathcal{C}_{\text{struct}}$ directly from forward-pass activations and Jacobian estimates across multiple LLMs spanning the Llama, Qwen, and GPT-OSS families (1.5B--235B parameters; full table in Appendix~\ref{app:empirical_table}). The factor is uniformly small across all tested architectures---below $10^{-1}$ at the median and below $0.5$ even at the 95th-percentile worst case. Moreover, $\mathcal{C}_{\text{struct}}$ remains stable throughout RL training, including at the onset of sample-reuse-induced collapse (Appendix~\ref{app:empirical_table}). \textit{LLMs therefore inherently suppress intermediate-layer gradient updates relative to the \texttt{lm\_head}.}

\textbf{Connection to the DWD phenomenon.}
Sample reuse in LLMs risks training collapse when the gradient is dominated by tail tokens, i.e., those with vanishingly small probability under $\pi_{\text{old}}$ but massively inflated importance ratios $r_i$ and large advantages $|\hat{A}_i|$ (Section~\ref{sec:preliminaries}). Theorem~\ref{thm:asymmetry} shows that the bound on intermediate-layer gradients is tightest on exactly these tokens: their low policy confidence drives $(1-\pi_\theta(a_i))^{-2}$ toward its minimum, and the architectural constant $\mathcal{C}_{\text{struct}}$ is empirically small ($<10^{-1}$ at the median). \textit{When tail tokens dominate, the gradient anomaly is therefore concentrated at the \texttt{lm\_head}, matching the asymmetry in Observation~\ref{obs:dwd}.}

\subsubsection{\texttt{lm\_head} Gradient Energy as a Lower Bound on Policy Divergence}
\label{subsubsec:chi2_divergence}
Our analysis proceeds in two steps: (1) relating the \texttt{lm\_head} gradient norm to the empirical second moment of importance ratios, and (2) connecting this moment to an empirical Pearson $\chi^2$ statistic.

Aggregating the per-token gradient $G_i^{\text{lm}}$ over $T$ active tokens, we define the batch-level \texttt{lm\_head} gradient $G^{\text{lm}} \triangleq \frac{1}{T}\sum_{i=1}^{T} G_i^{\text{lm}}$, matching the gradient propagated through the GRPO objective (Eq.~\ref{eqn:raw_grad}).

\begin{lemma}[Gradient Norm via Importance-Ratio Second Moment]
\label{lem:grad_bound}
Let $\overline{r^2} \triangleq \tfrac{1}{T}\sum_{i=1}^{T} r_i^2$ denote the empirical estimator of the Pearson $\chi^2$ divergence. Then
\begin{equation}
\|G^{\text{lm}}\|_F^2 \;\leq\; c_{\max} \, \overline{r^2}, 
\qquad \text{where} \;\; c_{\max} \triangleq \max_{i} 
\hat{A}_i^2 \, \| e_{a_i} - \pi_\theta(\cdot \mid h_{L,i}) \|_2^2 \, 
\| h_{L,i} \|_2^2.
\end{equation}
\end{lemma}

We now relate $\overline{r^2}$ to the Pearson $\chi^2$ statistic, yielding our second main result.
\begin{tcolorbox}[colback=gray!3, colframe=black, arc=4pt, left=5pt, right=5pt, top=2pt, bottom=2pt, boxrule=0.6pt]
\begin{theorem}[\texttt{lm\_head} Gradient Lower-Bounds Policy Divergence]
\label{thm:chi2_bound}
Let $\hat{\chi}^2 \triangleq \frac{1}{T}\sum_{i=1}^T (r_i^2 - 1)$ denote the empirical estimator of the Pearson $\chi^2$ divergence $\chi^2(\pi_\theta \| \pi_{\text{old}})$. Then
\begin{equation}
\|G^{\text{lm}}\|_F^2 \;\leq\; c_{\max} \bigl( 1 + \hat{\chi}^2 \bigr), \quad \text{or equivalently,} \quad \hat{\chi}^2 \geq \|G^{\text{lm}}\|_F^2 / c_{\max} - 1.
\label{eq:chi2_bound}
\end{equation}
\end{theorem}
\end{tcolorbox}

The proofs of Lemma~\ref{lem:grad_bound} and 
Theorem~\ref{thm:chi2_bound} are deferred to 
Appendix~\ref{app:proof_chi2_bound}.

\textbf{Connection to the DWD phenomenon.}
Sample reuse in LLMs aggravates the policy shift between $\pi_\theta$ and $\pi_{\text{old}}$, eventually triggering training collapse (Section~\ref{sec:preliminaries}). The DWD phenomenon (Observation~\ref{obs:dwd}) captures this collapse through the synchronized emergence of an abrupt \texttt{lm\_head} weight surge and the onset of performance degradation. Theorem~\ref{thm:chi2_bound} provides the structural reason: since the \texttt{lm\_head} gradient norm lower-bounds $\hat{\chi}^2$, its surge implies a substantial amplification of the policy shift—a known driver of sample-reuse collapse. \emph{The \texttt{lm\_head} gradient norm is therefore not merely correlated with collapse; its surge is a mathematically certified signature of catastrophic policy shift.}

\section{Dynamic Gradient Gating (DGG)}
\label{sec:dgg}
Together, the DWD phenomenon (Observation~\ref{obs:dwd}) and our theoretical analysis establish the \texttt{lm\_head} gradient norm as a structurally faithful indicator of catastrophic policy shift: its surge is structurally localized (Theorem~\ref{thm:asymmetry}) and lower-bounds the policy shift between $\pi_\theta$ and $\pi_{\text{old}}$ (Theorem~\ref{thm:chi2_bound}). This empirical-theoretical convergence motivates a simple yet principled intervention—\textit{Dynamic Gradient Gating} (DGG)—that monitors this localized signal in real time and intercepts harmful updates \textit{before} they corrupt the model weights and optimizer state.

\textbf{Online detection via Z-score test.}
Since the magnitude of $\|G^{\text{lm}}\|_F^2$ varies across models, tasks, and training stages, we adopt the classical Z-score test, which adaptively standardizes against its recent statistics. Let $g_t \triangleq \|G_t^{\text{lm}}\|_F^2$ denote the \texttt{lm\_head} gradient energy at step $t$, and $\Delta g_t \triangleq g_t - g_{t-1}$ its step-wise increment. Since DWD manifests as an abrupt surge rather than sustained high magnitude, we monitor $\Delta g_t$, isolating fast changes from slow drift. Maintaining running estimates of the mean $\mu_t$ and standard deviation $\sigma_t$ of $\Delta g_t$, we compute the instantaneous Z-score:
\begin{equation}
z_t \;\triangleq\; \frac{\Delta g_t - \mu_t}{\sigma_t + \varepsilon}.
\label{eq:zscore}
\end{equation}
 A step is flagged as anomalous whenever $z_t > \tau$, where $\tau$ is a tunable threshold.

\textbf{Pre-optimizer gradient interception.}
Once $z_t > \tau$, DGG triggers a two-stage intervention: \textit{(i) Gradient Discard.} We zero out the gradient \textit{before} Adam, preventing corruption of its moment estimates that would persist for many steps. \textit{(ii) Reuse Termination.} We terminate the reuse loop and draw a fresh batch, breaking the feedback loop where a drifted $\pi_\theta$ keep inflating importance ratios. This pre-optimizer design stops the harm at its source (pseudocode in Appendix~\ref{app:algorithm}; implementation details, including running statistics formulation and hyperparameter settings, in Appendix~\ref{app:our_details}).

\section{Experiments}
\label{sec:experiments}

\subsection{Experimental Setup}
\textbf{Benchmarks.} 
We validate DGG on mathematical reasoning and agentic tasks. For mathematical reasoning, we use MATH500~\cite{lightman2024lets}, AIME25~\cite{aime25}, Minerva Math~\cite{NEURIPS2022_18abbeef}, and Olympiad Bench~\cite{he-etal-2024-olympiadbench}, reporting accuracy at sampling size 16 (mean@16). For agentic tasks---where higher rollout costs make DGG's benefits most pronounced---we consider ALFWorld~\cite{shridhar2021alfworld}, WebShop~\cite{NEURIPS2022_82ad13ec}, and search-augmented QA on HotpotQA~\cite{yang-etal-2018-hotpotqa}, 2Wiki~\cite{ho-etal-2020-constructing}, MuSiQue~\cite{trivedi2022musique}, and Bamboogle~\cite{press-etal-2023-measuring}.

\textbf{Training details.} 
We use Qwen3-4B-Instruct-2507 and Qwen2.5-7B-Instruct~\cite{yang2025qwen3} as our base models. We set the maximum reuse $K = 4$ and select the anomaly threshold $\tau$ from the candidate set $\{0.1, 0.5, 1.0\}$. For RL training, the rollout group size is $16$ for mathematical reasoning tasks and $8$ for agentic tasks.  Full configurations are deferred to Appendix~\ref{app:more_details}.

\begin{figure}[t]
    \centering\scriptsize\renewcommand\arraystretch{0.9}
    \setlength{\tabcolsep}{0pt}
	\begin{tabular}{c}
	\includegraphics[width=0.6\linewidth]{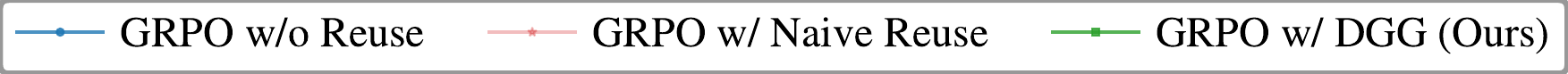}\\
	\end{tabular}
    \begin{tabular}{@{}cccc@{}}
    \includegraphics[width=0.25\textwidth]{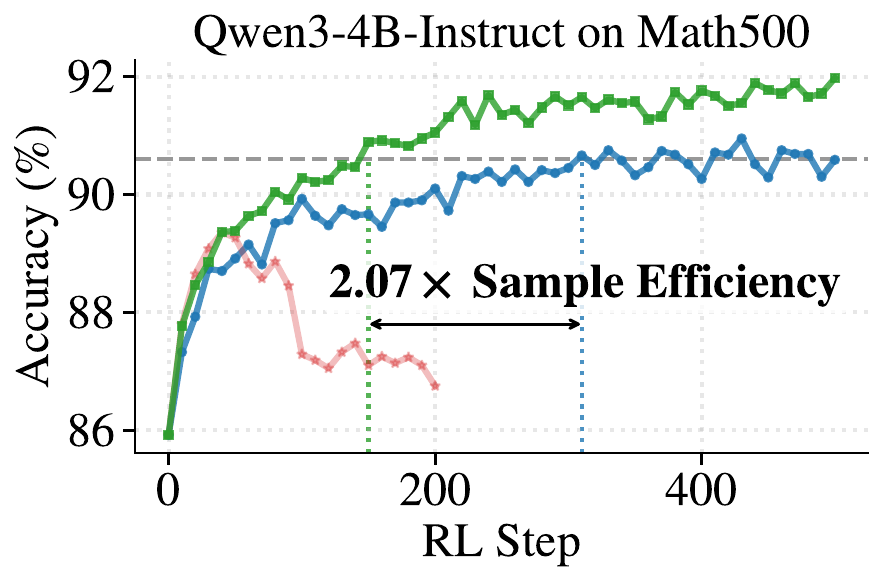}&
    \includegraphics[width=0.25\textwidth]{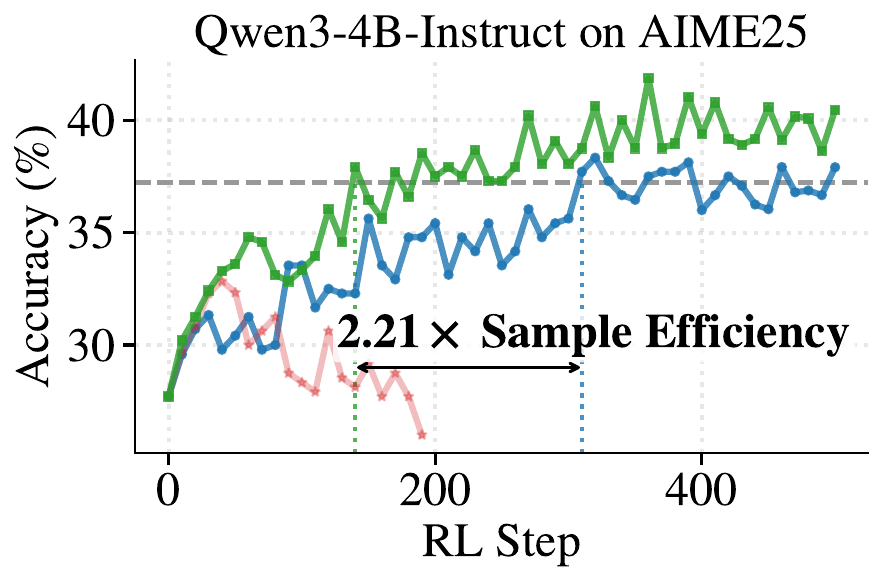}&
    \includegraphics[width=0.25\textwidth]{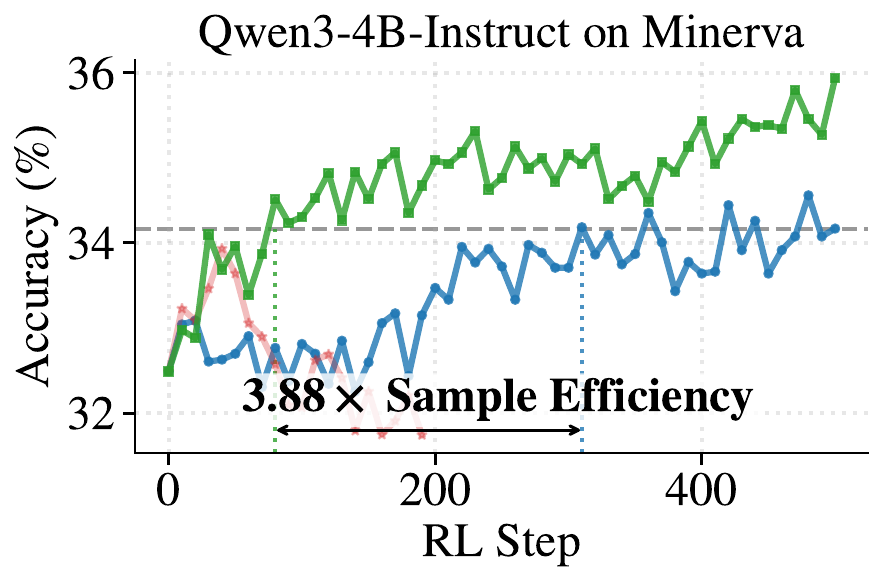}&
    \includegraphics[width=0.25\textwidth]{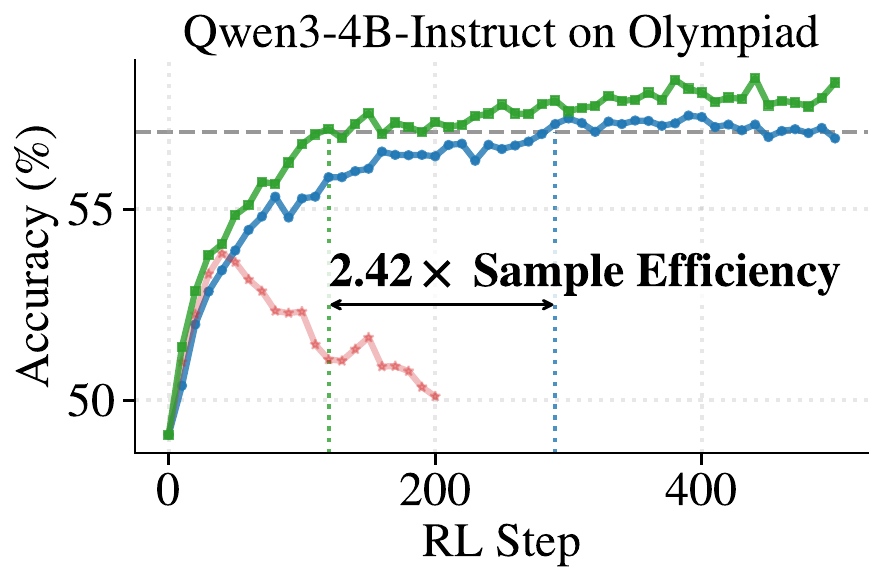}\\[-2pt]
    \includegraphics[width=0.25\textwidth]{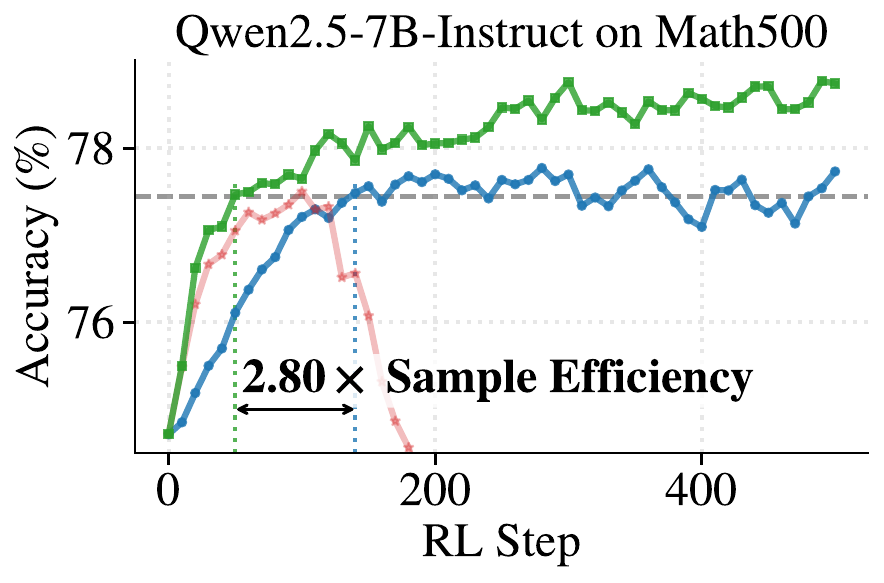}&
    \includegraphics[width=0.25\textwidth]{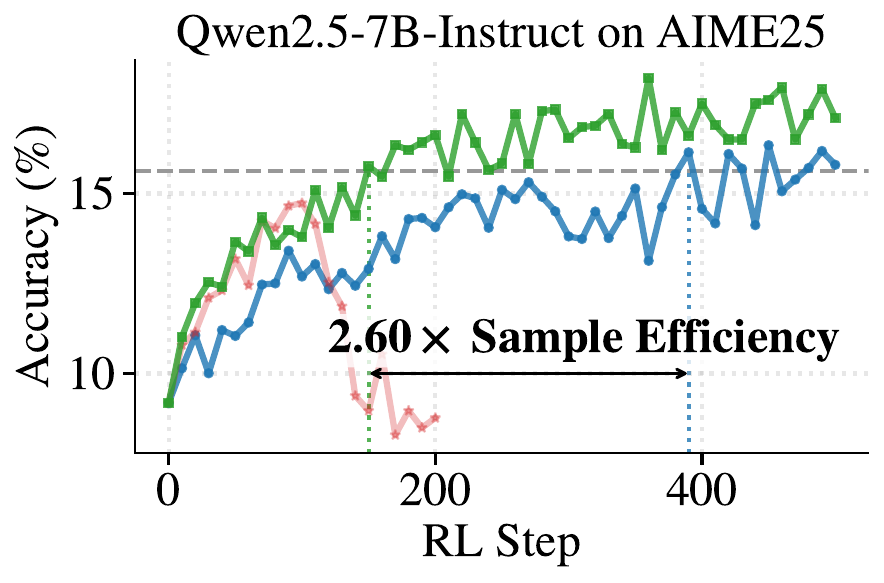}&
    \includegraphics[width=0.25\textwidth]{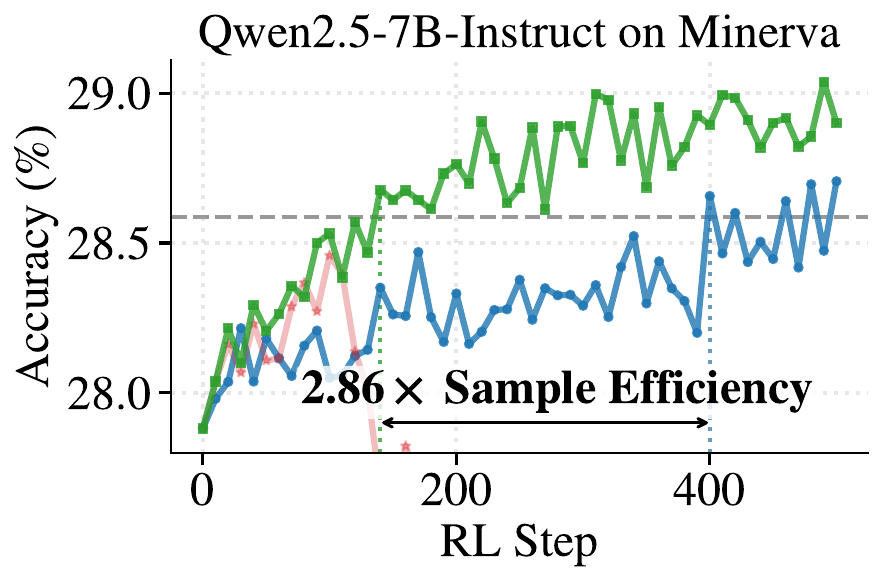}&
    \includegraphics[width=0.25\textwidth]{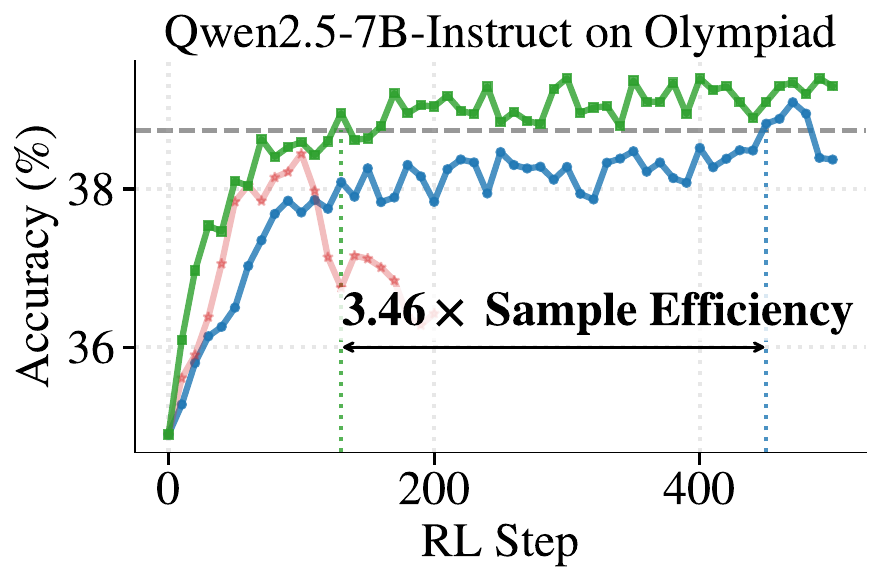}\\[-1pt]
    \multicolumn{4}{c}{\ \ \ \ \ \ \ \textbf{(a) Mathematical Reasoning}}\\[4pt]
    \includegraphics[width=0.25\textwidth]{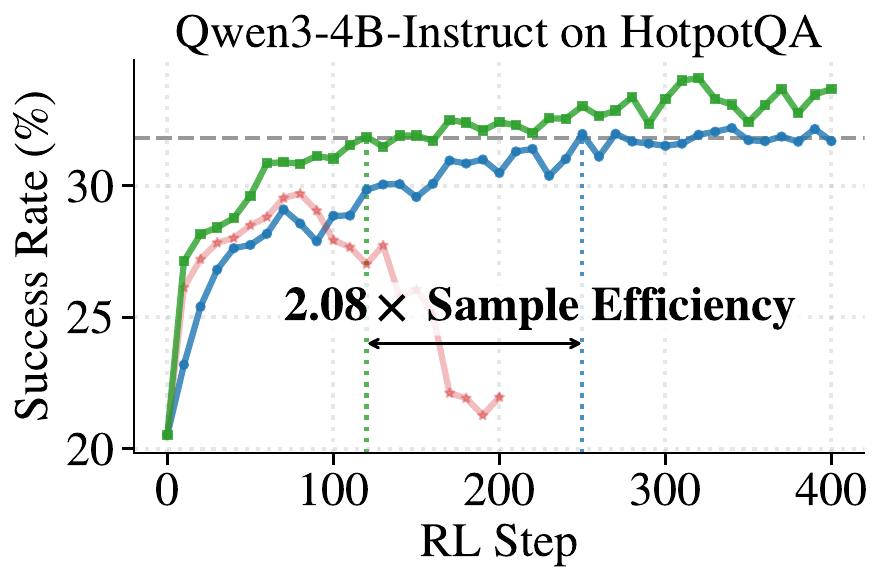}&
    \includegraphics[width=0.25\textwidth]{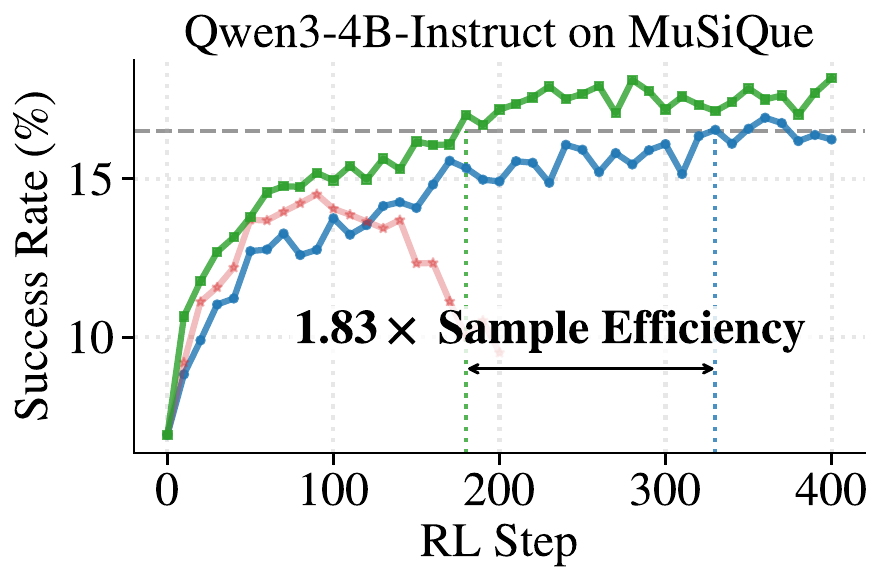}&
    \includegraphics[width=0.25\textwidth]{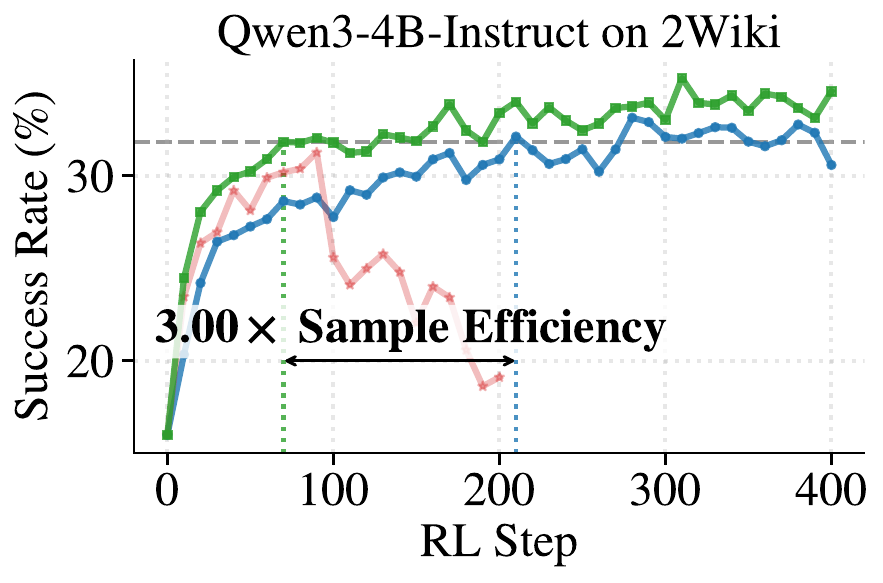}&
    \includegraphics[width=0.25\textwidth]{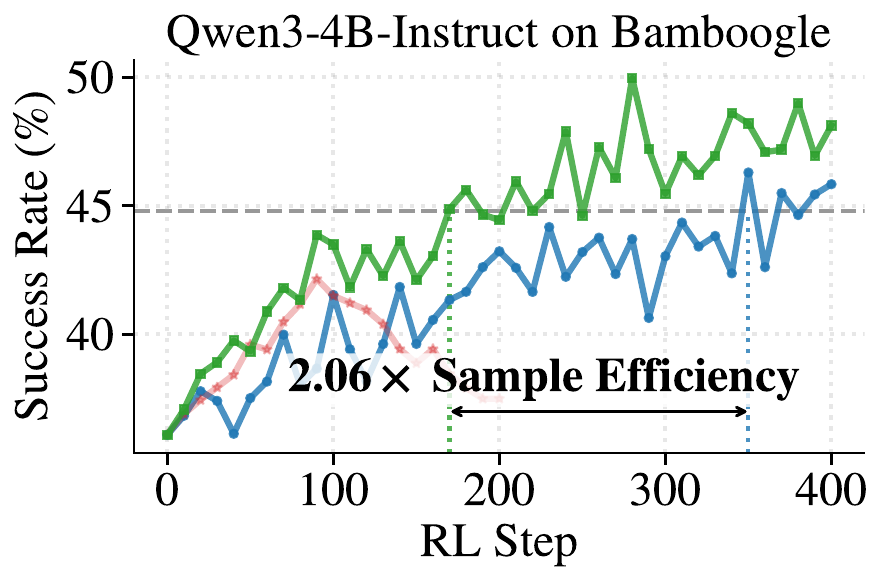}\\[-2pt]
    \includegraphics[width=0.25\textwidth]{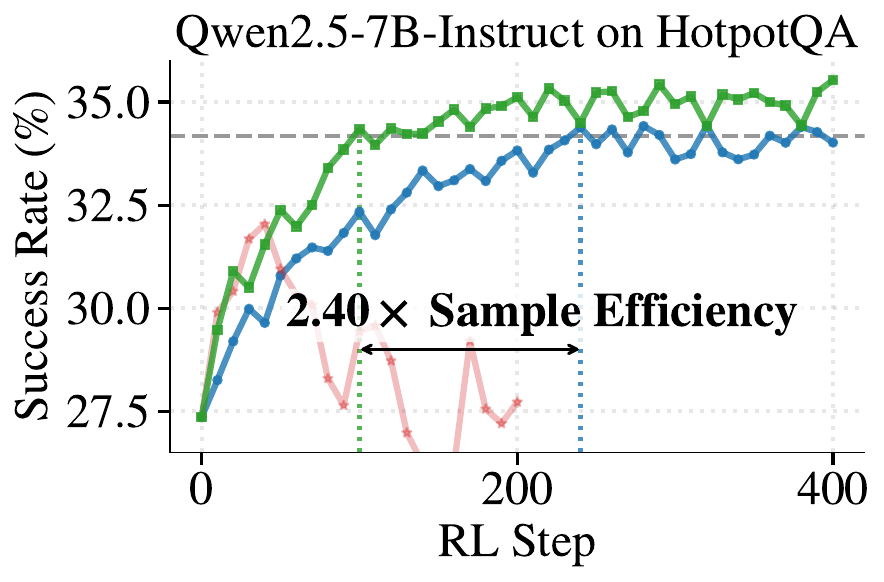}&
    \includegraphics[width=0.25\textwidth]{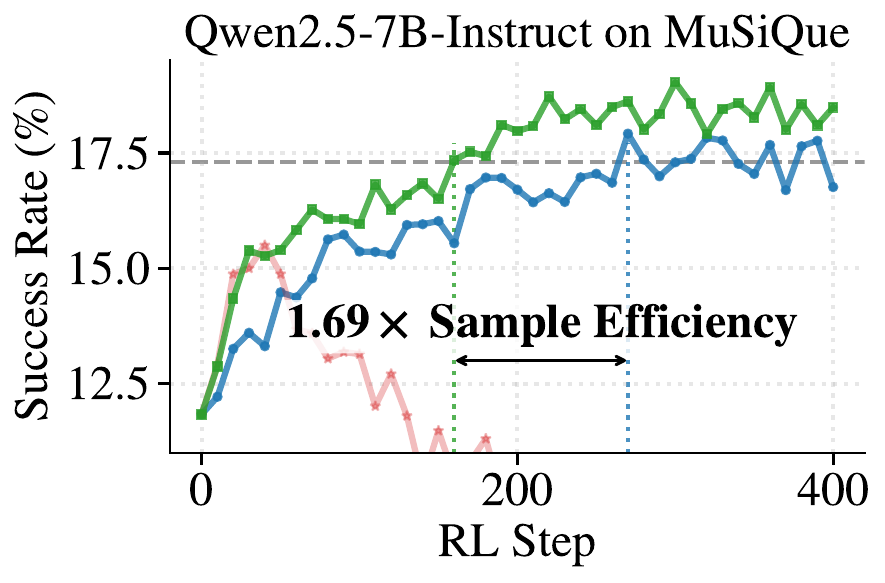}&
    \includegraphics[width=0.25\textwidth]{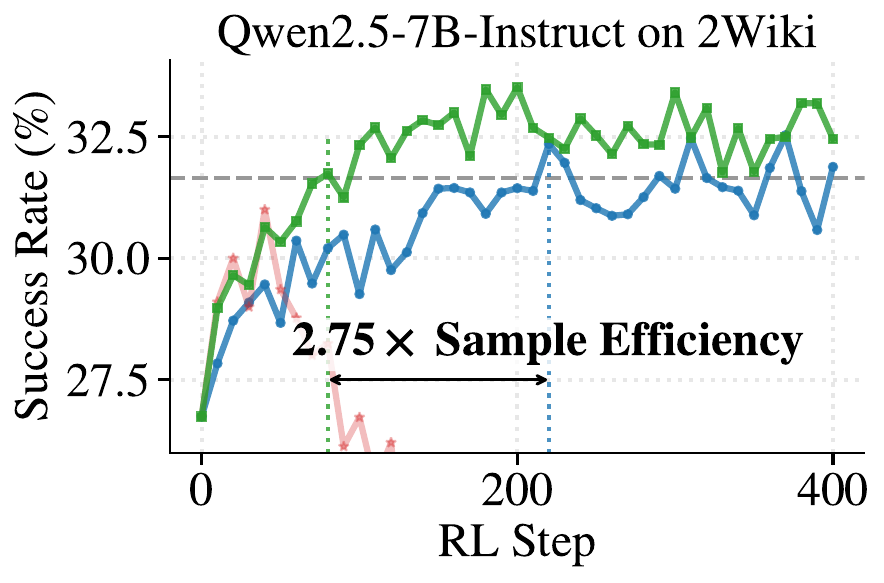}&
    \includegraphics[width=0.25\textwidth]{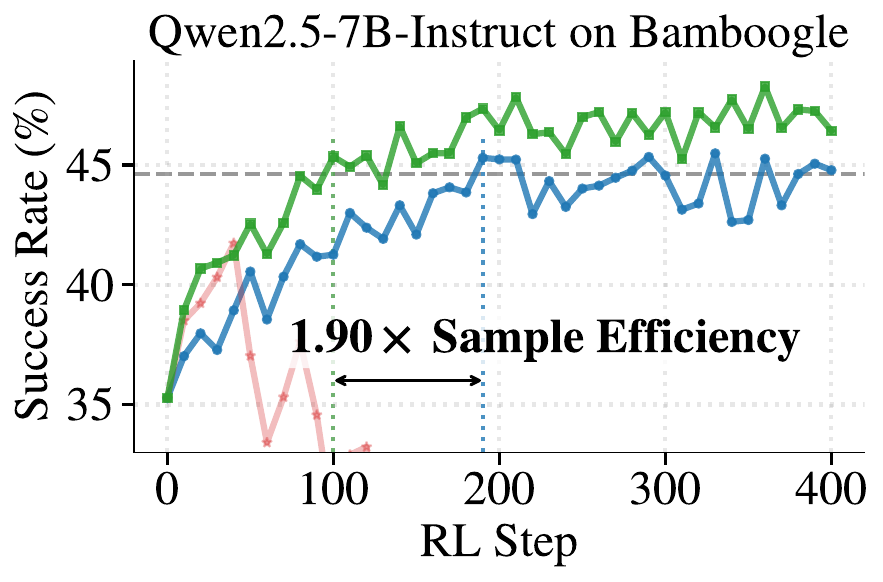}\\[-1pt]
    \multicolumn{4}{c}{\ \ \ \ \ \ \ \textbf{(b)  Search-Augmented QA}}\\[4pt]
    \includegraphics[width=0.25\textwidth]{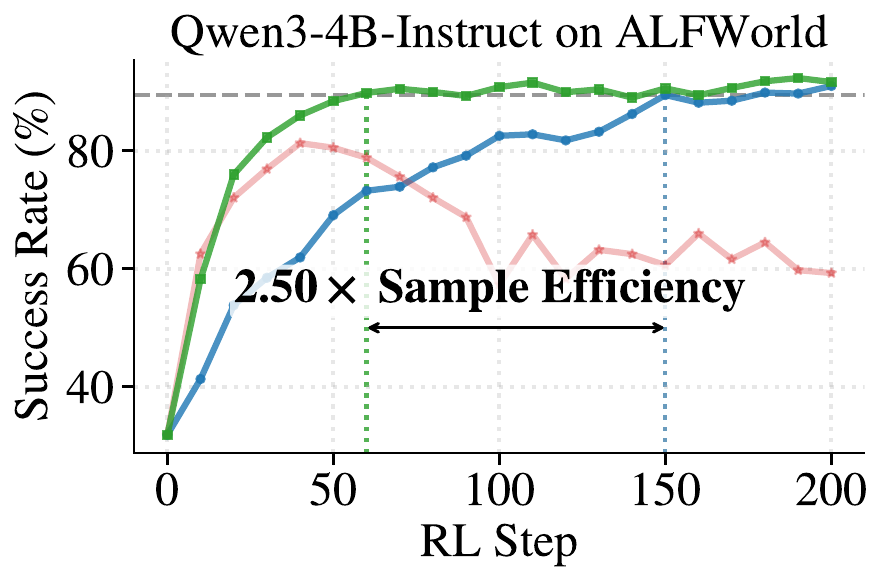}&
    \includegraphics[width=0.25\textwidth]{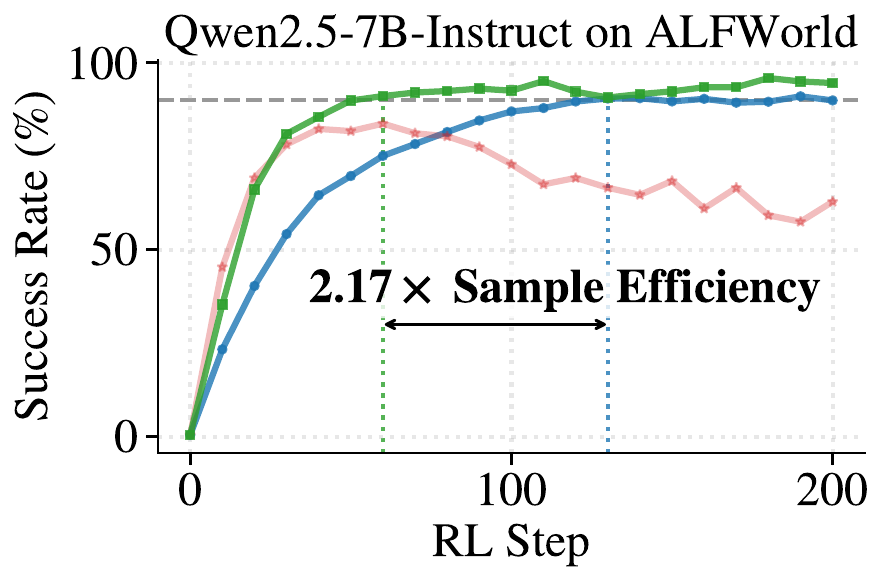}&
    \includegraphics[width=0.25\textwidth]{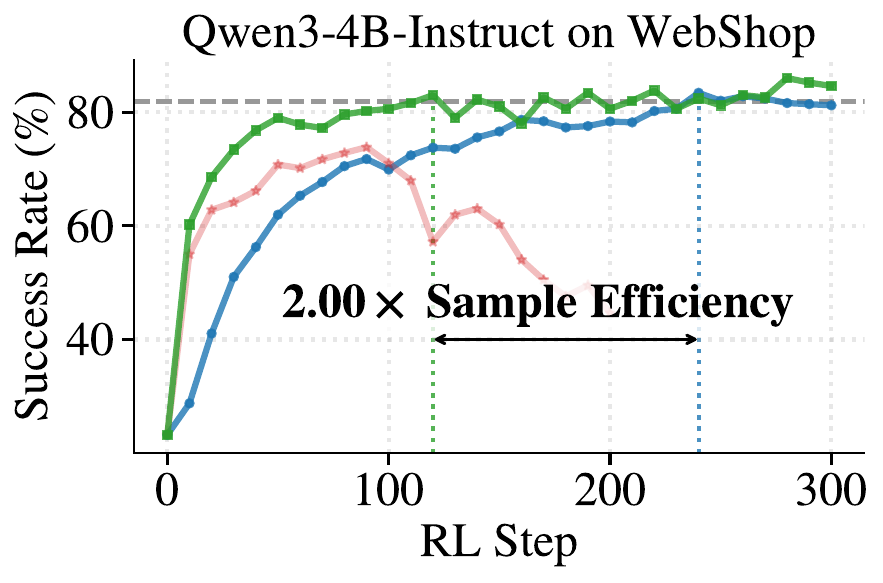}&
    \includegraphics[width=0.25\textwidth]{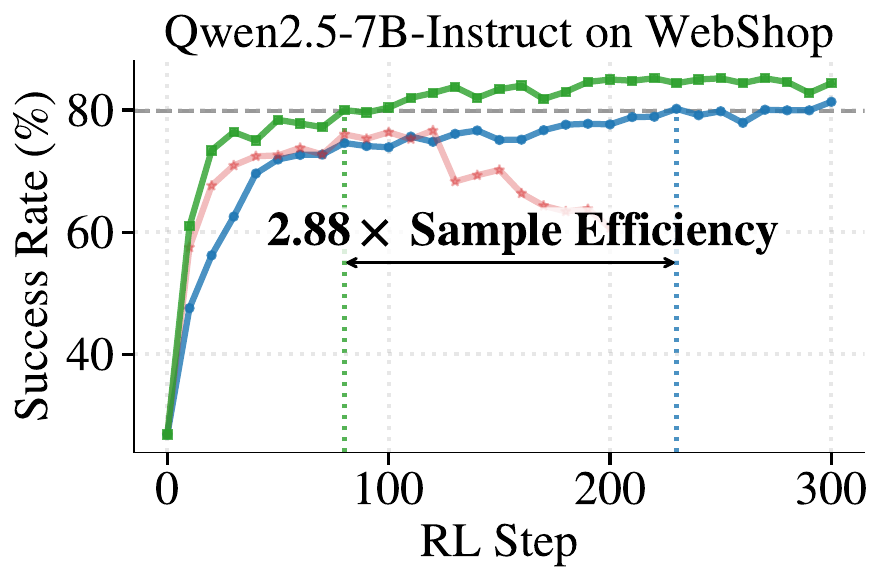}\\[-1pt]
    \multicolumn{2}{c}{\ \ \ \ \ \ \ \textbf{(c) ALFWorld}} & \multicolumn{2}{c}{\ \ \ \ \ \ \ \textbf{(d) WebShop}}\\[-7pt]
    \end{tabular}
    \caption{\textbf{Performance comparison} across two LLMs and four tasks. GRPO w/ Naive Reuse uses a fixed number of updates per batch, equal to DGG's maximum reuse. Gray dashed lines mark the converged performance of GRPO (Single-Use Rollout)—averaged over the last five checkpoints—used as the reference for sample efficiency.  Observation: \textit{DGG eliminates the instability of naive sample reuse, substantially improving sample efficiency with modest gains in final performance.}}
    \label{fig:performance}
\end{figure}

\begin{table}[]
\centering
\scriptsize
\caption{\textbf{Rollout\protect\footnotemark{} and wall-clock speedups of DGG} across two models and four tasks. Speedups are measured against the GRPO (Single-Use Rollout) baseline's converged performance, averaged over its last five checkpoints. Each cell shows the speedup factor with relative reduction in parentheses\protect\footnotemark{}. Math and Search-QA are averaged over four datasets each. Observation: \textit{DGG achieves up to $2.93\times$ rollout and $2.14\times$ wall-clock speedups across all settings.}
}
\label{tab:efficiency}
\renewcommand{\arraystretch}{0.8}
\setlength{\tabcolsep}{1pt}
\begin{tabularx}{\textwidth}{@{} >{\centering\arraybackslash}m{2.4cm} *{8}{>{\centering\arraybackslash}X} @{}}
\toprule
\multirow{2}{*}[-1ex]{\textbf{Model}}
& \multicolumn{2}{c}{\textbf{Math (avg.)}} 
& \multicolumn{2}{c}{\textbf{ALFWorld}} 
& \multicolumn{2}{c}{\textbf{WebShop}} 
& \multicolumn{2}{c}{\textbf{Search-QA (avg.)}} \\
\cmidrule(lr){2-3} \cmidrule(lr){4-5} \cmidrule(lr){6-7} \cmidrule(lr){8-9}
& Rollout & Time
& Rollout & Time
& Rollout & Time
& Rollout & Time \\
\midrule
\multirow{2}{*}[-0.1ex]{Qwen3-4B-Instruct}   & 2.64$\times$ & 1.45$\times$ & 2.50$\times$ & 2.14$\times$ & 2.00$\times$ & 1.31$\times$ & 2.24$\times$ & 1.82$\times$ \\
                                              & (62.2\%)     & (30.9\%)     & (60.0\%)     & (53.4\%)     & (50.0\%)     & (23.6\%)     & (55.4\%)     & (45.1\%)     \\
\midrule
\multirow{2}{*}[-0.1ex]{Qwen2.5-7B-Instruct} & 2.93$\times$ & 1.37$\times$ & 2.17$\times$ & 1.90$\times$ & 2.88$\times$ & 1.81$\times$ & 2.19$\times$ & 1.64$\times$ \\
                                              & (65.9\%)     & (26.9\%)     & (53.9\%)     & (47.5\%)     & (65.2\%)     & (44.7\%)     & (54.3\%)     & (38.9\%)     \\
\bottomrule
\end{tabularx}
\end{table}
\footnotetext{Rollout speedup is equivalent to the sample efficiency reported in Figure~\ref{fig:performance}.}
\footnotetext{Speedup factor is defined as $\text{baseline cost}/\text{DGG cost}$, and relative reduction  as $1 - \text{DGG cost}/\text{baseline cost}$, where cost is the total rollouts (or wall-clock time) needed to reach the baseline's converged performance.}

\subsection{Main Results}
We compare DGG against vanilla GRPO with single-use rollouts (the standard practice in current RLVR systems) and GRPO with naive sample reuse (the most direct alternative). Figure~\ref{fig:performance} reports the training dynamics across all settings, while Table~\ref{tab:efficiency} quantifies the resulting savings in rollouts and wall-clock time; seed stability is 
further verified in Appendix~\ref{app:seed_stability}. We summarize main findings below.

\textbf{DGG substantially mitigates the training instability induced by naive sample reuse.} As shown in Figure~\ref{fig:performance}, although GRPO w/ Naive Reuse converges faster than the single-use baseline in the early stage, it suffers from a consistent failure pattern: performance is severely degraded following the initial improvement, often before reaching the single-use baseline's final performance. In contrast, GRPO w/ DGG retains the early-stage acceleration while avoiding such collapses across all settings. This is consistent with our analysis in Section~\ref{sec:dwd}: by monitoring the \texttt{lm\_head} gradient norm as a signal of policy shift (Theorem~\ref{thm:chi2_bound}), DGG discards harmful updates before they affect the optimizer state, addressing the main failure mode of sample reuse.

\textbf{DGG significantly improves both sample and wall-clock efficiency without sacrificing final performance.} Figure~\ref{fig:performance} shows that GRPO w/ DGG consistently reaches the single-use baseline's converged performance with substantially fewer RL steps across all 16 settings. Table~\ref{tab:efficiency} further quantifies these gains in terms of total training cost: DGG achieves $2.00\times$–$2.93\times$ rollout speedups and $1.31\times$–$2.14\times$ wall-clock speedups over the single-use baseline. Moreover, DGG also yields modest improvements in final performance over the single-use baseline, indicating that safe sample reuse not only accelerates training but also helps the model converge to a slightly better solution.

\textbf{The benefits of DGG generalize across model scales and task domains.}
The improvements above are consistent across two LLMs from different generations of the Qwen family (Qwen3-4B and Qwen2.5-7B) and four representative task domains spanning mathematical reasoning and three agentic settings (search-augmented QA, ALFWorld, and WebShop). This generality is consistent with our observation in Section~\ref{sec:dwd} that the DWD phenomenon is broadly present across diverse LLM architectures and tasks, providing a principled foundation for DGG's wide applicability.

 \begin{figure*}[]
    \centering\scriptsize\renewcommand\arraystretch{0.9}
	\begin{tabular}{c}
	\includegraphics[width=0.5\linewidth]{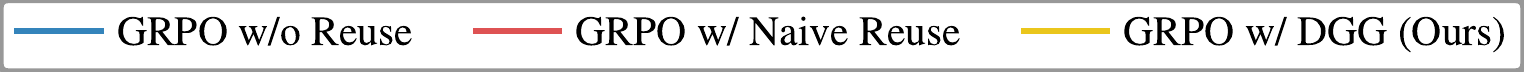}\\
	\end{tabular}
    \setlength{\tabcolsep}{0pt}
    \begin{tabular}{@{}cccc@{}}
    \includegraphics[width=0.25\textwidth]{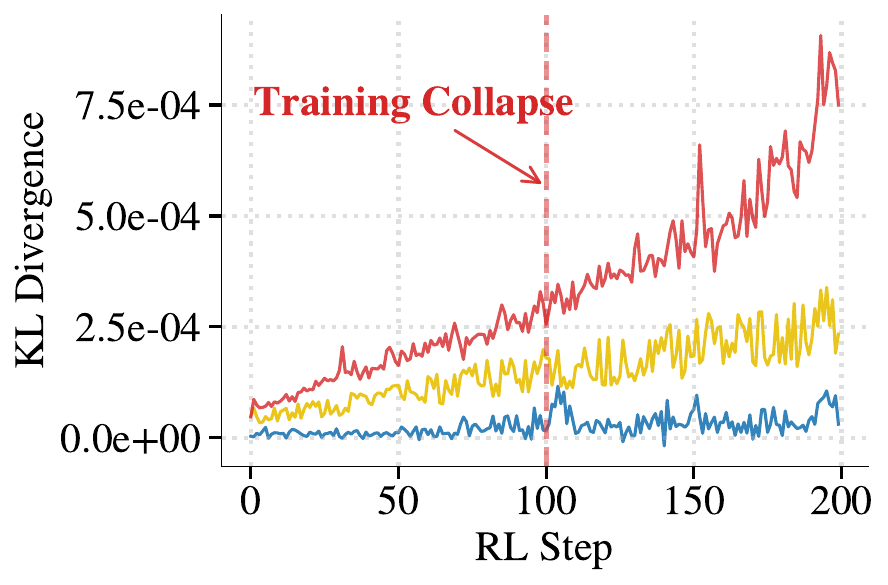}&
    \includegraphics[width=0.25\textwidth]{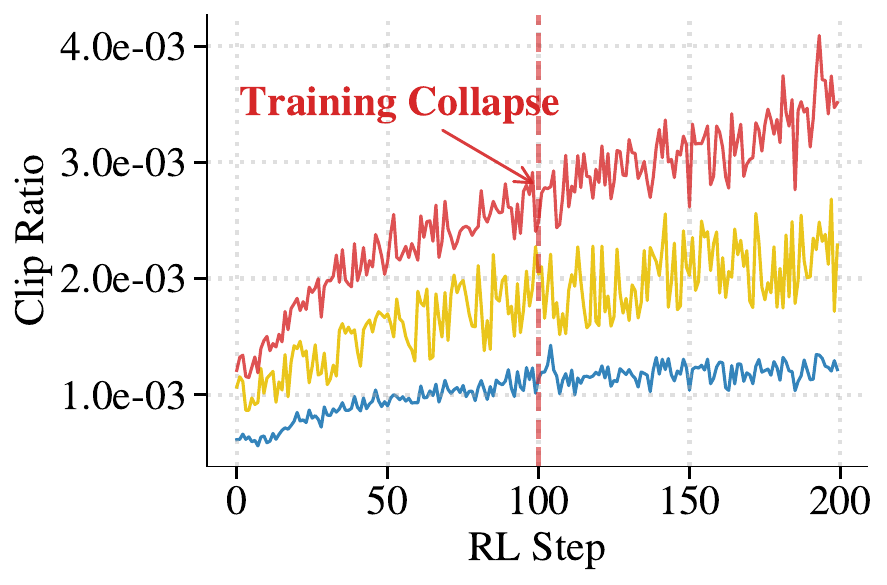}&
    \includegraphics[width=0.25\textwidth]{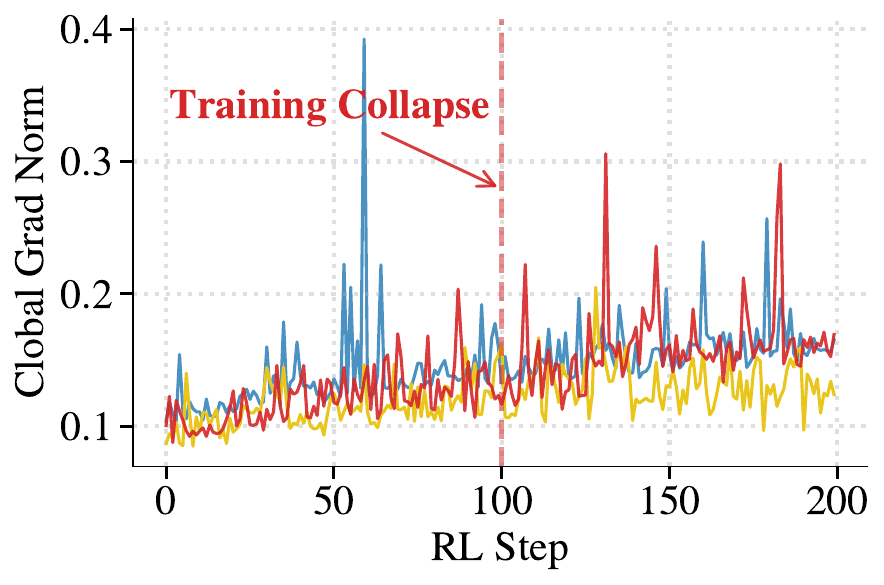}&
    \includegraphics[width=0.25\textwidth]{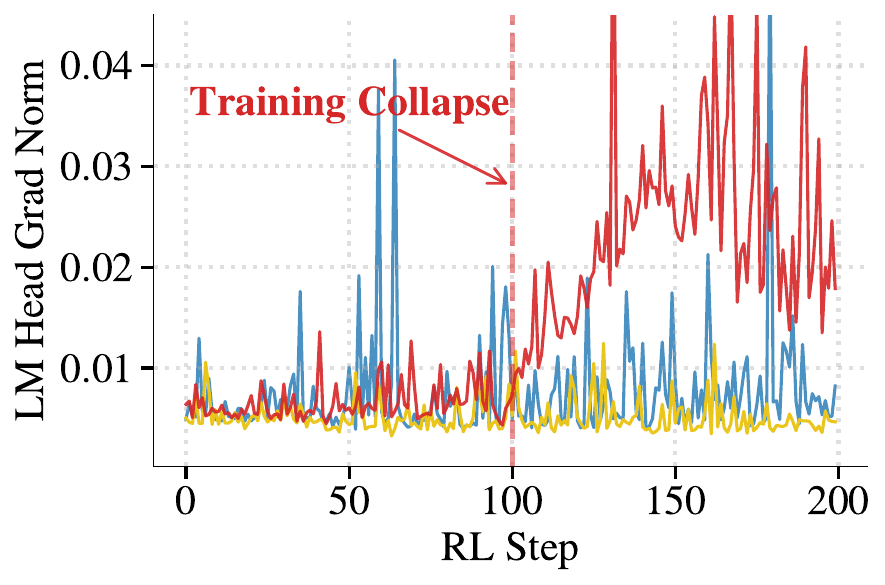}\\
    \end{tabular}
    \vspace{-0.4cm}
    \caption{\textbf{Comparison of different monitoring signals} on Qwen2.5-7B-Instruct (Math500). The red dashed line marks the onset of training collapse for GRPO w/ Naive Reuse. Observation: \textit{While KL divergence, clip ratio, and global gradient norm lack clear collapse signals, the \texttt{lm\_head} gradient norm provides a sharp spike at collapse, serving as a reliable indicator that empirically validates our Structural Gradient Asymmetry (Theorem~\ref{thm:asymmetry}) and Divergence Lower Bound (Theorem~\ref{thm:chi2_bound}).}}    \label{fig:compare_monitor}
\end{figure*}

\begin{figure}[]
    \centering
    \begin{minipage}[b]{0.49\textwidth}
    	\centering\scriptsize\renewcommand\arraystretch{1.0}
		\begin{tabular}{c}
		\hspace*{0.3cm}\includegraphics[width=0.9\linewidth]{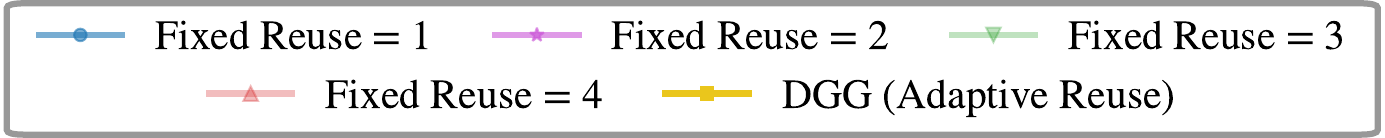}\\[-1pt]
		\end{tabular}
    	\setlength{\tabcolsep}{0pt}
    	\begin{tabular}{cc}
    	\includegraphics[width=0.5\textwidth]{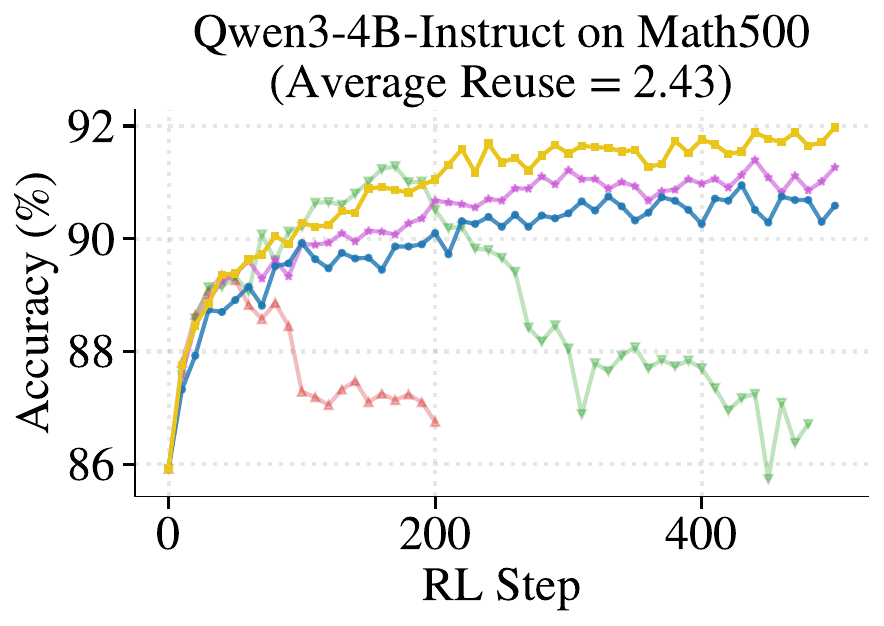}&
    	\includegraphics[width=0.5\textwidth]{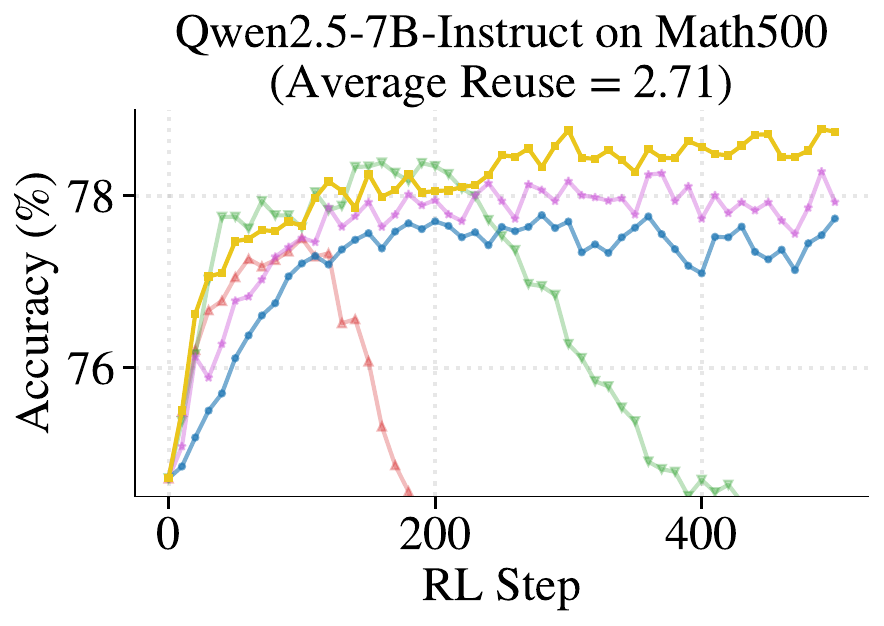}\\
    	\end{tabular}
    	\vspace{-0.4cm}
        \caption{\textbf{Comparison between fixed and adaptive reuse mechanisms.} Observation: \textit{The adaptive mechanism consistently performs better than all fixed counterparts across multiple reuse steps.}}
        \label{fig:diff_staleness}
    \end{minipage}
    \hfill 
    \begin{minipage}[b]{0.49\textwidth}
        \centering\scriptsize\renewcommand\arraystretch{1.0}
    	\setlength{\tabcolsep}{0pt}
    	\begin{tabular}{cc}
    	\includegraphics[width=0.5\textwidth]{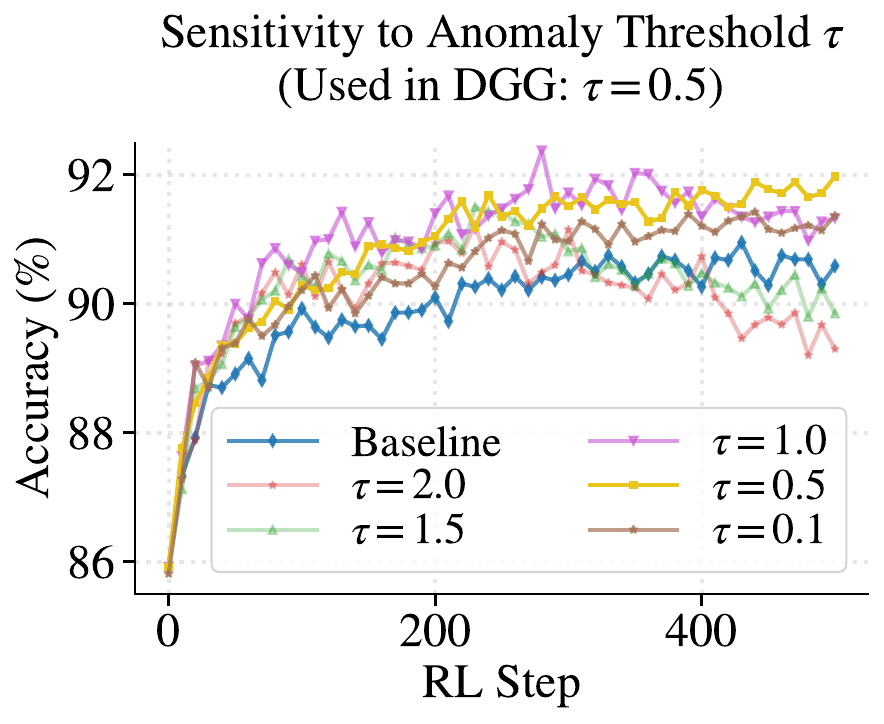}&
    	\includegraphics[width=0.5\textwidth]{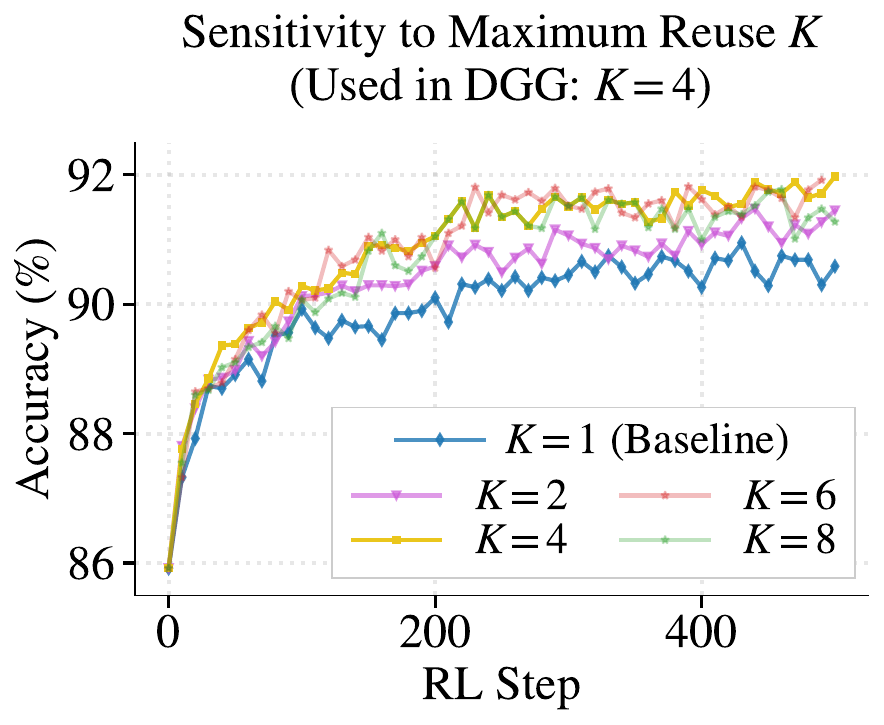}\\
    	\end{tabular}
    	\vspace{-0.4cm}
        \caption{\textbf{Sensitivity analysis on DGG's hyperparameters $\tau$ and $K$}, conducted on Qwen3-4B-Instruct (Math500). Note that $K=1$ reduces to the single-use baseline. Observation: \textit{DGG achieves better performance than the single-use baseline across most  settings.}}
        \label{fig:sensitivity}
    \end{minipage}
\end{figure}

\subsection{Analysis and Discussion}
Beyond demonstrating DGG's effectiveness, we now examine three aspects of its design: the choice of \texttt{lm\_head} gradient norm as the monitoring signal, the advantage of adaptive gating over fixed-reuse alternatives, and practical guidelines for setting its hyperparameters.

\textbf{Comparison with alternative monitoring signals.} Figure~\ref{fig:compare_monitor} evaluates four candidate monitoring signals: the KL divergence between $\pi_\theta$ and $\pi_{\text{old}}$, the clip ratio, the global gradient norm, and the \texttt{lm\_head} gradient norm. While the KL divergence, clip ratio, and global gradient norm fail to exhibit a distinct spike at the onset of collapse, the \texttt{lm\_head} gradient norm displays a sharp, well-scaled spike precisely synchronized with collapse onset. \textit{This observation directly corroborates our theoretical framework in Section~\ref{sec:dwd}: the \texttt{lm\_head} gradient norm is structurally localized (Theorem~\ref{thm:asymmetry}) and provides a lower bound on policy divergence (Theorem~\ref{thm:chi2_bound}), thereby justifying its selection as DGG's monitoring signal.}

\textbf{Comparison between fixed and adaptive reuse mechanisms.} A natural baseline for stability is to manually tune a conservative fixed-reuse mechanism. Figure~\ref{fig:diff_staleness} compares DGG against fixed-reuse variants with $K \in \{1, 2, 3, 4\}$. Fixed reuse exposes a fundamental trade-off: small counts maintain stability but limit efficiency, whereas larger counts accelerate early convergence but trigger collapse. By contrast, DGG dynamically modulates reuse based on the \texttt{lm\_head} gradient, inheriting aggressive reuse's early acceleration while preserving stability via real-time gating. \textit{DGG thus consistently outperforms all fixed counterparts, resolving the stability-efficiency trade-off intrinsic to fixed reuse.}

\textbf{Sensitivity analysis and practical recommendations.} DGG introduces two hyperparameters: the anomaly threshold $\tau$ and the maximum reuse $K$, where $K = 1$ reduces to the single-use baseline. As shown in Figure~\ref{fig:sensitivity}, $\tau \in \{0.1, 0.5, 1.0\}$ achieves better performance compared with the single-use baseline, and \textit{the cost of mis-tuning $\tau$ is asymmetric}: an overly small $\tau$ merely sacrifices a few reuse steps and still outperforms the single-use baseline, whereas an overly large $\tau$ allows harmful gradients to pass and degrades performance. We therefore recommend selecting $\tau$ from the candidate set $\{0.1, 0.5, 1.0\}$, biased toward the smaller end. For $K$, all settings with $K \geq 2$ outperform the single-use baseline, with marginal differences across $K \in \{4, 6, 8\}$. Since larger $K$ incurs higher worst-case computation, we recommend $K = 4$ as a practical default.

\vspace{-0.1cm}
\section{Conclusion and Limitations}
\label{sec:conclusion}
\vspace{-0.1cm}
We identify the DWD phenomenon: sample-reuse collapse in RLVR is driven by harmful gradients structurally concentrating at the \texttt{lm\_head}. We prove this gradient norm lower-bounds policy divergence, enabling real-time detection of catastrophic shifts. Consequently, we propose DGG to intercept anomalous updates pre-optimizer, significantly improving sample efficiency across reasoning and agentic tasks without performance loss. However, our analysis focuses on the GRPO objective. Future work will investigate standard PPO to determine if a similar structural divergence emerges at the \texttt{value\_head} under value network collapse.

\bibliography{example_paper.bib}
\bibliographystyle{plainnat}
\newpage
\appendix

{\setlength{\baselineskip}{10pt}
 \tableofcontents}
 
\clearpage

\section{Detailed Proofs and Empirical Measurements for Section~\ref{subsec:dwd_theory}}

\label{app:dwd_proofs}

This appendix provides detailed proofs for all theoretical results presented in Section~\ref{subsec:dwd_theory}, alongside empirical measurements supporting Theorem~\ref{thm:asymmetry}. To make this section self-contained, we restate each result prior to its proof. Throughout the appendix, we retain the notation introduced in the main text (see ``Setup and notation'' in Section~\ref{subsec:dwd_theory}).

\subsection{Proof of Proposition~\ref{prop:grad_decomp} (Gradient Decomposition)}
\label{app:proof_grad_decomp}

\begin{proposition}[Restatement of Proposition~\ref{prop:grad_decomp}]
Let $\mathcal{L}_i$ denote the token-level objective for token $a_i$. Define the raw error signal at the final logits as $E_i \triangleq r_i \hat{A}_i ( e_{a_i} - \pi_\theta(\ \cdot \ | h_{L,i}) ) \in \mathbb{R}^{d_{\text{vocab}}}$. Then for any token position $i$, the gradients of $\mathcal{L}_i$ with respect to the \texttt{lm\_head} weight $W_{\text{lm}}$ and an arbitrary intermediate linear layer weight $W_{\text{int}}$ admit the following closed-form rank-1 structures:\vspace{-2pt}
\begin{equation}
G_i^{\text{lm}} \triangleq \nabla_{W_{\text{lm}}} \mathcal{L}_i = E_i \, h_{L,i}^\top, \qquad
G_i^{\text{int}} \triangleq \nabla_{W_{\text{int}}} \mathcal{L}_i = \bigl( J_i^\top E_i \bigr) \bigl( x_i^{\text{int}} \bigr)^\top,
\end{equation}
where $J_i = \partial z_i / \partial y_i$ is the Jacobian of the output logits $z_i$ with respect to the output of the intermediate layer $y_i = W_{\text{int}} x_i^{\text{int}}$.
\end{proposition}

\begin{proof}
The proof relies on the standard multivariate chain rule applied to the computational graph of the transformer architecture. We decouple the derivation into two steps, tracking the gradient flow from the final loss down to the respective layer weights.

\textbf{Step 1: Gradient with respect to the \texttt{lm\_head} weight.}
Recall that the policy distribution $\pi_\theta(\cdot \mid h_{L,i})$ is parameterized by the softmax function applied to the final pre-activation logits $z_i \in \mathbb{R}^{d_{\text{vocab}}}$. In the unclipped region of the GRPO surrogate objective, the objective for a single token $a_i$ is given by $\mathcal{L}_i = r_i \hat{A}_i = \frac{\pi_\theta(a_i \mid h_{L,i})}{\pi_{\text{old}}(a_i \mid h_{L,i})} \hat{A}_i$.

To compute the gradient of $\mathcal{L}_i$ with respect to the logits $z_i$, we apply the chain rule:
\begin{equation}
\nabla_{z_i} \mathcal{L}_i = \frac{\partial \mathcal{L}_i}{\partial \pi_\theta(a_i \mid h_{L,i})} \nabla_{z_i} \pi_\theta(a_i \mid h_{L,i}).
\end{equation}

The derivative of the loss with respect to the scalar probability of the sampled token is simply:
\begin{equation}
\frac{\partial \mathcal{L}_i}{\partial \pi_\theta(a_i \mid h_{L,i})} = \frac{\hat{A}_i}{\pi_{\text{old}}(a_i \mid h_{L,i})}.
\end{equation}

Furthermore, the gradient of the softmax probability $\pi_\theta(a_i \mid h_{L,i})$ with respect to the full logit vector $z_i$ evaluates to:
\begin{equation}
\nabla_{z_i} \pi_\theta(a_i \mid h_{L,i}) = \pi_\theta(a_i \mid h_{L,i}) \bigl( e_{a_i} - \pi_\theta(\cdot \mid h_{L,i}) \bigr),
\end{equation}
where $e_{a_i} \in \mathbb{R}^{d_{\text{vocab}}}$ is the one-hot indicator vector for the sampled token $a_i$. Multiplying these two components together, the behavior policy probability $\pi_{\text{old}}$ and the current policy probability $\pi_\theta$ merge back to recover the importance ratio $r_i$:
\begin{equation}
\nabla_{z_i} \mathcal{L}_i = \frac{\pi_\theta(a_i \mid h_{L,i})}{\pi_{\text{old}}(a_i \mid h_{L,i})} \hat{A}_i \bigl( e_{a_i} - \pi_\theta(\cdot \mid h_{L,i}) \bigr) = r_i \hat{A}_i \bigl( e_{a_i} - \pi_\theta(\cdot \mid h_{L,i}) \bigr) \triangleq E_i.
\label{eq:logit_grad}
\end{equation}

The forward pass of the \texttt{lm\_head} is a linear projection mapping the final hidden state $h_{L,i} \in \mathbb{R}^{d_{\text{model}}}$ to the logits $z_i \in \mathbb{R}^{d_{\text{vocab}}}$:
\begin{equation}
z_i = W_{\text{lm}} h_{L,i}.
\end{equation}
To derive the gradient with respect to the weight matrix $W_{\text{lm}} \in \mathbb{R}^{d_{\text{vocab}} \times d_{\text{model}}}$, we apply the chain rule at the scalar level. For any specific entry $(W_{\text{lm}})_{k, m}$, it only contributes to the $k$-th logit $(z_i)_k = \sum_{m'} (W_{\text{lm}})_{k, m'} (h_{L,i})_{m'}$. Thus, the partial derivative is:
\begin{equation}
\frac{\partial \mathcal{L}_i}{\partial (W_{\text{lm}})_{k, m}} = \frac{\partial \mathcal{L}_i}{\partial (z_i)_k} \frac{\partial (z_i)_k}{\partial (W_{\text{lm}})_{k, m}} = (E_i)_k (h_{L,i})_m.
\end{equation}
Reconstructing these element-wise partial derivatives into matrix form precisely yields the outer product of the upstream gradient vector $E_i$ and the input vector $h_{L,i}$:
\begin{equation}
G_i^{\text{lm}} \triangleq \nabla_{W_{\text{lm}}} \mathcal{L}_i = E_i h_{L,i}^\top.
\end{equation}
This establishes the first identity.

\textbf{Step 2: Gradient with respect to an intermediate layer weight.}
For an arbitrary intermediate linear layer, let its forward pass be denoted as $y_i = W_{\text{int}} x_i^{\text{int}}$, where $x_i^{\text{int}}$ is the input to this layer and $y_i$ is its pre-activation output.

Let $J_i = \frac{\partial z_i}{\partial y_i} $ denote the composite Jacobian matrix, where its $(k, m)$-th entry is $(J_i)_{k, m} = \frac{\partial (z_i)_k}{\partial (y_i)_m}$. To find the upstream gradient arriving at the vector $y_i$, we apply the multivariate chain rule to its $m$-th component:
\begin{equation}
\frac{\partial \mathcal{L}_i}{\partial (y_i)_m} = \sum_{k=1}^{d_{\text{vocab}}} \frac{\partial \mathcal{L}_i}{\partial (z_i)_k} \frac{\partial (z_i)_k}{\partial (y_i)_m} = \sum_{k=1}^{d_{\text{vocab}}} (E_i)_k (J_i)_{k, m}.
\end{equation}
Recognizing this sum as a matrix-vector product, the full upstream gradient vector can be compactly written using the transposed Jacobian:
\begin{equation}
\nabla_{y_i} \mathcal{L}_i = J_i^\top E_i.
\end{equation}
Next, to compute the gradient with respect to the weight matrix $W_{\text{int}}$, we again apply the scalar chain rule. Since the element $(W_{\text{int}})_{m, p}$ only affects $(y_i)_m = \sum_{p'} (W_{\text{int}})_{m, p'} (x_i^{\text{int}})_{p'}$, we have:
\begin{equation}
\frac{\partial \mathcal{L}_i}{\partial (W_{\text{int}})_{m, p}} = \frac{\partial \mathcal{L}_i}{\partial (y_i)_m} \frac{\partial (y_i)_m}{\partial (W_{\text{int}})_{m, p}} = (\nabla_{y_i} \mathcal{L}_i)_m (x_i^{\text{int}})_p.
\end{equation}
Assembling these scalar derivatives back into matrix form yields the outer product:
\begin{equation}
G_i^{\text{int}} \triangleq \nabla_{W_{\text{int}}} \mathcal{L}_i = (\nabla_{y_i} \mathcal{L}_i) (x_i^{\text{int}})^\top = \bigl( J_i^\top E_i \bigr) \bigl( x_i^{\text{int}} \bigr)^\top.
\end{equation}
This establishes the second identity and completes the proof of Proposition 1.
\end{proof}

\paragraph{Remark on token-level analysis.} 
A natural question is why Proposition~\ref{prop:grad_decomp} is stated at the token level. Our token-level formulation is a deliberate analytical choice driven by the nature of sample-reuse-induced collapse. As established in Section~\ref{sec:preliminaries}, this collapse is mechanistically driven by a small subset of \emph{tail tokens}---those with vanishingly small $\pi_{\text{old}}(a_i)$ but massively inflated importance ratios $r_i$ and large advantages, which dominate the gradient and trigger the anomaly. Proposition~\ref{prop:grad_decomp} is therefore designed to localize the structural asymmetry precisely on these harmful tokens: the policy-confidence factor $(1-\pi_\theta(a_i))^{-2}$ is tightest exactly when $\pi_\theta(a_i)$ is small, faithfully capturing the regime where collapse originates. A batch-level statement, by averaging over many benign tokens, would dilute this signal and obscure the very mechanism we aim to expose. The empirical surge in $\|G^{\text{lm}}\|_F^2$ observed in The empirical surge in $\|G^{\text{lm}}\|_F^2$ observed in Figures~\ref{fig:empirical_dwd_qwen3-4b-instruct} and~\ref{fig:compare_monitor}---which is itself a batch-level quantity---is consistent with the token-level analysis precisely because tail tokens dominate the gradient signal during collapse.

\subsection{Proof of Lemma~\ref{lem:activation_bound} (Bounded Activations under Pre-Norm Architectures)}
\label{app:activation_bounds}

\begin{lemma}[Restatement of Lemma~\ref{lem:activation_bound}]
For any active token $i$, let $h_{L,i}$ denote the \texttt{lm\_head} input and $x_i^{\text{int}}$ denote the intermediate-layer input. There exist strictly positive constants $\alpha_{\min}$ and $\beta_{\max}$ such that:
\vspace{-2pt}
\begin{equation}
\| h_{L,i} \|_2^2 \;\geq\; \alpha_{\min} \cdot d_{\text{model}}, \qquad
\| x_i^{\text{int}} \|_2^2 \;\leq\; \beta_{\max} \cdot d_{\text{model}}.
\end{equation}
\end{lemma}

\begin{proof}
We establish the two bounds separately, as they involve activations from distinct sub-modules: $h_{L,i}$ is produced by the final RMSNorm preceding the \texttt{lm\_head}, whereas $x_i^{\text{int}}$ is the input to an arbitrary intermediate linear layer.

\textbf{Lower bound on $\|h_{L,i}\|_2^2$.}
The \texttt{lm\_head} input takes the form $h_{L,i} = \text{RMSNorm}(v_L) \odot w^{\text{lm}}$, where $w^{\text{lm}} \in \mathbb{R}^{d_{\text{model}}}$ is the affine scale parameter of the final RMSNorm layer, and $\odot$ denotes element-wise multiplication. By the definition of the squared $\ell_2$ norm and the RMSNorm operation, we have:
\begin{equation}
\|h_{L,i}\|_2^2 \;=\; \sum_{k=1}^{d_{\text{model}}} \frac{(v_L)_k^2}{\frac{1}{d_{\text{model}}}\|v_L\|_2^2 + \epsilon} \, (w^{\text{lm}}_k)^2,
\end{equation}
where $\epsilon$ is a fixed numerical stabilizer (typically $10^{-5}$). To simplify the analysis, let $r \triangleq \frac{1}{d_{\text{model}}}\|v_L\|_2^2$ denote the mean squared magnitude of the pre-normalization vector $v_L$. The denominator in the summation can then be rewritten as $r + \epsilon$. Since this denominator is constant with respect to the summation index $k$, we can factor it out:
\begin{equation}
\|h_{L,i}\|_2^2 \;=\; \frac{1}{r+\epsilon} \sum_{k=1}^{d_{\text{model}}} (v_L)_k^2 \, (w^{\text{lm}}_k)^2.
\end{equation}
Next, we multiply and divide the terms inside the summation by $\|v_L\|_2^2$ to construct a weighted average:
\begin{equation}
\|h_{L,i}\|_2^2 \;=\; \frac{1}{r+\epsilon} \sum_{k=1}^{d_{\text{model}}} \left( \frac{(v_L)_k^2}{\|v_L\|_2^2} \right) \cdot \|v_L\|_2^2 \cdot (w^{\text{lm}}_k)^2.
\end{equation}
Factoring out $\|v_L\|_2^2$ (which is equal to $r \cdot d_{\text{model}}$ by definition) yields:
\begin{equation}
\|h_{L,i}\|_2^2 \;=\; \frac{r \cdot d_{\text{model}}}{r+\epsilon} \sum_{k=1}^{d_{\text{model}}} \frac{(v_L)_k^2}{\|v_L\|_2^2} \, (w^{\text{lm}}_k)^2.
\end{equation}
Observe that the coefficients $c_k = \frac{(v_L)_k^2}{\|v_L\|_2^2}$ satisfy $c_k \geq 0$ for all $k$ and $\sum_k c_k = 1$. Therefore, the summation is a convex combination of the squared scale parameters $\{(w^{\text{lm}}_k)^2\}$. Since a convex combination of a set of values is lower-bounded by their minimum, we obtain:
\begin{equation}
\|h_{L,i}\|_2^2 \;\geq\; \frac{r}{r+\epsilon} \cdot d_{\text{model}} \cdot \Bigl(\min_k (w^{\text{lm}}_k)^2\Bigr).
\end{equation}
Finally, we observe that the mean energy of the pre-normalization vector typically exceeds the numerical stabilizer by a large margin in trained Transformers. Specifically, with $\epsilon = 10^{-5}$ and $d_{\text{model}} \geq 10^3$, the condition $r \geq \epsilon$ merely requires $\|v_L\|_2^2 \geq 10^{-2}$, which is easily satisfied in practice. Consequently, the monotonically increasing function $f(r) = \frac{r}{r+\epsilon}$ safely satisfies $f(r) \geq f(\epsilon) = \frac{1}{2}$, yielding the final bound:
\begin{equation}
\|h_{L,i}\|_2^2 \;\geq\; \alpha_{\min} \cdot d_{\text{model}}, \qquad \text{where} \quad \alpha_{\min} \;\triangleq\; \frac{1}{2} \min_k (w^{\text{lm}}_k)^2.
\end{equation}
This establishes the strictly positive lower bound for the \texttt{lm\_head} input.

\textbf{Upper bound on $\|x_i^{\text{int}}\|_2^2$.}
We classify intermediate linear layers into two classes based on their input source and bound each class separately.

\emph{Class I: Inputs directly from RMSNorm.} The inputs to QKV/up/gate projections take the form $x = \text{RMSNorm}(v) \odot w^{\text{int}}$, where $w^{\text{int}}$ is the affine weight of an intermediate RMSNorm layer. Following the same factorization as above and using the convex-combination upper bound,
\begin{equation}
\|x\|_2^2 \;=\; \frac{r}{r+\epsilon} \cdot d_{\text{model}} 
\sum_k \frac{v_k^2}{\|v\|_2^2} (w^{\text{int}}_k)^2 
\;\leq\; \Bigl(\max_k (w^{\text{int}}_k)^2\Bigr) \cdot d_{\text{model}},
\end{equation}
where the inequality uses $\frac{r}{r+\epsilon} \leq 1$. Defining $\beta^{\text{RMS}} \triangleq \max_k (w^{\text{int}}_k)^2$ (taken over all intermediate RMSNorm layers), we have
\begin{equation}
\|x\|_2^2 \;\leq\; \beta^{\text{RMS}} \cdot d_{\text{model}}.
\label{eq:class1_upper}
\end{equation}

\emph{Class II: Inputs from attention or FFN internal states.} We treat the two cases separately.

\emph{(a) Attention output projection $W_O$.} The input is $x_i^O = \sum_j \alpha_{ij} v_j$, where $v_j = W_V x_j^{\text{RMS}}$ are value vectors and $\alpha_{ij} \geq 0$ with $\sum_j \alpha_{ij} = 1$. By Jensen's inequality applied to the convex function $\|\cdot\|_2^2$,
\begin{equation}
\|x_i^O\|_2^2 \;\leq\; \sum_j \alpha_{ij} \|v_j\|_2^2 \;\leq\; \max_j \|v_j\|_2^2.
\end{equation}
Since $\|v_j\|_2^2 \leq \|W_V\|_2^2 \cdot \|x_j^{\text{RMS}}\|_2^2 \leq \rho_V^2 \beta^{\text{RMS}} d_{\text{model}}$ where $\rho_V \triangleq \|W_V\|_2$ is the operator norm implicitly controlled via weight decay, we obtain
\begin{equation}
\|x_i^O\|_2^2 \;\leq\; \rho_V^2 \beta^{\text{RMS}} \cdot d_{\text{model}}.
\end{equation}

\emph{(b) FFN down projection $W_{\text{down}}$.} We treat the FFN down projection input in a manner agnostic to the specific activation function used. All standard activations $\sigma$ in modern Transformers are Lipschitz continuous, satisfying $|\sigma(z)| \leq L_\sigma |z|$ for some constant $L_\sigma > 0$ (e.g., $L_\sigma = 1$ for ReLU and SiLU/Swish, $L_\sigma \approx 1.13$ for GELU).

For gated variants (SwiGLU, GeGLU), the FFN down input is $x_i^{\text{down}} = \sigma(W_{\text{gate}} x_i^{\text{RMS}}) \odot (W_{\text{up}} x_i^{\text{RMS}})$. Applying the Hadamard product bound $\|u \odot v\|_2 \leq \|u\|_\infty \|v\|_2$:
\begin{equation}
\|x_i^{\text{down}}\|_2 \;\leq\; \|\sigma(W_{\text{gate}} x_i^{\text{RMS}})\|_\infty \cdot \|W_{\text{up}} x_i^{\text{RMS}}\|_2.
\end{equation}
The first factor satisfies $\|\sigma(W_{\text{gate}} x_i^{\text{RMS}})\|_\infty \leq L_\sigma \|W_{\text{gate}} x_i^{\text{RMS}}\|_\infty \leq L_\sigma B_{\text{gate}}$, where $B_{\text{gate}} > 0$ is a finite constant bounding the pre-activation coordinate amplitudes---a property maintained in any trained network through the joint effect of the preceding RMSNorm and weight decay regularization on $W_{\text{gate}}$.

The second factor admits the standard spectral-norm bound $\|W_{\text{up}} x_i^{\text{RMS}}\|_2 \leq \rho_{\text{up}} \sqrt{\beta^{\text{RMS}} d_{\text{model}}}$, where the operator norm $\rho_{\text{up}} \triangleq \|W_{\text{up}}\|_2$ is implicitly controlled via weight decay. Combining these terms, we obtain:
\begin{equation}
\|x_i^{\text{down}}\|_2^2 \;\leq\; L_\sigma^2 B_{\text{gate}}^2 \rho_{\text{up}}^2 \beta^{\text{RMS}} \cdot d_{\text{model}} \;\triangleq\; \rho_{\text{FFN}}^2 \beta^{\text{RMS}} \cdot d_{\text{model}}.
\end{equation}

For non-gated variants ($x_i^{\text{down}} = \sigma(W_{\text{up}} x_i^{\text{RMS}})$), the coordinate-wise Lipschitz property of $\sigma$ implies $\|\sigma(z)\|_2 \leq L_\sigma \|z\|_2$, and the operator-norm bound directly gives $\|x_i^{\text{down}}\|_2^2 \leq L_\sigma^2 \rho_{\text{up}}^2 \beta^{\text{RMS}} d_{\text{model}}$, fitting the same form (subsumed by $\rho_{\text{FFN}}^2$ defined to cover both cases).

\textbf{Unified upper bound.} Define
\begin{equation}
\beta_{\max} \;\triangleq\; \max\bigl( \beta^{\text{RMS}}, \;\rho_V^2 \beta^{\text{RMS}}, \;\rho_{\text{FFN}}^2 \beta^{\text{RMS}} \bigr).
\end{equation}
Combining the upper bounds across Class I and Class II yields $\|x_i^{\text{int}}\|_2^2 \leq \beta_{\max} \cdot d_{\text{model}}$ uniformly across all intermediate linear layers. Note that $\alpha_{\min}$ and $\beta_{\max}$ are determined by separate sub-modules of the network---the final RMSNorm affine weight $w^{\text{lm}}$ for $\alpha_{\min}$, and the intermediate RMSNorm affine weights $w^{\text{int}}$ together with attention/FFN operator norms for $\beta_{\max}$. The empirical values of $\beta_{\max}/\alpha_{\min}$ are reported across diverse architectures in Appendix~\ref{app:empirical_table}.
\end{proof}

\subsection{Proof of Theorem~\ref{thm:asymmetry} (Structural Gradient Asymmetry)}
\label{app:proof_main_theorem}

\begin{theorem}[Restatement of Theorem~\ref{thm:asymmetry}]
Under Lemma~\ref{lem:activation_bound} and Assumption~\ref{ass:jacobian}, for any sampled token $a_i$ with policy probability $\pi_\theta(a_i) < 1$ and non-zero advantage, the ratio of the Frobenius gradient energy between any intermediate layer and the \texttt{lm\_head} is bounded by:
\begin{equation}
\frac{ \| G_i^{\text{int}} \|_F^2 }{ \| G_i^{\text{lm}} \|_F^2 } \;\leq\; \mathcal{C}_{\text{struct}} \cdot \frac{1}{(1-\pi_\theta(a_i))^2},
\label{eq:main_bound_F_appendix}
\end{equation}
where $\mathcal{C}_{\text{struct}} \triangleq \frac{4 \beta_{\max} C}{\alpha_{\min}}$ is a strictly positive architectural constant.
\end{theorem}

\begin{proof}
Throughout this proof, we suppress the conditioning of $\pi_\theta$ on $h_{L,i}$ for notational brevity, writing $\pi_\theta(j)$ for $\pi_\theta(j \mid h_{L,i})$.

We first establish a useful identity for the Frobenius norm of rank-1 matrices. For any vectors $u \in \mathbb{R}^n$ and $v \in \mathbb{R}^m$, the outer product $M = u v^\top \in \mathbb{R}^{n \times m}$ has entries $M_{j,k} = u_j v_k$. By definition of the Frobenius norm and applying the distributive law to the double sum,
\begin{equation}
\|M\|_F^2 = \sum_{j=1}^{n} \sum_{k=1}^{m} M_{j,k}^2 = \sum_{j=1}^{n} \sum_{k=1}^{m} u_j^2 v_k^2 = \biggl( \sum_{j=1}^{n} u_j^2 \biggr) \biggl( \sum_{k=1}^{m} v_k^2 \biggr) = \|u\|_2^2 \cdot \|v\|_2^2.
\label{eq:rank1_frobenius}
\end{equation}
We apply this identity to both $G_i^{\text{lm}}$ and $G_i^{\text{int}}$, both of which are rank-1 outer products by Proposition~\ref{prop:grad_decomp}.

\textbf{Step 1: Lower bound on $\|G_i^{\text{lm}}\|_F^2$.}
By Proposition~\ref{prop:grad_decomp}, $G_i^{\text{lm}} = E_i h_{L,i}^\top$. Applying Eq.~\eqref{eq:rank1_frobenius},
\begin{equation}
\|G_i^{\text{lm}}\|_F^2 = \|E_i\|_2^2 \cdot \|h_{L,i}\|_2^2.
\label{eq:lm_factorization}
\end{equation}
We bound $\|E_i\|_2^2$ from below. Recall $E_i = r_i \hat{A}_i ( e_{a_i} - \pi_\theta )$, whose entries are
\begin{equation}
(E_i)_j = \begin{cases}
r_i \hat{A}_i \bigl(1 - \pi_\theta(a_i)\bigr), & j = a_i, \\
-r_i \hat{A}_i \pi_\theta(j), & j \neq a_i.
\end{cases}
\end{equation}
By definition of the squared $\ell_2$ norm and isolating the $j = a_i$ component,
\begin{align}
\|E_i\|_2^2 &= \sum_{j=1}^{d_{\text{vocab}}} (E_i)_j^2 \\
&= (E_i)_{a_i}^2 + \sum_{j \neq a_i} (E_i)_j^2 \\
&= r_i^2 \hat{A}_i^2 \biggl[ \bigl(1 - \pi_\theta(a_i)\bigr)^2 + \sum_{j \neq a_i} \pi_\theta(j)^2 \biggr] \\
&\geq r_i^2 \hat{A}_i^2 \cdot \bigl(1 - \pi_\theta(a_i)\bigr)^2,
\label{eq:e_l2_lower}
\end{align}
where the final inequality holds since $\sum_{j \neq a_i} \pi_\theta(j)^2 \geq 0$ (each term is non-negative). Combining Eq.~\eqref{eq:e_l2_lower} with Lemma~\ref{lem:activation_bound}'s lower bound $\|h_{L,i}\|_2^2 \geq \alpha_{\min} d_{\text{model}}$,
\begin{equation}
\|G_i^{\text{lm}}\|_F^2 \;\geq\; r_i^2 \hat{A}_i^2 \cdot \bigl(1 - \pi_\theta(a_i)\bigr)^2 \cdot \alpha_{\min} \cdot d_{\text{model}}.
\label{eq:lm_lower_F}
\end{equation}

\textbf{Step 2: Upper bound on $\|G_i^{\text{int}}\|_F^2$.}
By Proposition~\ref{prop:grad_decomp}, $G_i^{\text{int}} = (J_i^\top E_i)(x_i^{\text{int}})^\top$. Applying Eq.~\eqref{eq:rank1_frobenius},
\begin{equation}
\|G_i^{\text{int}}\|_F^2 = \|J_i^\top E_i\|_2^2 \cdot \|x_i^{\text{int}}\|_2^2.
\label{eq:int_factorization}
\end{equation}
We focus on bounding $\|J_i^\top E_i\|_2^2 = \sum_{k=1}^{d_{\text{model}}} (J_i^\top E_i)_k^2$.

\emph{Closed-form for $(J_i^\top E_i)_k$.} By the matrix-vector product definition and substituting the entries of $E_i$:
\begin{align}
(J_i^\top E_i)_k 
&= \sum_{j=1}^{d_{\text{vocab}}} (J_i)_{j,k} (E_i)_j \tag{matrix-vector definition}\\
&= r_i \hat{A}_i \biggl[ (J_i)_{a_i,k} \bigl(1 - \pi_\theta(a_i)\bigr) - \sum_{j \neq a_i} (J_i)_{j,k} \pi_\theta(j) \biggr] \tag{substituting $(E_i)_j$}\\
&= r_i \hat{A}_i \biggl[ (J_i)_{a_i,k} - \pi_\theta(a_i) (J_i)_{a_i,k} - \sum_{j \neq a_i} \pi_\theta(j) (J_i)_{j,k} \biggr] \tag{expanding the first product}\\
&= r_i \hat{A}_i \biggl[ (J_i)_{a_i,k} - \sum_{j=1}^{d_{\text{vocab}}} \pi_\theta(j) (J_i)_{j,k} \biggr] \tag{merging into a full sum over $j$}\\
&= r_i \hat{A}_i \Bigl( (J_i)_{a_i,k} - \mathbb{E}_{j \sim \pi_\theta}\bigl[ (J_i)_{j,k} \bigr] \Bigr).
\label{eq:component}
\end{align}

\emph{Squaring and applying an elementary inequality.} For any real numbers $a, b$, expanding $(a-b)^2 = a^2 - 2ab + b^2$ and applying the AM-GM inequality $-2ab \leq a^2 + b^2$ yields $(a-b)^2 \leq 2a^2 + 2b^2$. Applying this to Eq.~\eqref{eq:component},
\begin{equation}
\bigl( (J_i^\top E_i)_k \bigr)^2 \;\leq\; 2 r_i^2 \hat{A}_i^2 \Bigl[ (J_i)_{a_i,k}^2 + \bigl( \mathbb{E}_{j \sim \pi_\theta}[(J_i)_{j,k}] \bigr)^2 \Bigr].
\label{eq:square_split_F}
\end{equation}

\emph{Summing over $k$ and decomposing.} Summing Eq.~\eqref{eq:square_split_F} over $k = 1, \ldots, d_{\text{model}}$,
\begin{equation}
\|J_i^\top E_i\|_2^2 = \sum_{k=1}^{d_{\text{model}}} \bigl( (J_i^\top E_i)_k \bigr)^2 \;\leq\; 2 r_i^2 \hat{A}_i^2 \biggl[ \underbrace{\sum_{k=1}^{d_{\text{model}}} (J_i)_{a_i,k}^2}_{\text{Term I}} + \underbrace{\sum_{k=1}^{d_{\text{model}}} \bigl( \mathbb{E}_{j \sim \pi_\theta}[(J_i)_{j,k}] \bigr)^2}_{\text{Term II}} \biggr].
\label{eq:two_terms}
\end{equation}
We bound each term using Assumption~\ref{ass:jacobian}.

\emph{Bound on Term I.} By definition of the row $\ell_2$ norm,
\begin{equation}
\text{Term I} = \sum_{k=1}^{d_{\text{model}}} (J_i)_{a_i,k}^2 = \|(J_i)_{a_i,:}\|_2^2 \;\leq\; C,
\label{eq:term1_bound}
\end{equation}
where the inequality applies the sampled-token bound in Assumption~\ref{ass:jacobian}(ii).

\emph{Bound on Term II.} Applying Jensen's inequality coordinate-wise (the function $x \mapsto x^2$ is convex on $\mathbb{R}$),
\begin{equation}
\bigl( \mathbb{E}_{j \sim \pi_\theta}[(J_i)_{j,k}] \bigr)^2 \;\leq\; \mathbb{E}_{j \sim \pi_\theta}\bigl[ (J_i)_{j,k}^2 \bigr], \quad \forall k.
\label{eq:jensen_pointwise}
\end{equation}
Summing Eq.~\eqref{eq:jensen_pointwise} over $k$ and exchanging the finite sum with the expectation by linearity,
\begin{align}
\text{Term II} 
&\leq \sum_{k=1}^{d_{\text{model}}} \mathbb{E}_{j \sim \pi_\theta}\bigl[ (J_i)_{j,k}^2 \bigr] \tag{by Eq.~\eqref{eq:jensen_pointwise}}\\
&= \mathbb{E}_{j \sim \pi_\theta}\biggl[ \sum_{k=1}^{d_{\text{model}}} (J_i)_{j,k}^2 \biggr] \tag{linearity of expectation; finite sum}\\
&= \mathbb{E}_{j \sim \pi_\theta}\bigl[ \|(J_i)_{j,:}\|_2^2 \bigr] \tag{definition of row $\ell_2$ norm}\\
&\leq C,
\label{eq:term2_bound}
\end{align}
where the final inequality applies the expected logit sensitivity bound in Assumption~\ref{ass:jacobian}(i).

\emph{Combining Term I and Term II.} Substituting Eq.~\eqref{eq:term1_bound} and Eq.~\eqref{eq:term2_bound} into Eq.~\eqref{eq:two_terms},
\begin{equation}
\|J_i^\top E_i\|_2^2 \;\leq\; 2 r_i^2 \hat{A}_i^2 \cdot (C + C) \;=\; 4 C \cdot r_i^2 \hat{A}_i^2.
\label{eq:JtE_bound}
\end{equation}
Substituting Eq.~\eqref{eq:JtE_bound} into Eq.~\eqref{eq:int_factorization} and applying Lemma~\ref{lem:activation_bound}'s upper bound $\|x_i^{\text{int}}\|_2^2 \leq \beta_{\max} d_{\text{model}}$,
\begin{equation}
\|G_i^{\text{int}}\|_F^2 \;\leq\; 4 C \cdot r_i^2 \hat{A}_i^2 \cdot \beta_{\max} \cdot d_{\text{model}} \;=\; 4 \beta_{\max} C \cdot r_i^2 \hat{A}_i^2 \cdot d_{\text{model}}.
\label{eq:int_upper_F}
\end{equation}

\textbf{Step 3: Combining the bounds.} 
The hypothesis $\pi_\theta(a_i) < 1$ ensures $(1 - \pi_\theta(a_i))^2 > 0$, so $\frac{1}{\alpha_{\min} (1-\pi_\theta(a_i))^2}$ is well-defined and strictly positive. Multiplying Eq.~\eqref{eq:lm_lower_F} by $\frac{4 \beta_{\max} C / \alpha_{\min}}{(1-\pi_\theta(a_i))^2}$ on both sides yields
\begin{equation}
\frac{4 \beta_{\max} C}{\alpha_{\min}} \cdot \frac{1}{(1-\pi_\theta(a_i))^2} \cdot \|G_i^{\text{lm}}\|_F^2 \;\geq\; 4 \beta_{\max} C \cdot r_i^2 \hat{A}_i^2 \cdot d_{\text{model}}.
\end{equation}
Combining with Eq.~\eqref{eq:int_upper_F},
\begin{equation}
\|G_i^{\text{int}}\|_F^2 \;\leq\; \frac{4 \beta_{\max} C}{\alpha_{\min}} \cdot \frac{1}{(1-\pi_\theta(a_i))^2} \cdot \|G_i^{\text{lm}}\|_F^2.
\label{eq:absolute_form}
\end{equation}
Under the hypothesis $\hat{A}_i \neq 0$, we have $\|G_i^{\text{lm}}\|_F^2 > 0$ (since Eq.~\eqref{eq:lm_lower_F} gives a strictly positive lower bound when $\hat{A}_i \neq 0$, $r_i > 0$, and $\pi_\theta(a_i) < 1$). Dividing both sides of Eq.~\eqref{eq:absolute_form} by $\|G_i^{\text{lm}}\|_F^2$,
\begin{equation}
\frac{\|G_i^{\text{int}}\|_F^2}{\|G_i^{\text{lm}}\|_F^2} \;\leq\; \frac{4 \beta_{\max} C}{\alpha_{\min}} \cdot \frac{1}{(1-\pi_\theta(a_i))^2}.
\end{equation}
Identifying $\mathcal{C}_{\text{struct}} \triangleq 4 \beta_{\max} C / \alpha_{\min}$ completes the proof.
\end{proof}

\subsection{Empirical Measurements of Architectural Constants}
\label{app:empirical_table}
The Quantitative instantiation in 
Section~\ref{subsubsec:lm_head_sensitivity} relies on direct 
measurements of the architectural constants entering 
Theorem~\ref{thm:asymmetry}; this appendix provides the 
detailed measurement protocol and per-model values, summarized 
in Table~\ref{tab:empirical_constants}. We measure all three 
architectural quantities entering the bound---the activation 
norm ratio $\beta_{\max}/\alpha_{\min}$, the Jacobian row 
energy $C$, and the composite constant 
$\mathcal{C}_{\text{struct}} = 4\beta_{\max}C/\alpha_{\min}$---across 
ten mainstream LLMs spanning the Llama-3, Qwen-2.5, Qwen-3, and 
GPT-OSS families (1.5B--235B parameters, including both dense 
and MoE architectures). Every constant is collected via 
forward-pass activations and Jacobian estimates on a held-out 
validation set; given the well-known heavy-tailed distribution 
of gradient-related quantities in deep 
networks~\citep{simsekli2019tail,zhang2020gradient}, we report 
both the median (Med.) and the robust 95th-percentile worst 
case (95\%) for each quantity.

The empirical measurements reveal a consistent pattern across 
all tested architectures. The architectural constant 
$\mathcal{C}_{\text{struct}}$ takes small values at the median 
in every dense model, ranging from $1.7\times10^{-3}$ 
(Qwen2.5-7B) to $6.1\times10^{-2}$ (Qwen3-32B), and is even 
smaller for the MoE architecture Qwen3-235B-A22B 
($1.9\times10^{-6}$). Consequently, $\mathcal{C}_{\text{struct}}$---which 
serves as the prefactor of the bound on 
$\|G_i^{\text{int}}\|_F^2 / \|G_i^{\text{lm}}\|_F^2$ given by 
Eq.~\eqref{eq:main_bound_F}---remains below $10^{-1}$ at the 
median across all models, and below $1$ even at the 
95th-percentile worst case. These measurements quantitatively 
ground Theorem~\ref{thm:asymmetry}: the ratio 
$\|G_i^{\text{int}}\|_F^2 / \|G_i^{\text{lm}}\|_F^2$ is 
uniformly small across diverse LLM architectures and scales, 
fully consistent with the localized \texttt{lm\_head} 
divergence observed empirically in 
Section~\ref{subsec:dwd_empirical}.

\textbf{Stability of $\mathcal{C}_{\text{struct}}$ throughout RL training.} 
A natural concern is whether $\mathcal{C}_{\text{struct}}$ remains small 
throughout RL training. Figure~\ref{fig:cstruct_training_dynamics} tracks 
its 95th-percentile across three regimes: it stays well below $1$ and 
exhibits no surge even at the Naive Reuse collapse onset, confirming that 
Assumption~\ref{ass:jacobian} holds throughout training and that 
Theorem~\ref{thm:asymmetry}'s structural attenuation persists accordingly.

\begin{table}[htbp]
    \centering
    \caption{\textbf{Empirical measurements of architectural constants} entering Theorem~\ref{thm:asymmetry} across ten mainstream LLM architectures, grouped by model family. Each constant is reported as median (Med.) and 95th-percentile worst case (95\%). The final column---the architectural constant $\mathcal{C}_{\text{struct}} = 4 \beta_{\max} C / \alpha_{\min}$---is the prefactor of the inter-layer Frobenius gradient ratio bound given by Eq.~\eqref{eq:main_bound_F_appendix}, and remains below $10^{-1}$ at the median (and below 0.5 even at the 95th-percentile worst case) across all tested models.}
    \label{tab:empirical_constants}
    \resizebox{\textwidth}{!}{
    \begin{tabular}{l cc cc cc}
        \toprule
        \multirow{2}{*}{Model}
        & \multicolumn{2}{c}{$\beta_{\max}/\alpha_{\min}$} 
        & \multicolumn{2}{c}{$C$} 
        & \multicolumn{2}{c}{$\mathcal{C}_{\text{struct}}$} \\
        \cmidrule(lr){2-3} \cmidrule(lr){4-5} \cmidrule(lr){6-7}
        & Med. & 95\% & Med. & 95\% & Med. & 95\% \\
        \midrule
        Llama-3.2-3B-Inst.    & $3.74{\times}10^{-2}$ & $9.42{\times}10^{-2}$ & $8.39{\times}10^{-2}$ & $9.15{\times}10^{-1}$ & $1.26{\times}10^{-2}$ & $3.45{\times}10^{-1}$ \\
        Llama-3.1-8B-Inst.    & $2.48{\times}10^{-2}$ & $5.35{\times}10^{-2}$ & $1.50{\times}10^{-1}$ & $9.85{\times}10^{-1}$               & $1.49{\times}10^{-2}$ & $2.11{\times}10^{-1}$ \\
        \midrule
        Qwen2.5-1.5B-Inst.    & $1.97{\times}10^{-2}$ & $3.45{\times}10^{-2}$ & $1.64{\times}10^{-1}$ & $6.45{\times}10^{-1}$ & $1.29{\times}10^{-2}$ & $8.90{\times}10^{-2}$ \\
        Qwen2.5-7B-Inst.      & $1.35{\times}10^{-2}$ & $4.96{\times}10^{-2}$ & $3.12{\times}10^{-1}$ & $9.36{\times}10^{-1}$ & $1.68{\times}10^{-2}$ & $1.86{\times}10^{-1}$ \\
        Qwen2.5-72B-Inst.     & $1.23{\times}10^{-2}$ & $2.94{\times}10^{-2}$ & $1.07{\times}10^{-1}$ & $5.03{\times}10^{-1}$ & $5.26{\times}10^{-3}$ & $5.92{\times}10^{-2}$ \\
        \midrule
        Qwen3-4B-Inst-2507    & $1.63{\times}10^{-2}$ & $1.53{\times}10^{-1}$ & $7.56{\times}10^{-2}$ & $3.84{\times}10^{-1}$ & $4.93{\times}10^{-3}$ & $2.35{\times}10^{-1}$ \\
        Qwen3-8B              & $1.81{\times}10^{-2}$ & $2.28{\times}10^{-1}$ & $2.45{\times}10^{-1}$ & $4.02{\times}10^{-1}$ & $1.77{\times}10^{-2}$ & $3.67{\times}10^{-1}$ \\
        Qwen3-32B             & $1.99{\times}10^{-2}$ & $1.53{\times}10^{-1}$ & $7.61{\times}10^{-1}$ & $2.85{\times}10^{-1}$               & $6.06{\times}10^{-2}$ & $1.74{\times}10^{-1}$               \\
        Qwen3-235B-A22B      & $1.63{\times}10^{-2}$ & $5.51{\times}10^{-2}$ & $2.83{\times}10^{-5}$ & $1.58{\times}10^{-3}$ & $1.85{\times}10^{-6}$ & $3.48{\times}10^{-4}$ \\
        \midrule
        GPT-OSS-120b         & $2.76{\times}10^{-1}$ & $5.29{\times}10^{-1}$ & $2.91{\times}10^{-3}$ & $4.86{\times}10^{-3}$ & $3.21{\times}10^{-3}$ & $1.03{\times}10^{-2}$ \\
        \bottomrule
    \end{tabular}
    }
\end{table}

\begin{figure}[h]
    \centering\scriptsize\renewcommand\arraystretch{0.5}
    \setlength{\tabcolsep}{10pt}
    \begin{tabular}{cc}
    \includegraphics[width=0.4\linewidth]{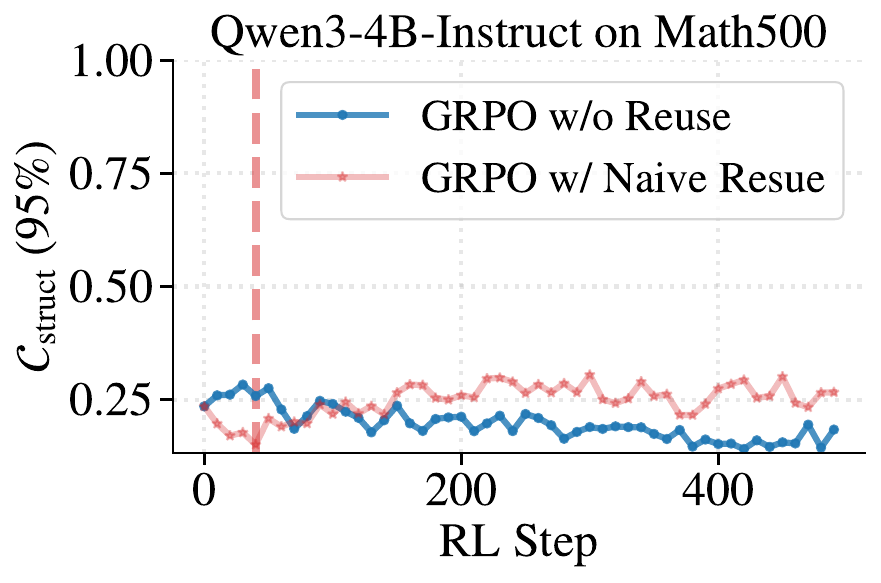}&
    \includegraphics[width=0.4\linewidth]{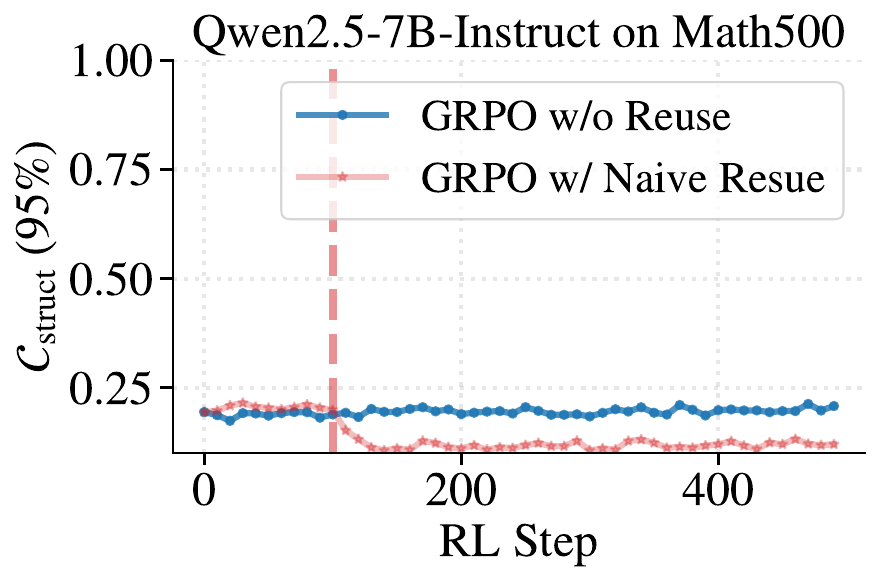}\\[-2pt]
    \end{tabular}
\caption{\textbf{Dynamics of $\mathcal{C}_{\text{struct}}$ (95th-percentile) 
during RL training} on Qwen3-4B-Instruct and Qwen2.5-7B-Instruct (Math500). 
Red dashed lines mark the onset of performance degradation under Naive Reuse. 
Observation: \emph{$\mathcal{C}_{\text{struct}}$ remains stable across all 
regimes, including at collapse onset.}}
\label{fig:cstruct_training_dynamics}
\end{figure}

\section{Detailed Proofs for 
Section~\ref{subsubsec:chi2_divergence}}
\label{app:proof_chi2_bound}

This appendix provides detailed proofs for all theoretical results presented in 
Section~\ref{subsubsec:chi2_divergence}. To make this section self-contained, we restate each result prior to its proof. Throughout the appendix, we retain the notation introduced in the main text
\subsection{Proof of Lemma~\ref{lem:grad_bound} (Gradient 
Norm in Terms of Importance-Ratio Second Moment)}
\label{app:proof_grad_bound}

\begin{lemma}[Restatement of Lemma~\ref{lem:grad_bound}]
Let $\overline{r^2} \triangleq \frac{1}{T}\sum_{i=1}^T r_i^2$ 
denote the empirical second moment of importance ratios. Then
\begin{equation}
\|G^{\text{lm}}\|_F^2 \;\leq\; c_{\max} \, \overline{r^2}, 
\qquad \text{where} \quad c_{\max} \triangleq \max_{1 \leq i 
\leq T} \hat{A}_i^2 \, \| e_{a_i} - \pi_\theta(\cdot \mid 
h_{L,i}) \|_2^2 \, \| h_{L,i} \|_2^2.
\end{equation}
\end{lemma}

\begin{proof}
The proof proceeds in three steps: (1) decomposing the squared 
gradient norm via Jensen's inequality, (2) factoring per-token 
contributions into rank-1 norms, and (3) absorbing the 
representation-dependent factors into $c_{\max}$.

\textbf{Step 1: Jensen's inequality on the squared norm.}
By Proposition~\ref{prop:grad_decomp}, the batch-level 
\texttt{lm\_head} gradient admits the rank-1 sum decomposition:
\begin{equation}
G^{\text{lm}} = \frac{1}{T}\sum_{i=1}^{T} E_i \, h_{L,i}^\top, 
\qquad 
E_i = r_i \hat{A}_i \bigl(e_{a_i} - \pi_\theta(\cdot \mid 
h_{L,i})\bigr).
\end{equation}
The squared Frobenius norm $\|\cdot\|_F^2$ is convex on matrix 
space. Applying Jensen's inequality with the uniform 
distribution over $\{1, \ldots, T\}$:
\begin{equation}
\|G^{\text{lm}}\|_F^2 = \left\| \frac{1}{T}\sum_{i=1}^T E_i 
h_{L,i}^\top \right\|_F^2 \;\leq\; \frac{1}{T}\sum_{i=1}^T 
\|E_i h_{L,i}^\top\|_F^2.
\label{eq:jensen_step}
\end{equation}

\textbf{Step 2: Rank-1 norm factorization.}
For any vectors $u \in \mathbb{R}^m, v \in \mathbb{R}^n$, the 
outer product satisfies $\|uv^\top\|_F^2 = \|u\|_2^2 \|v\|_2^2$. 
Applying this identity to each rank-1 term in 
Eq.~\eqref{eq:jensen_step}:
\begin{equation}
\|E_i h_{L,i}^\top\|_F^2 = \|E_i\|_2^2 \cdot \|h_{L,i}\|_2^2.
\end{equation}
Substituting the explicit form of $E_i$:
\begin{equation}
\|E_i\|_2^2 = r_i^2 \hat{A}_i^2 \|e_{a_i} - \pi_\theta(\cdot 
\mid h_{L,i})\|_2^2.
\end{equation}
Combining these:
\begin{equation}
\|G^{\text{lm}}\|_F^2 \;\leq\; \frac{1}{T} \sum_{i=1}^T r_i^2 
\hat{A}_i^2 \|e_{a_i} - \pi_\theta(\cdot \mid h_{L,i})\|_2^2 
\|h_{L,i}\|_2^2.
\label{eq:per_token_decomp}
\end{equation}

\textbf{Step 3: Absorbing representation factors into 
$c_{\max}$.}
By the definition of $c_{\max}$, the per-token factor 
involving $\hat{A}_i^2$, $\|e_{a_i}-\pi_\theta\|^2$, and 
$\|h_{L,i}\|^2$ is bounded by $c_{\max}$ uniformly across 
$i \in \{1,\ldots,T\}$:
\begin{equation}
\hat{A}_i^2 \|e_{a_i} - \pi_\theta(\cdot \mid h_{L,i})\|_2^2 
\|h_{L,i}\|_2^2 \;\leq\; c_{\max}, \quad \forall i.
\end{equation}
Substituting into Eq.~\eqref{eq:per_token_decomp}:
\begin{equation}
\|G^{\text{lm}}\|_F^2 \;\leq\; \frac{c_{\max}}{T} \sum_{i=1}^T 
r_i^2 = c_{\max} \cdot \overline{r^2},
\end{equation}
which establishes the claimed bound.
\end{proof}

\subsection{Proof of Theorem~\ref{thm:chi2_bound} 
(\texttt{lm\_head} Gradient Energy Lower-Bounds Policy 
Divergence)}
\label{app:proof_chi2}

\begin{theorem}[Restatement of Theorem~\ref{thm:chi2_bound}]
Let $\hat{\chi}^2 \triangleq \frac{1}{T}\sum_{i=1}^T (r_i^2 - 1)$ denote the empirical estimator of the Pearson $\chi^2$ divergence $\chi^2(\pi_\theta \| \pi_{\text{old}})$. Then
\begin{equation}
\|G^{\text{lm}}\|_F^2 \;\leq\; c_{\max} \bigl( 1 + 
\hat{\chi}^2 \bigr),
\label{eq:chi2_bound_app}
\end{equation}
or equivalently, $\hat{\chi}^2 \geq \|G^{\text{lm}}\|_F^2 / 
c_{\max} - 1$.
\end{theorem}

\begin{proof}
The proof reduces to a direct algebraic decomposition of the 
empirical second moment $\overline{r^2}$. By the definition of 
$\hat{\chi}^2$:
\begin{equation}
\hat{\chi}^2 = \frac{1}{T}\sum_{i=1}^T (r_i^2 - 1) = \frac{1}{T} 
\sum_{i=1}^T r_i^2 - 1 = \overline{r^2} - 1,
\end{equation}
which gives the identity $\overline{r^2} = 1 + \hat{\chi}^2$. 
Substituting this into the bound from 
Lemma~\ref{lem:grad_bound}:
\begin{equation}
\|G^{\text{lm}}\|_F^2 \;\leq\; c_{\max} \overline{r^2} = 
c_{\max} (1 + \hat{\chi}^2),
\end{equation}
which establishes Eq.~\eqref{eq:chi2_bound_app}. Rearranging 
this inequality (well-defined since $c_{\max} > 0$):
\begin{equation}
\hat{\chi}^2 \;\geq\; \frac{\|G^{\text{lm}}\|_F^2}
{c_{\max}} - 1,
\end{equation}
which establishes the equivalent lower-bound form. The two 
forms are mathematically equivalent.
\end{proof}

\textbf{Remark on the empirical $\chi^2$ statistic.}
The quantity $\hat{\chi}^2 = \frac{1}{T}\sum_i (r_i^2 - 1)$ is 
the natural batch-level analog of the Pearson $\chi^2$ 
divergence: under the relation $\mathbb{E}_{a \sim 
\pi_{\text{old}}}[r^2] - 1 = \chi^2(\pi_\theta \| 
\pi_{\text{old}})$ for the population divergence, $\hat{\chi}^2$ 
is the empirical estimator obtained by replacing the population 
expectation with the batch sample mean. We work with 
$\hat{\chi}^2$ throughout, since the bound in 
Theorem~\ref{thm:chi2_bound} holds deterministically at the 
batch level without invoking any concentration argument.

\section{Pseudocode of Dynamic Gradient Gating (DGG)}
\label{app:algorithm}
We provide the full pseudocode of GRPO with Dynamic Gradient Gating (DGG) in Algorithm~\ref{alg:dgg}. DGG integrates as a lightweight wrapper around the standard sample-reuse loop, requiring only a single Frobenius-norm computation per step and no modification to the underlying objective or optimizer. The Z-score test (Eq.~\ref{eq:zscore}) is computed at every reuse step, and gating activates only at $k > 1$ to target reuse-induced surges.

\begin{algorithm}[h]
\caption{GRPO with Dynamic Gradient Gating (DGG)}
\label{alg:dgg}
\begin{algorithmic}[1]
\Require Initial policy $\pi_\theta$, reuse factor $K$, threshold $\tau$
\For{each training iteration}
    \State Generate rollout batch $\mathcal{D}$ from $\pi_{\text{old}} \leftarrow \pi_\theta$
    \For{$k = 1, \dots, K$}
        \State Compute GRPO loss $\mathcal{L}_{\text{GRPO}}$ on $\mathcal{D}$ via Eq.~\eqref{eqn:raw_obj} and backpropagate
        \State $g_t \leftarrow \|G^{\text{lm}}\|_F^2$;\quad $\Delta g_t \leftarrow g_t - g_{t-1}$ \hfill {\color{gray}\# current observation}
        \State $z_t \leftarrow (\Delta g_t - \mu_t) / (\sigma_t + \varepsilon)$ \hfill {\color{gray}\# tested against historical stats}
        \If{$k > 1$ \textbf{and} $z_t > \tau$}
            \State Zero all gradients; \textbf{break} \hfill {\color{gray}\# discard \& terminate reuse}
        \EndIf
        \State Step the Adam optimizer
        \State Update running statistics $(\mu_{t+1}, \sigma_{t+1})$ with $\Delta g_t$ \hfill {\color{gray}\# stats ready for step $t{+}1$}
    \EndFor
\EndFor
\end{algorithmic}
\end{algorithm}

\section{Broader Impacts}
\label{app:broader_impacts}

This work proposes Dynamic Gradient Gating (DGG), a lightweight 
training-stability mechanism that improves the sample efficiency of RLVR 
for LLMs. As a foundational methodological contribution, DGG does not 
introduce new models, datasets, or applications; its primary effect is to 
reduce the rollout and wall-clock cost of post-training, which lowers the 
energy consumption and computational barrier of RLVR research. By making 
RL post-training more accessible, DGG may help broaden participation 
beyond well-resourced labs.

Like any general-purpose training improvement, DGG inherits the broader 
societal considerations of the underlying LLM and RLVR pipeline: any 
downstream model trained with our method retains the safety, fairness, 
and misuse concerns of its base model and reward design. We do not see a 
direct path from DGG to specific negative applications beyond those 
already associated with LLM training in general, and we encourage 
practitioners to pair efficiency gains with the standard safeguards used 
in responsible LLM deployment.

\section{Empirical Validation of DWD Phenomenon on Additional LLMs}
\label{app:more_empirical_dwd}

In this section, we extend our profiling of DWD phenomenon  in Section~\ref{subsec:dwd_empirical} to additional LLMs and tasks. 
Figure~\ref{fig:empirical_dwd_qwen2.5-7b} reports the layer-wise weight dynamics of Qwen2.5-7B-Instruct across all four tasks, and Figure~\ref{fig:empirical_dwd_more_llm} reports the dynamics on mathematical reasoning across four additional LLMs from the Qwen-2.5, Qwen-3, and Llama-3 families (1.5B--8B). In all settings, sample reuse accelerates early convergence but is followed by severe performance degradation, synchronized with an abrupt, isolated surge in the \texttt{lm\_head}, while MLP and Attention layers remain stable. These results confirm DWD as a broadly present phenomenon across diverse architectures, scales, and tasks, supporting our theoretical analysis in Section~\ref{subsec:dwd_theory} and the design of DGG.

\vspace{2cm}

\begin{figure}[h]
    \centering\scriptsize\renewcommand\arraystretch{0.5}
    \setlength{\tabcolsep}{1pt}
    \begin{tabular}{@{}c@{\hspace{2.5pt}}ccc}
    &\textbf{Task Performance} & \textbf{Weight Change (w/o Reuse)} & \textbf{Weight Change (w/ Naive Reuse)} \\
    \noalign{\smallskip}
    \raisebox{1.0cm}{\rotatebox{90}{\textbf{Math}}}&
    \includegraphics[width=0.32\linewidth]{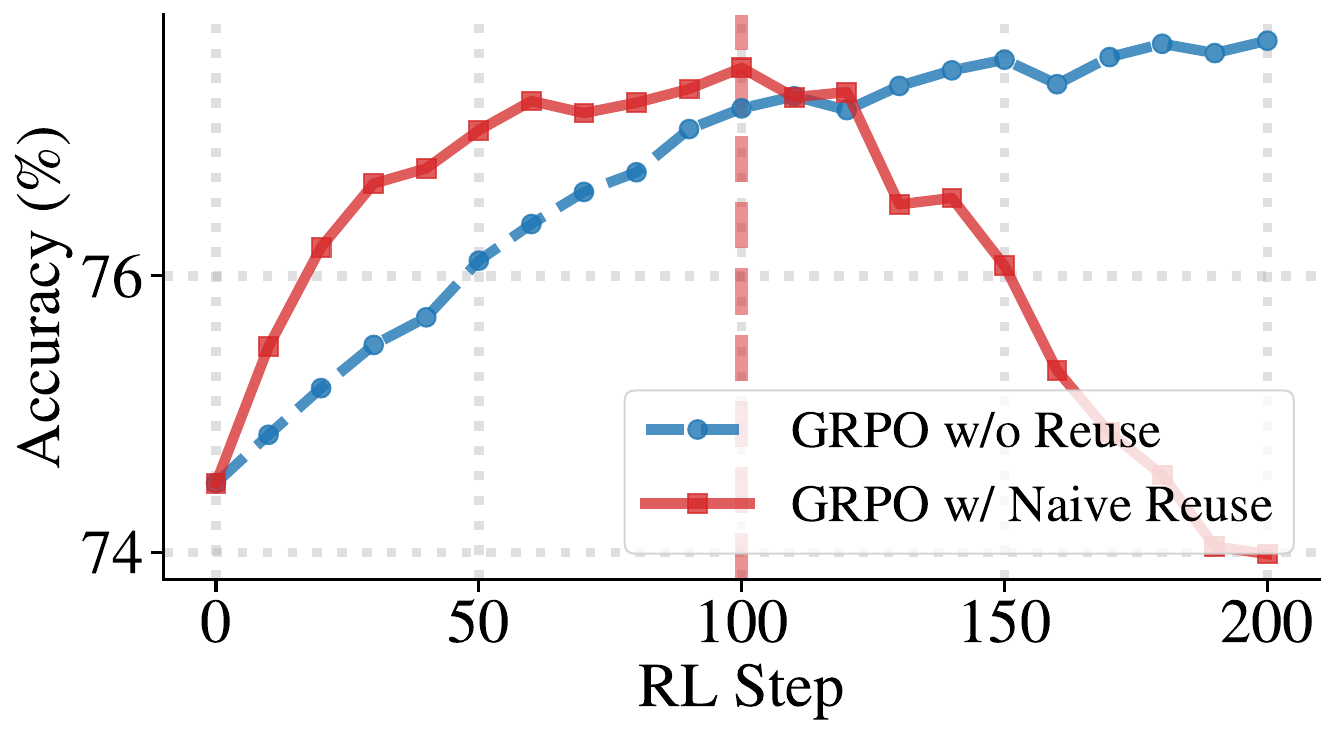}&
    \includegraphics[width=0.32\linewidth]{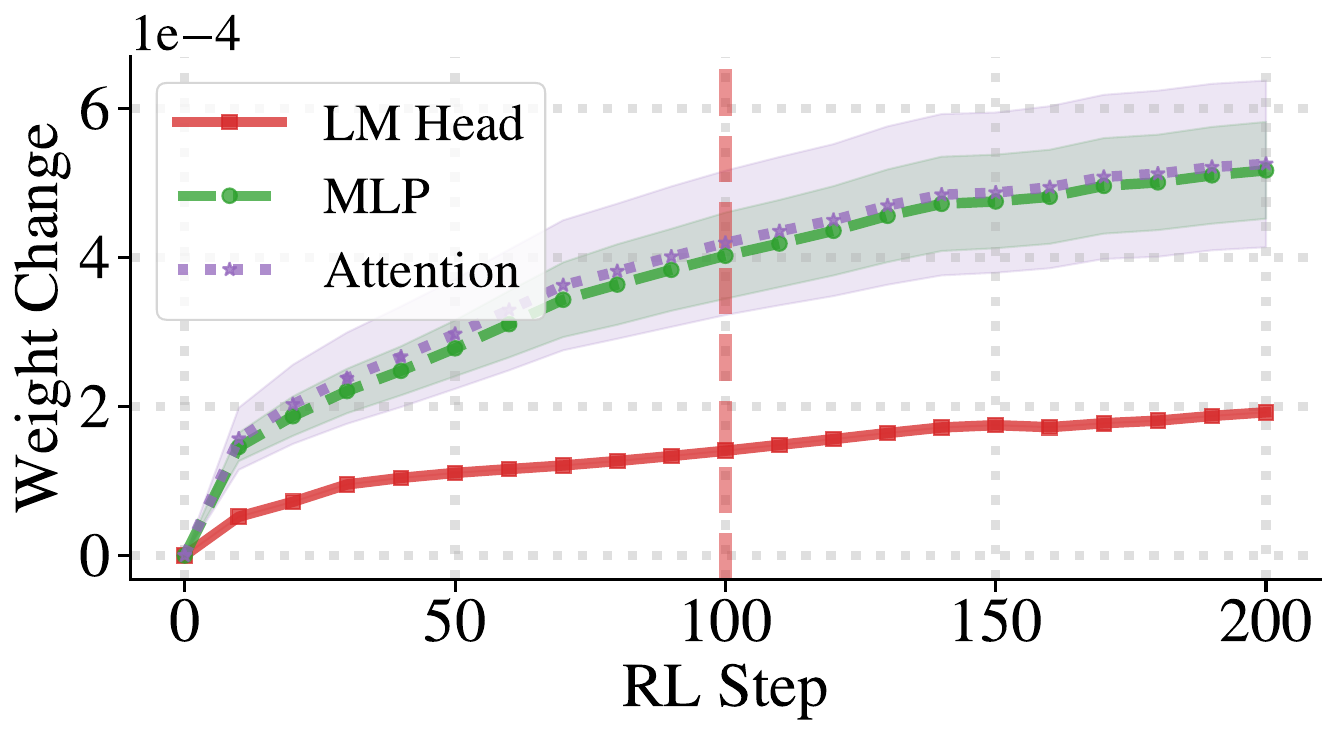}&
    \includegraphics[width=0.32\linewidth]{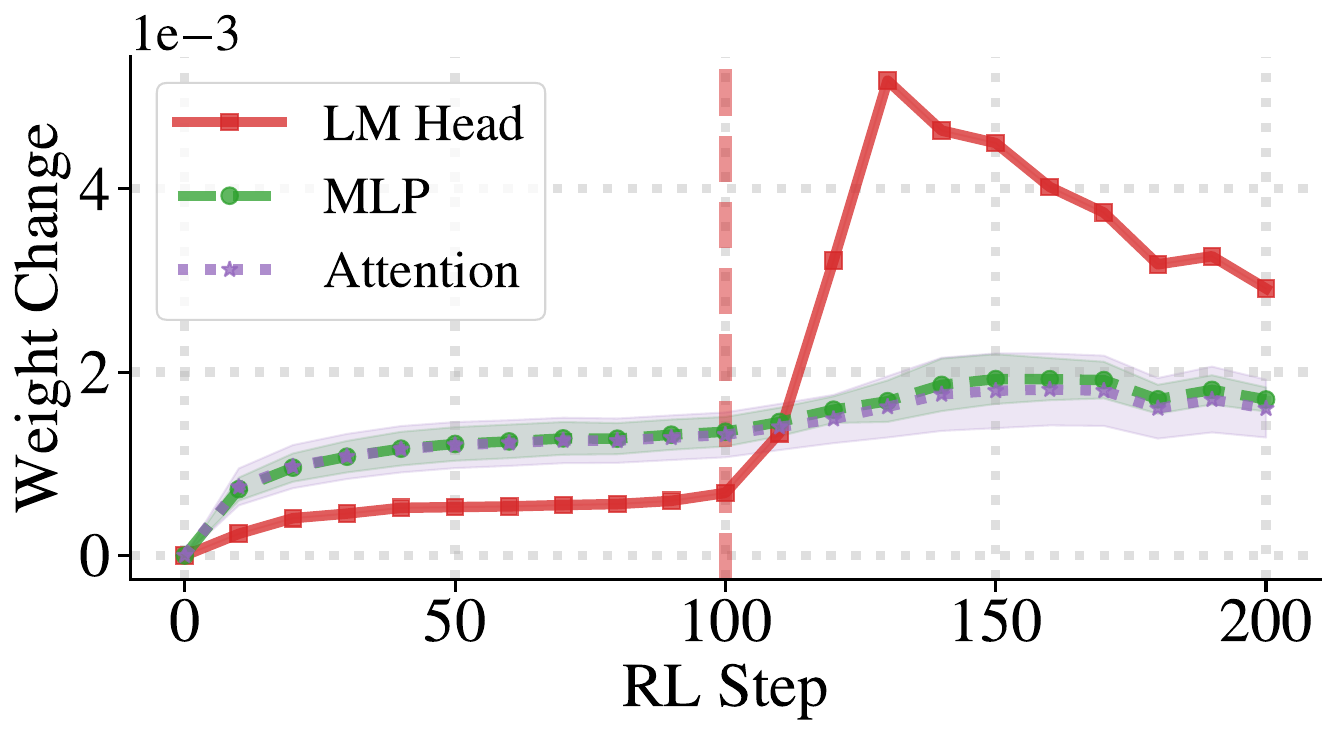}\\[-2pt]
    \raisebox{1.0cm}{\rotatebox{90}{\textbf{ALFWorld}}}&
    \includegraphics[width=0.32\linewidth]{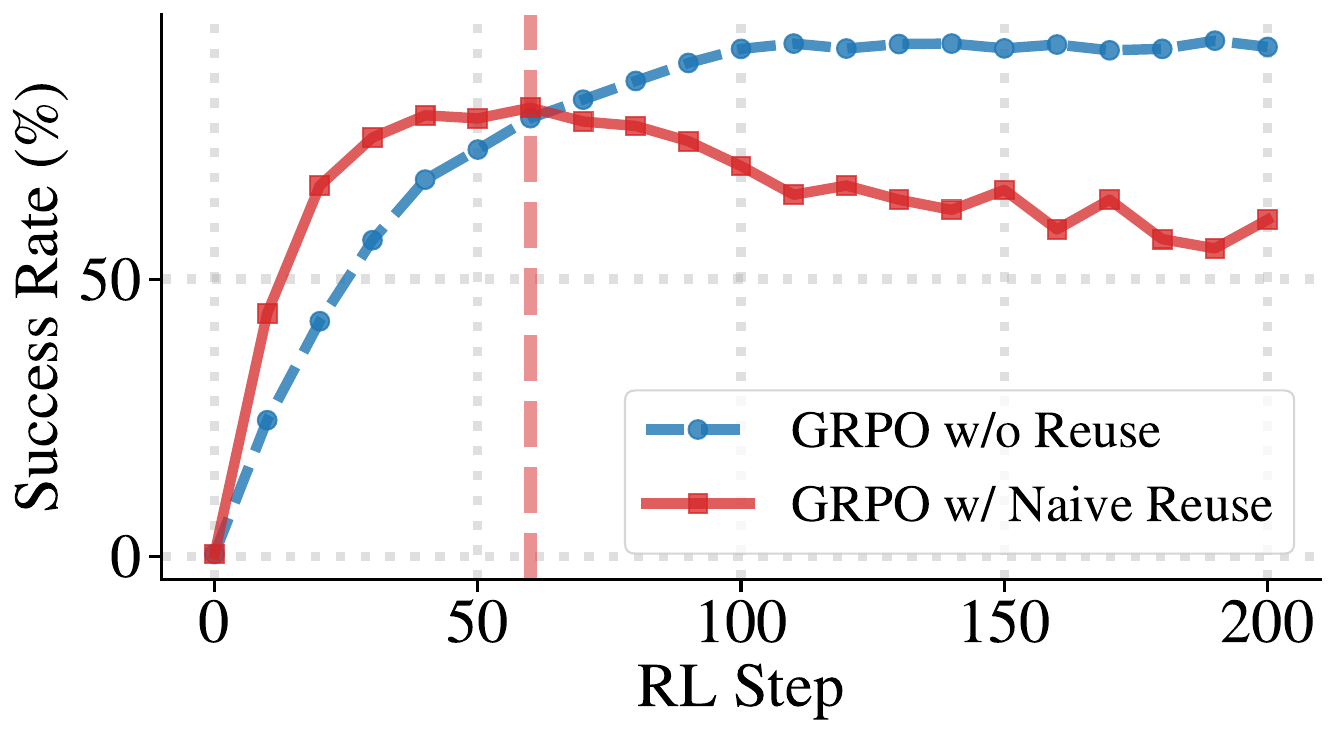}&
    \includegraphics[width=0.32\linewidth]{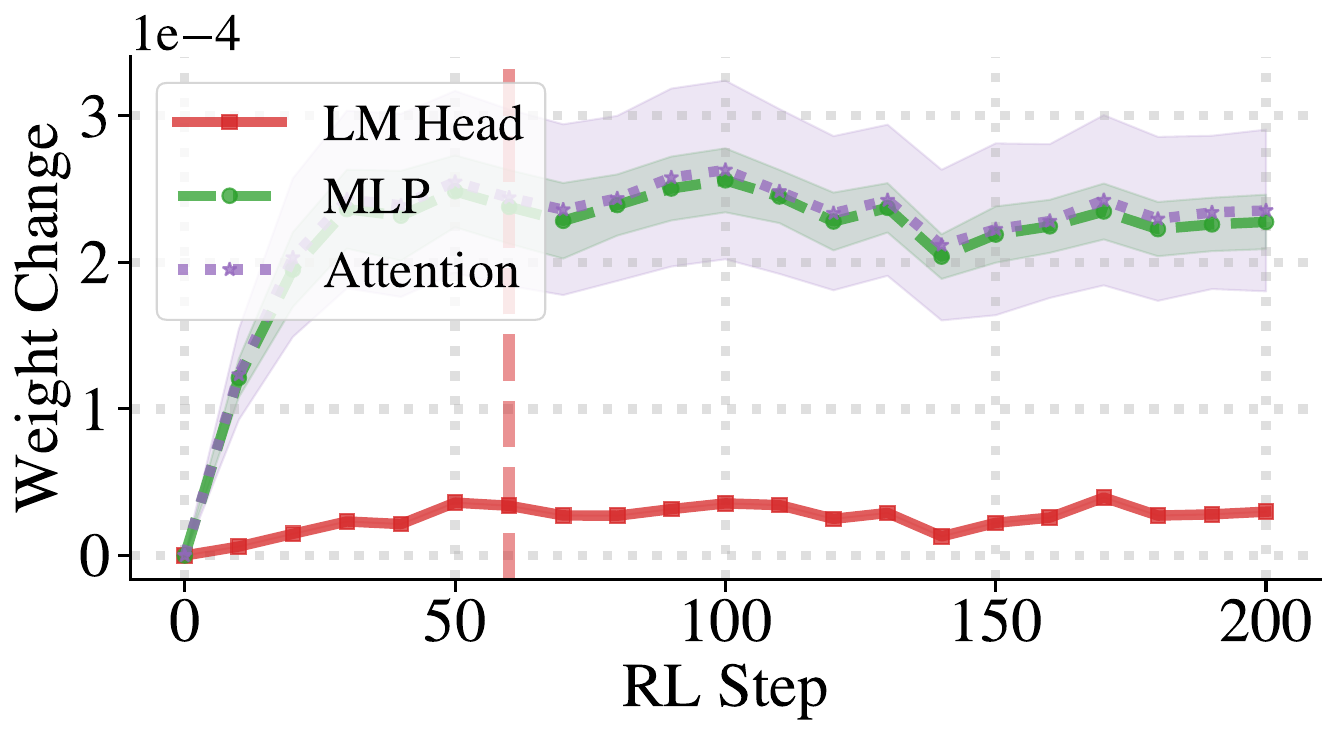}&
    \includegraphics[width=0.32\linewidth]{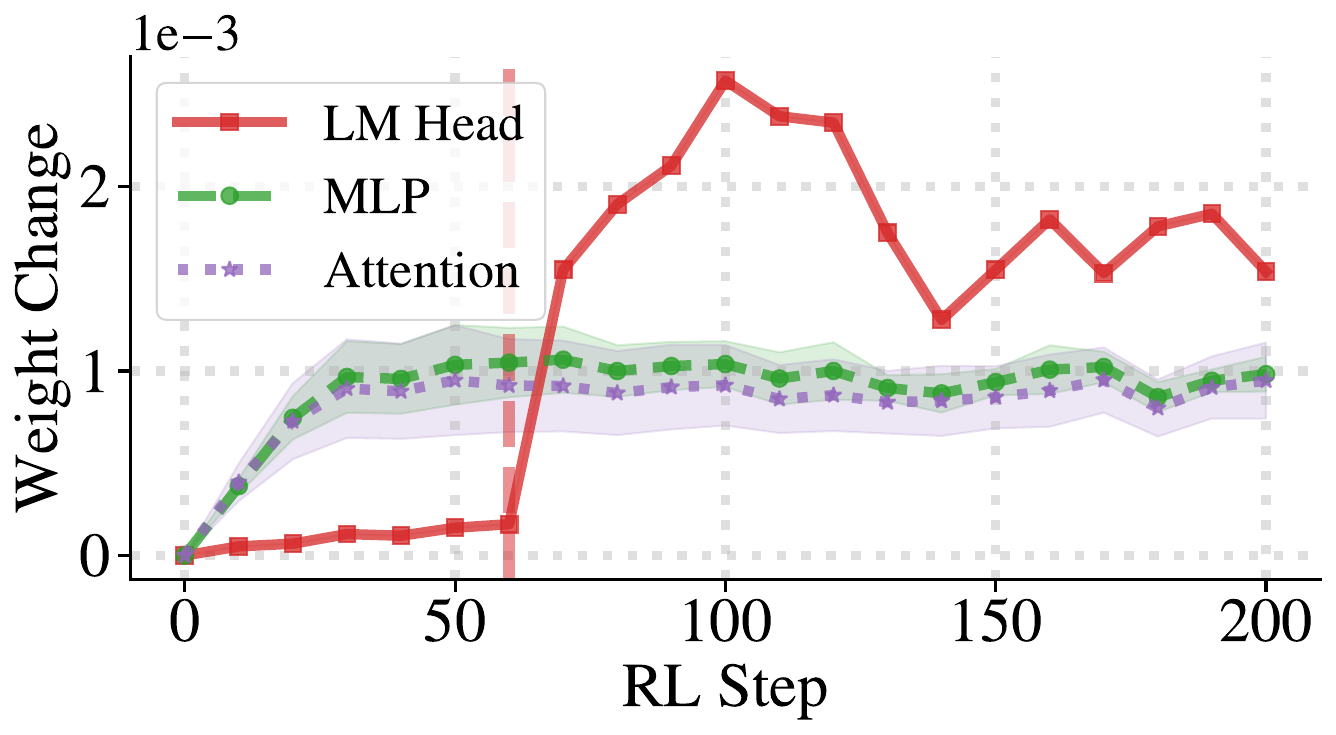}\\[-2pt]
    \raisebox{1.0cm}{\rotatebox{90}{\textbf{WebShop}}}&
    \includegraphics[width=0.32\linewidth]{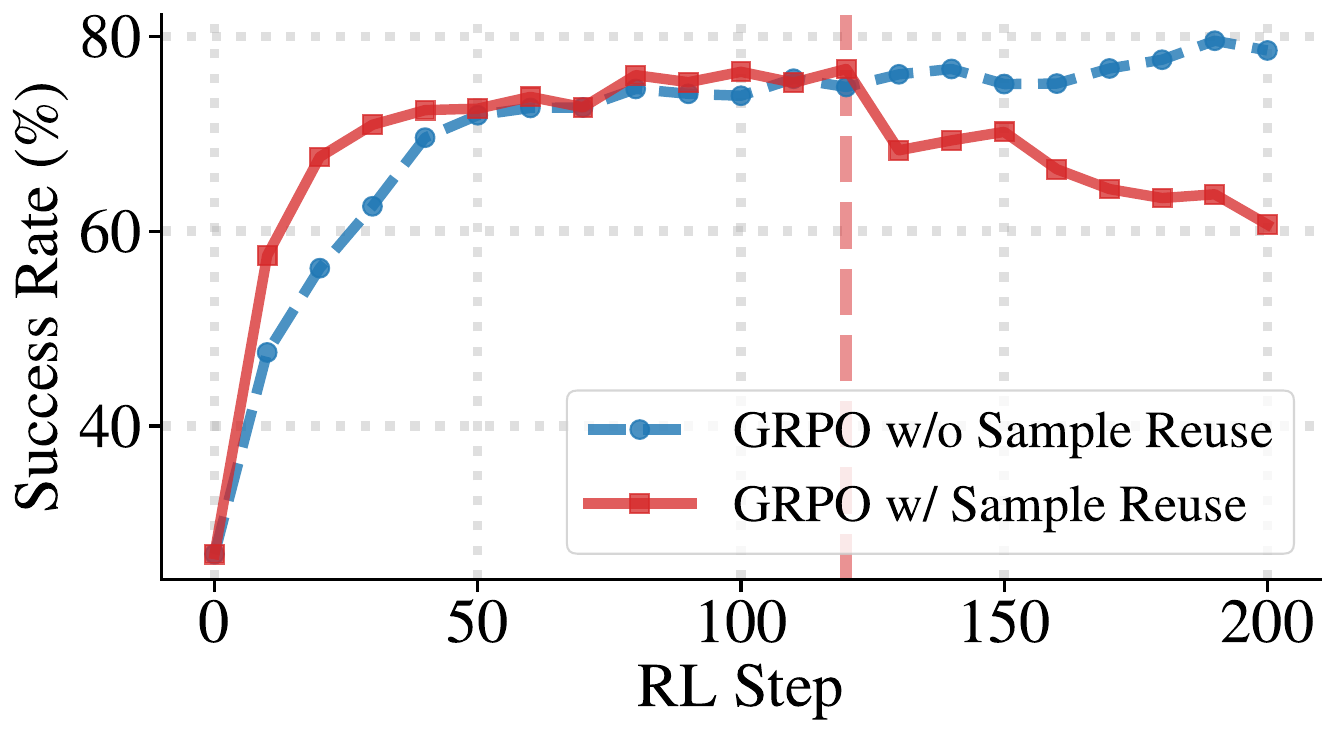}&
    \includegraphics[width=0.32\linewidth]{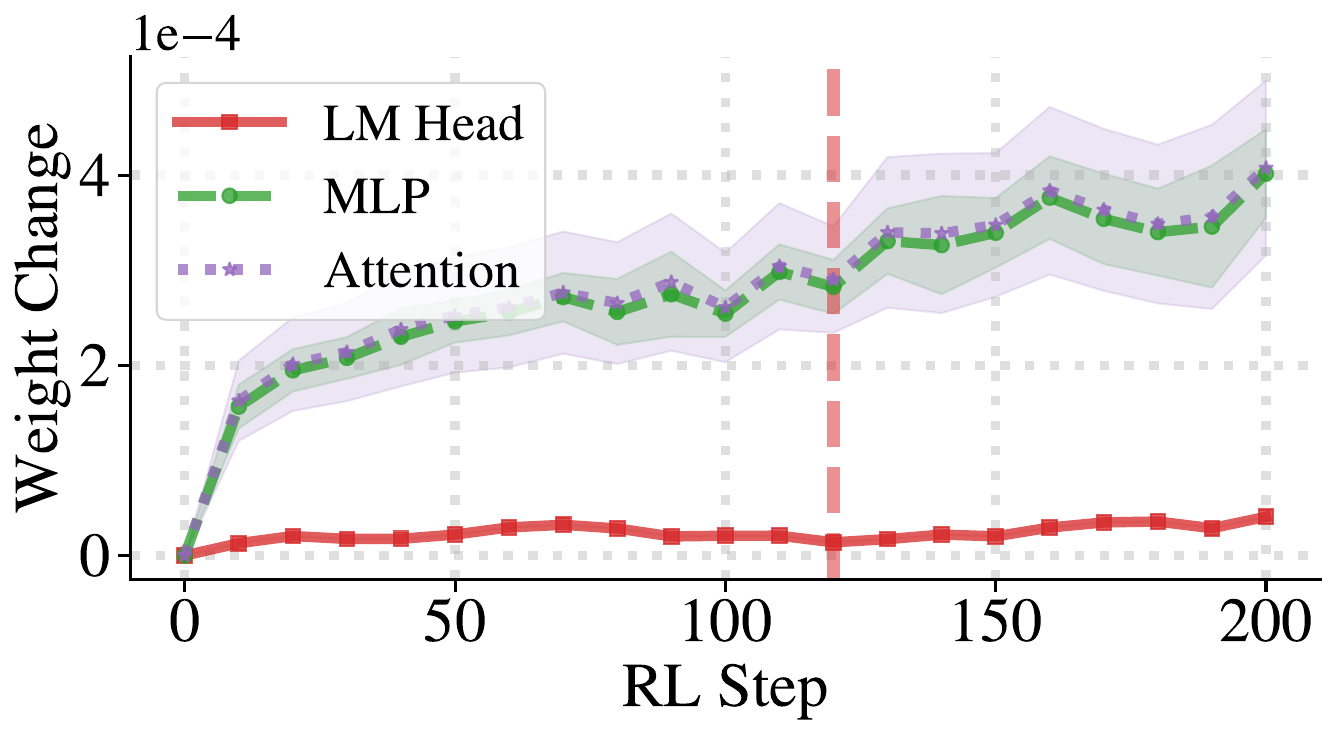}&
    \includegraphics[width=0.32\linewidth]{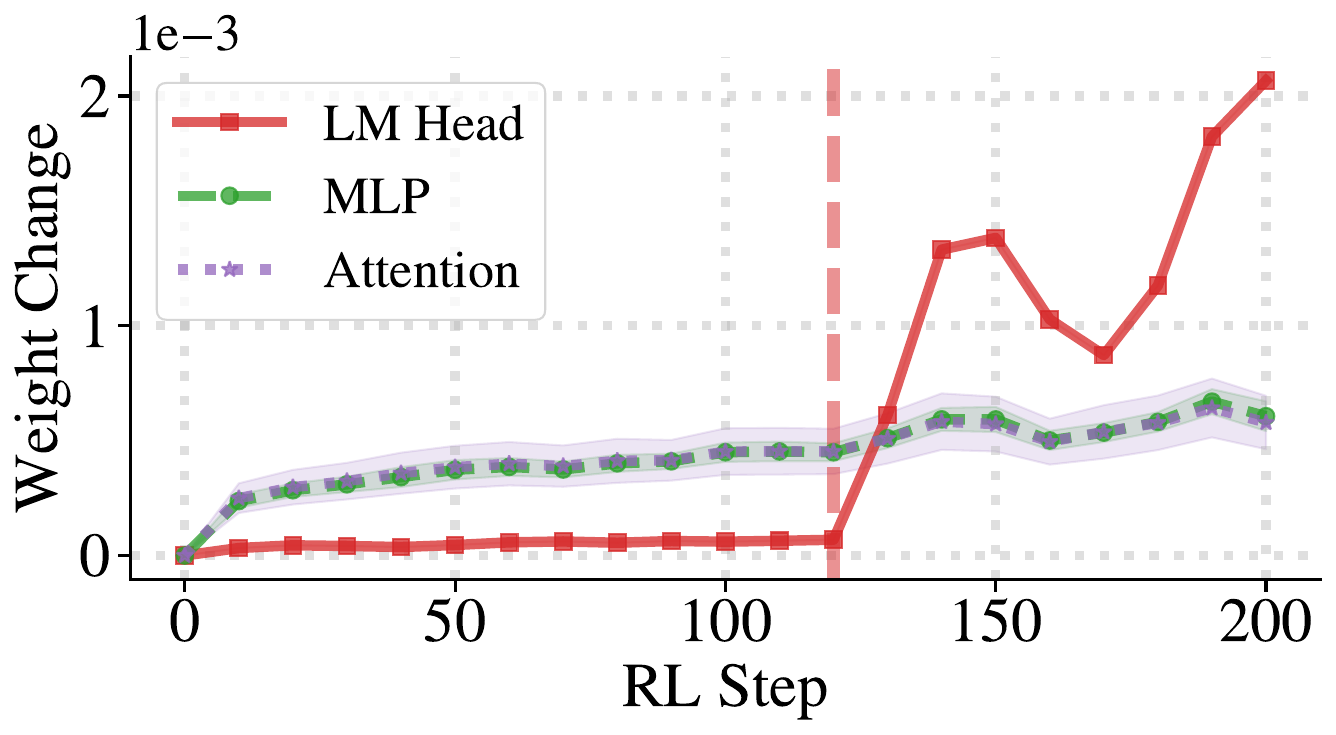}\\[-2pt]
    \raisebox{1.0cm}{\rotatebox{90}{\textbf{Search QA}}}&
    \includegraphics[width=0.32\linewidth]{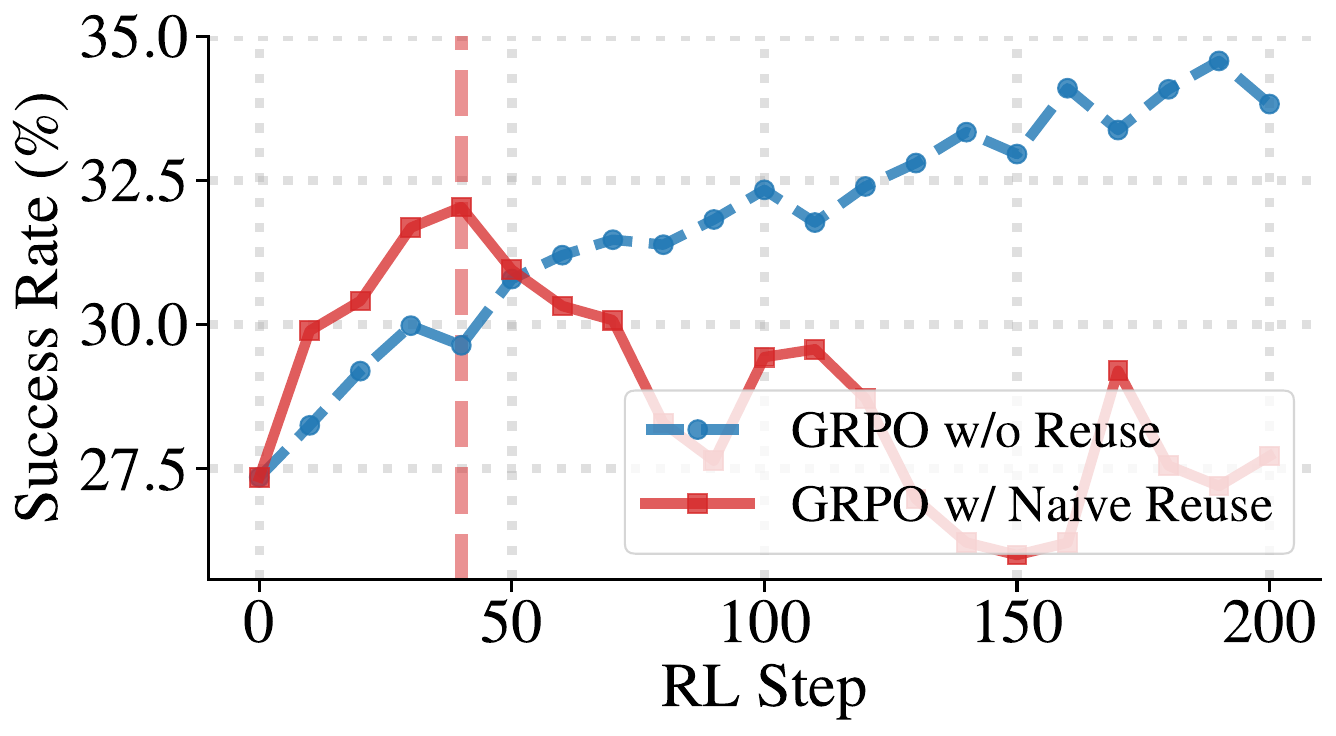}&
    \includegraphics[width=0.32\linewidth]{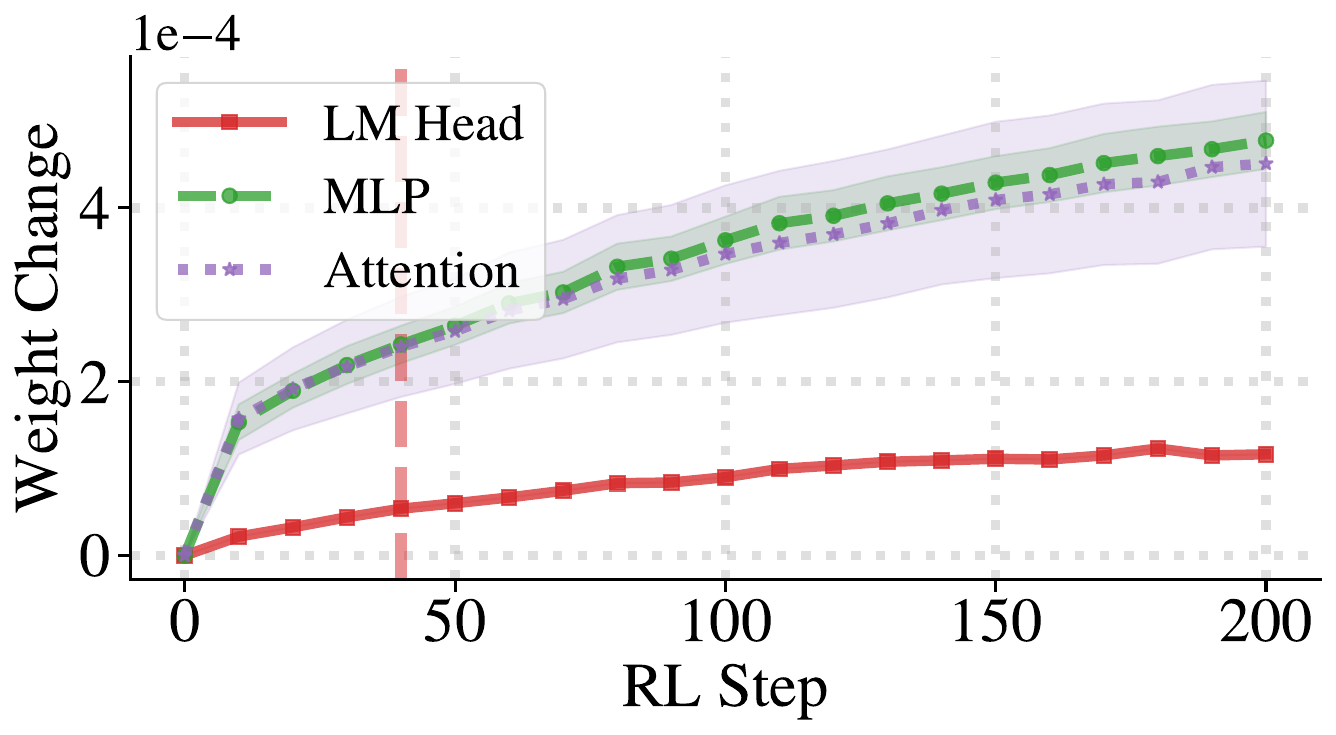}&
    \includegraphics[width=0.32\linewidth]{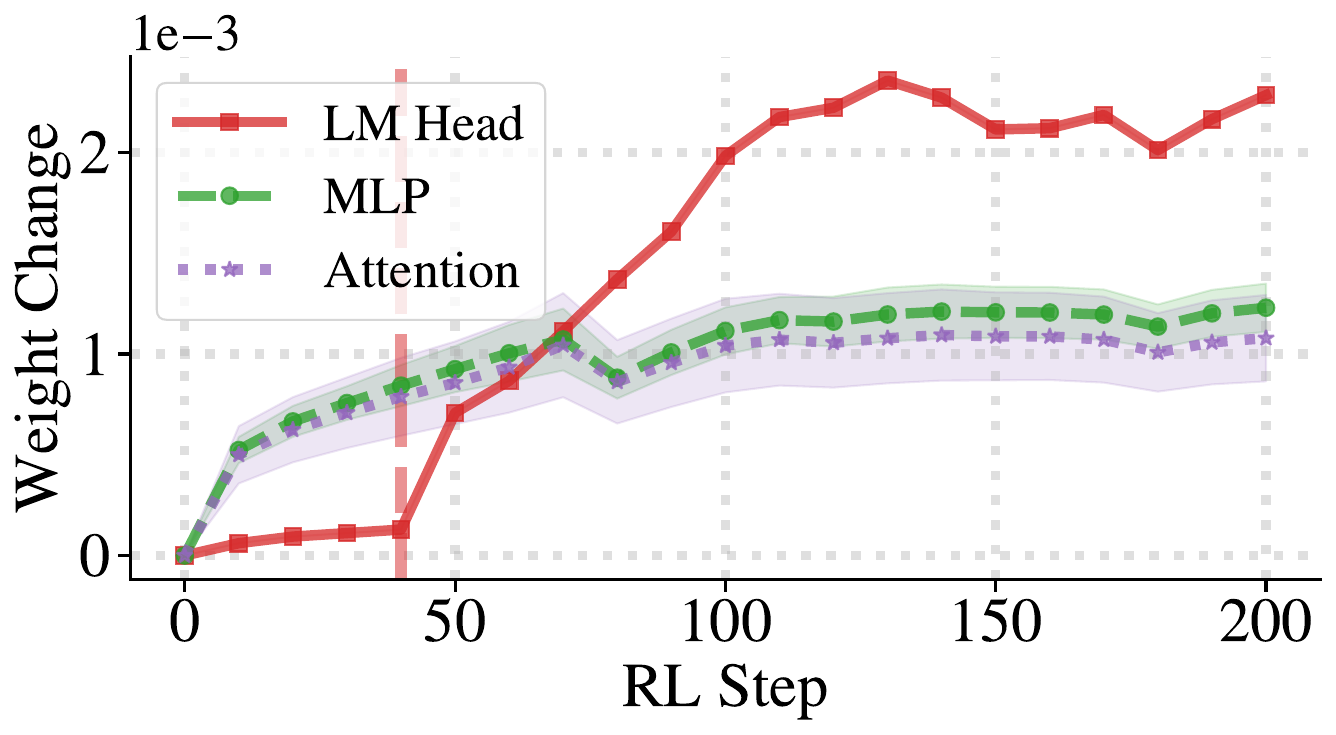}\\[-7pt]
    \end{tabular}
    \caption{\textbf{Illustration of the Disproportionate Weight Divergence (DWD) phenomenon} on Qwen2.5-7B-Instruct across four tasks. Relative weight change is defined as $\|W_t - W_{t-\Delta}\|_F / \|W_{\text{ref}}\|_F$ ($W_{\text{ref}}$: initial pretrained weight; $\Delta$: profiling interval, set to 10 RL steps); red dashed lines mark the onset of performance degradation for GRPO w/ Naive Reuse. Observation: \textit{Sample reuse accelerates early convergence, followed by a severe performance degradation synchronized with an abrupt, isolated surge in the LM Head — while intermediate layers (Attention and MLP) remain stable.}}
    \label{fig:empirical_dwd_qwen2.5-7b}
\end{figure}

\begin{figure}[]
    \centering\scriptsize\renewcommand\arraystretch{0.5}
    \setlength{\tabcolsep}{1pt}
    \begin{tabular}{@{}c@{\hspace{2.5pt}}ccc}
    &\textbf{Task Performance} & \textbf{Weight Change (w/o Reuse)} & \textbf{Weight Change (w/ Naive Reuse)} \\
    \noalign{\smallskip}
    \raisebox{0.5cm}{\rotatebox{90}{\textbf{Qwen2.5-1.5B-Ins.}}}&
    \includegraphics[width=0.32\linewidth]{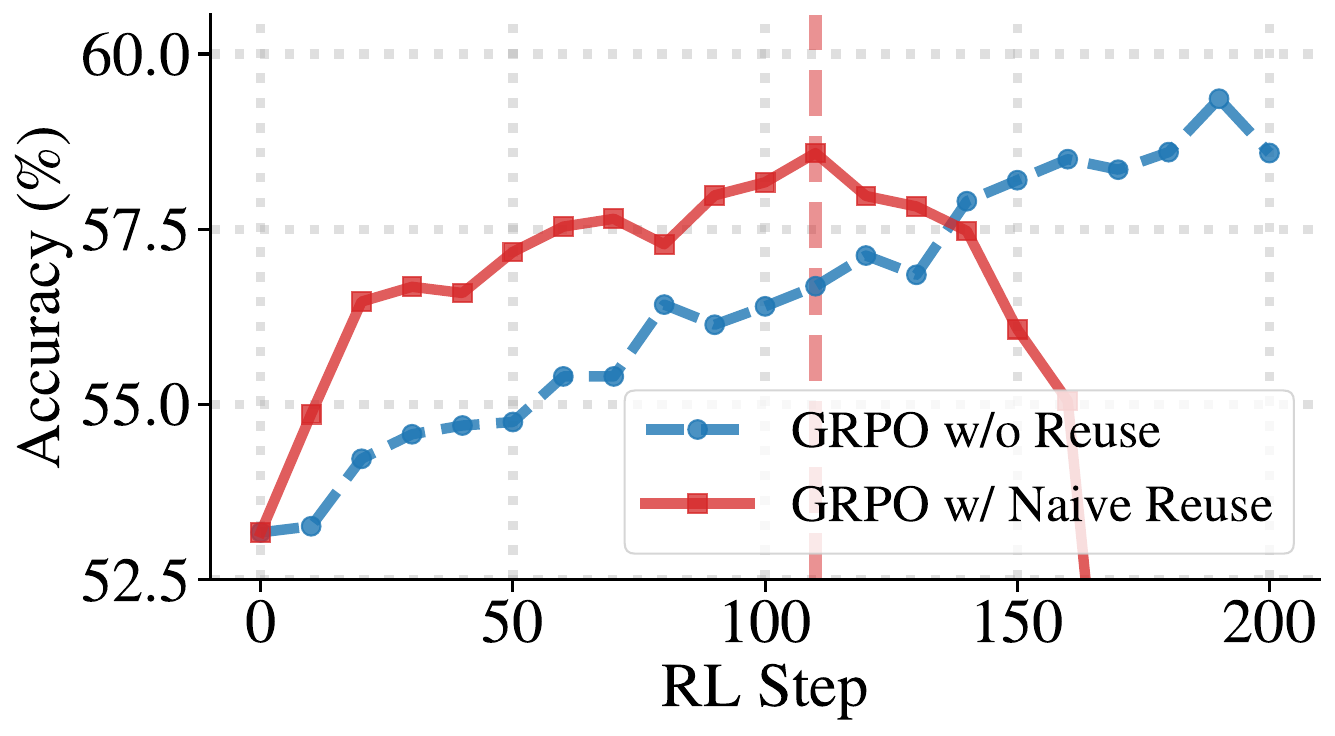}&
    \includegraphics[width=0.32\linewidth]{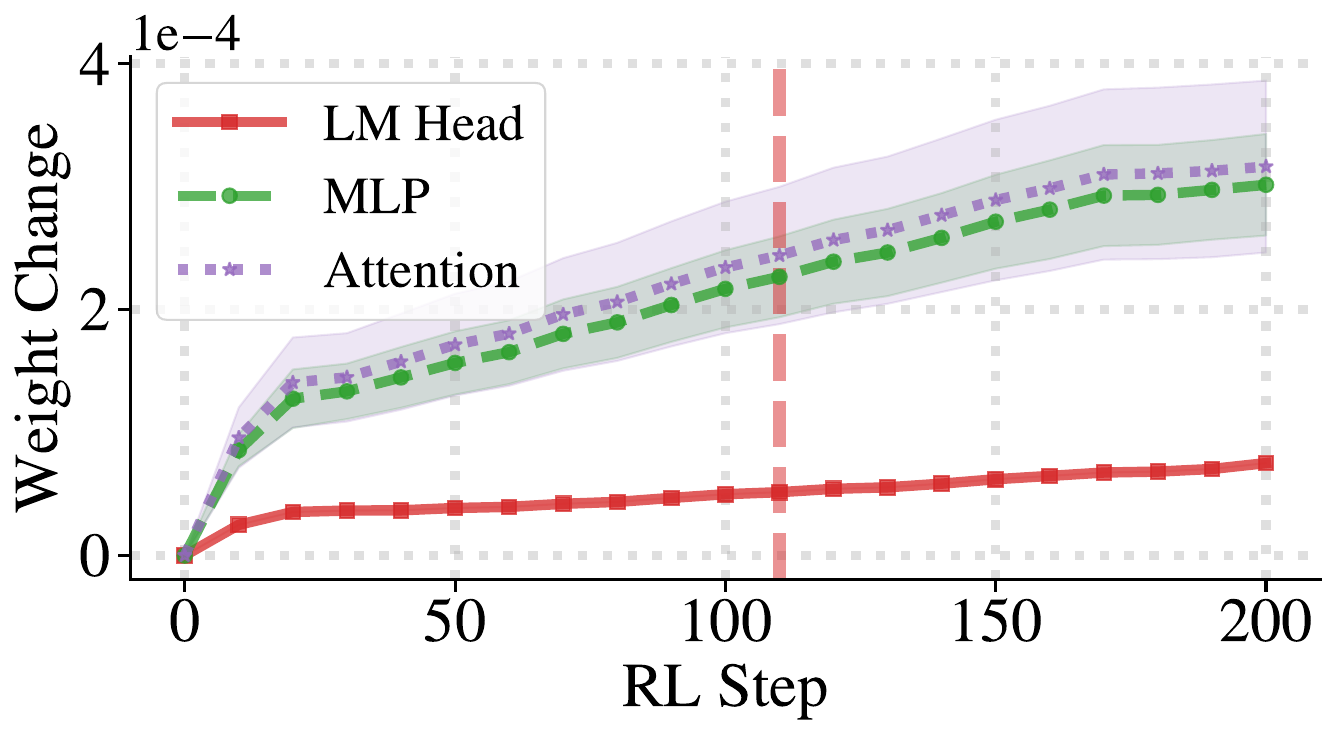}&
    \includegraphics[width=0.32\linewidth]{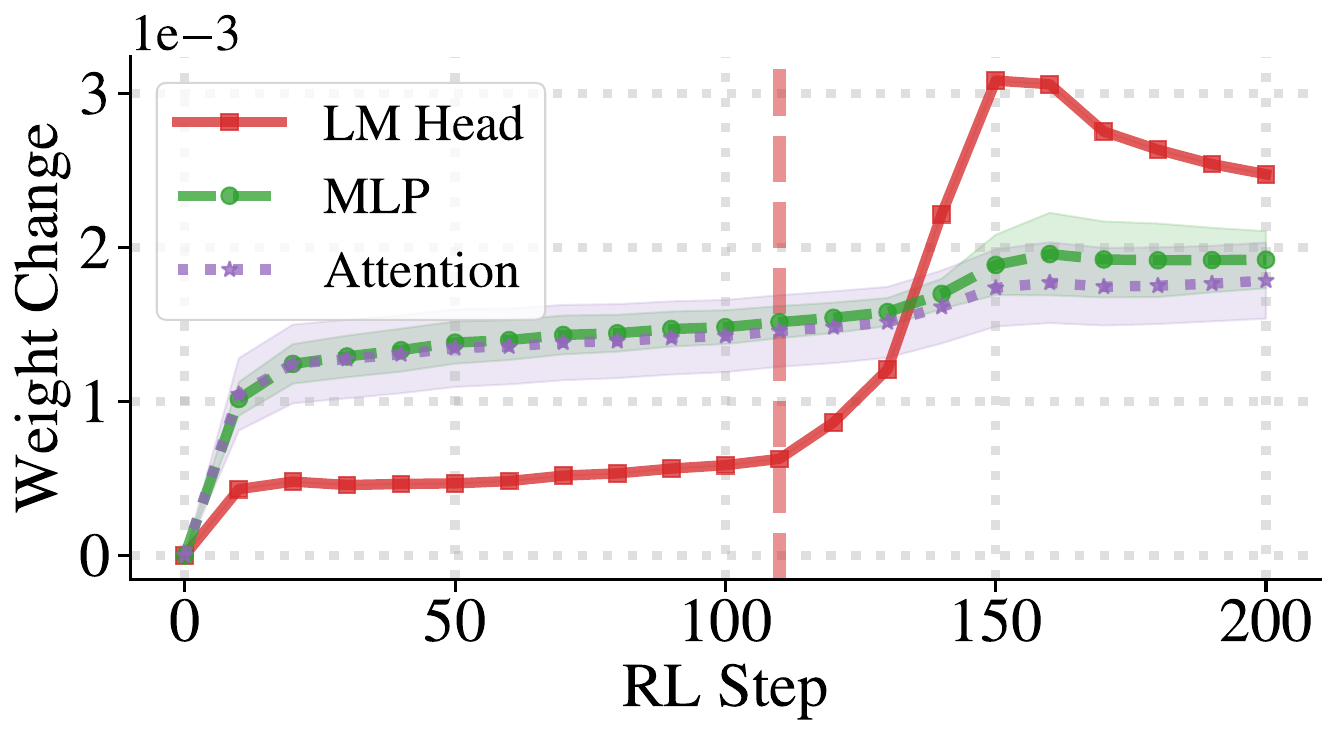}\\[-2pt]
    \raisebox{0.9cm}{\rotatebox{90}{\textbf{Qwen3-4B}}}&
    \includegraphics[width=0.32\linewidth]{figs/motivation/qwen3_4b_math500_performance.pdf}&
    \includegraphics[width=0.32\linewidth]{figs/motivation/qwen3_4b_math500_wc_wosr.pdf}&
    \includegraphics[width=0.32\linewidth]{figs/motivation/qwen3_4b_math500_wc_wsr.pdf}\\[-2pt]
    \raisebox{0.5cm}{\rotatebox{90}{\textbf{Llama3.2-3B-Ins.}}}&
    \includegraphics[width=0.32\linewidth]{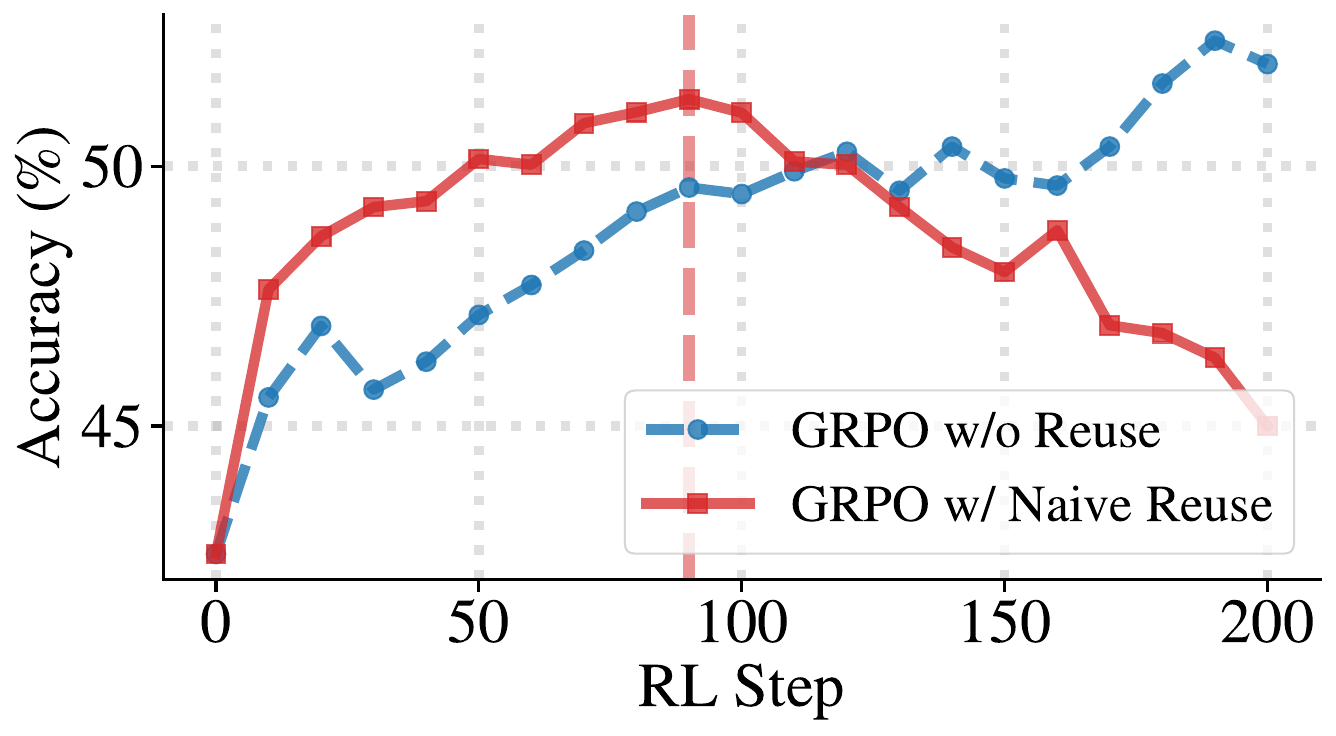}&
    \includegraphics[width=0.32\linewidth]{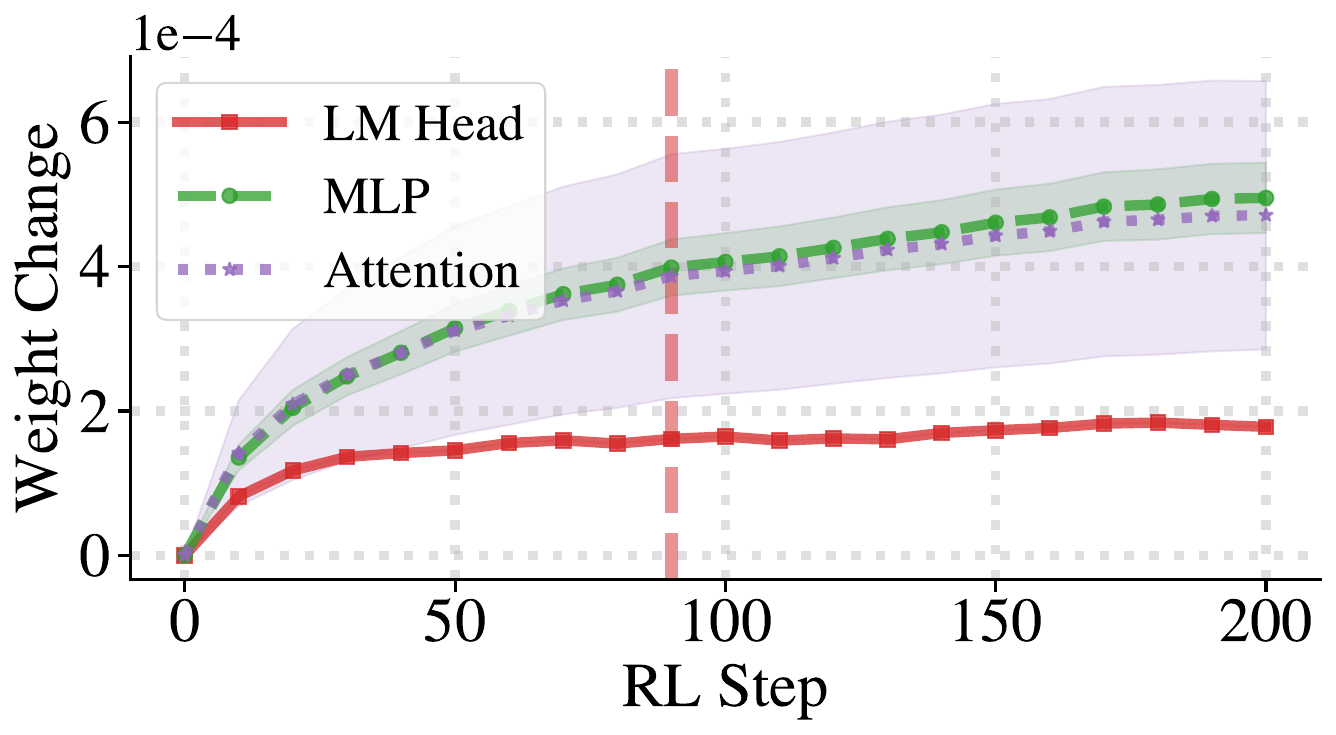}&
    \includegraphics[width=0.32\linewidth]{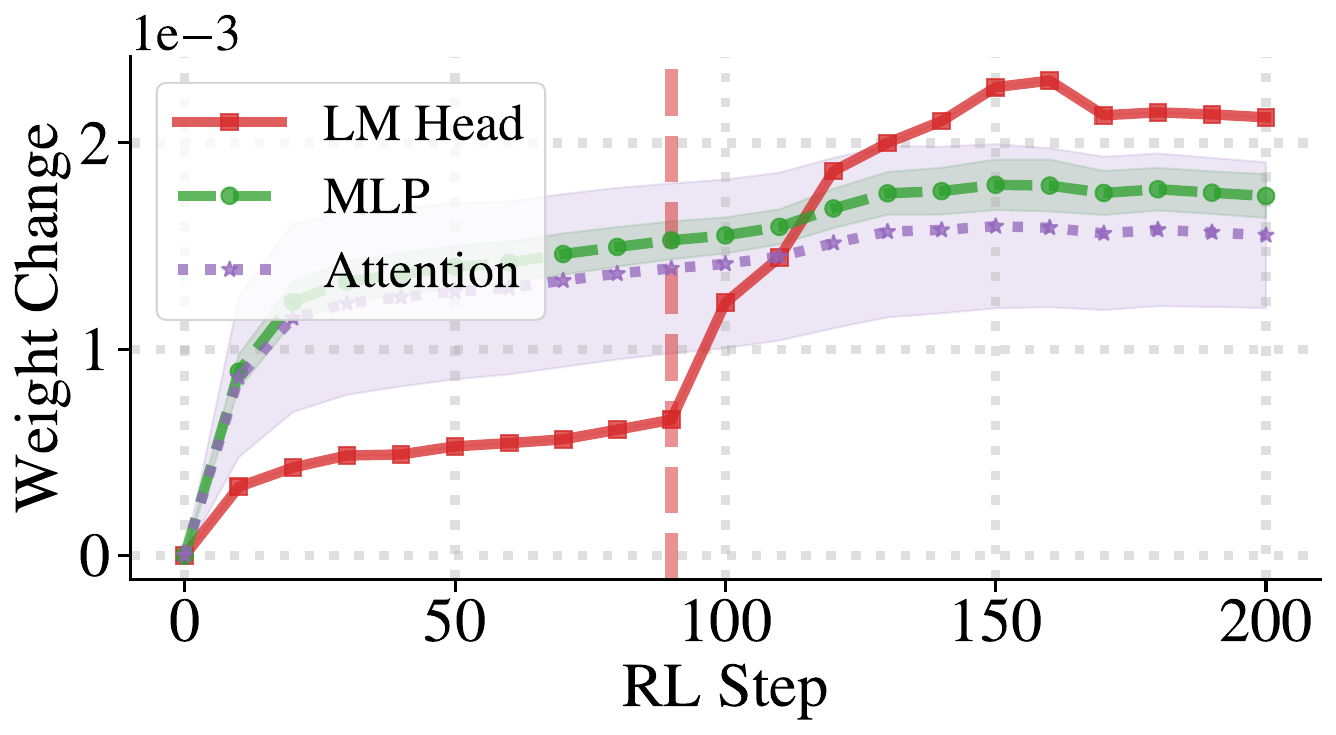}\\[-2pt]
    \raisebox{0.5cm}{\rotatebox{90}{\textbf{Llama3.1-8B-Ins.}}}&
    \includegraphics[width=0.32\linewidth]{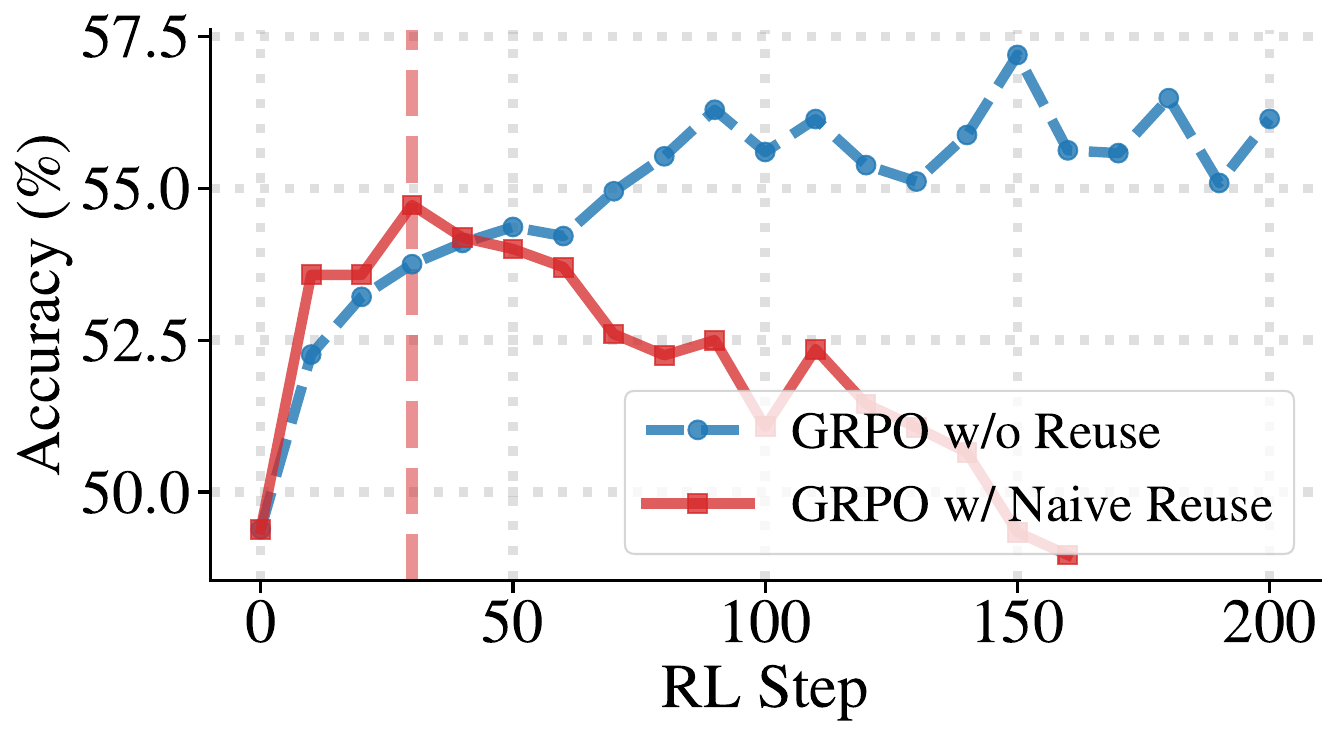}&
    \includegraphics[width=0.32\linewidth]{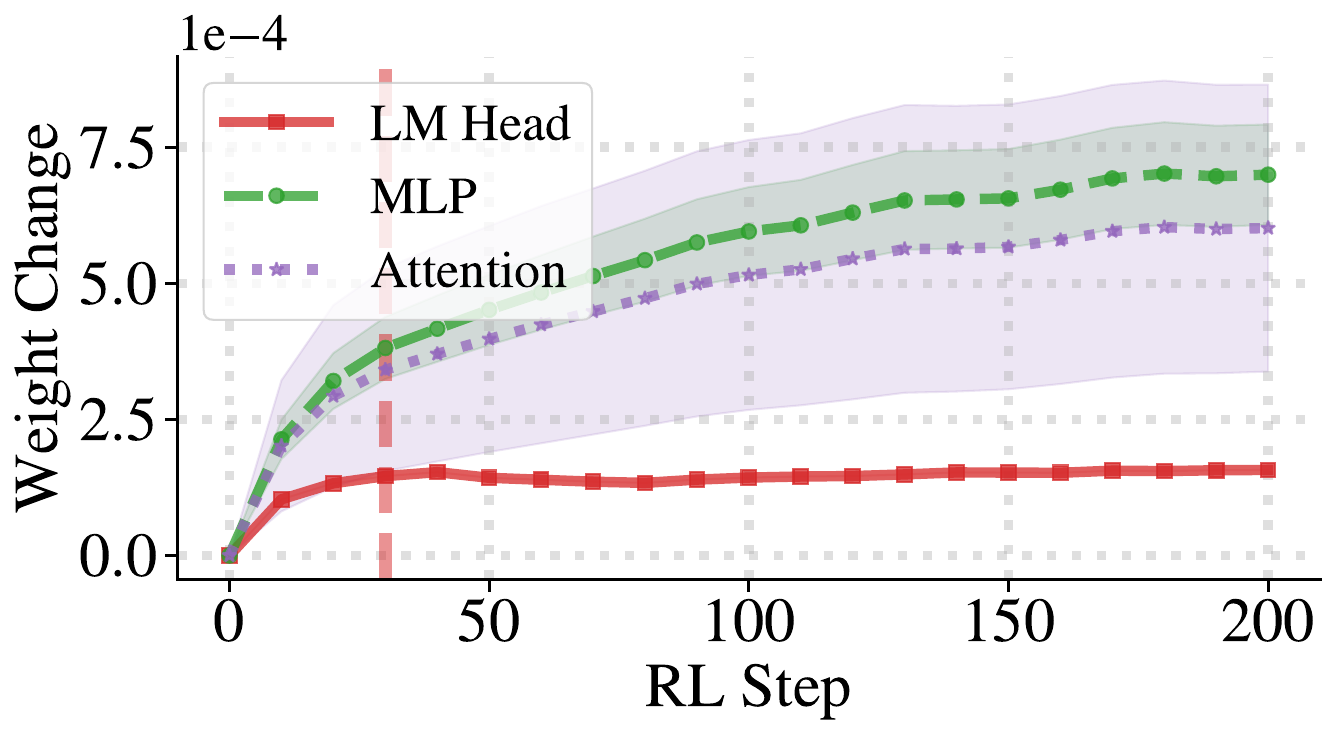}&
    \includegraphics[width=0.32\linewidth]{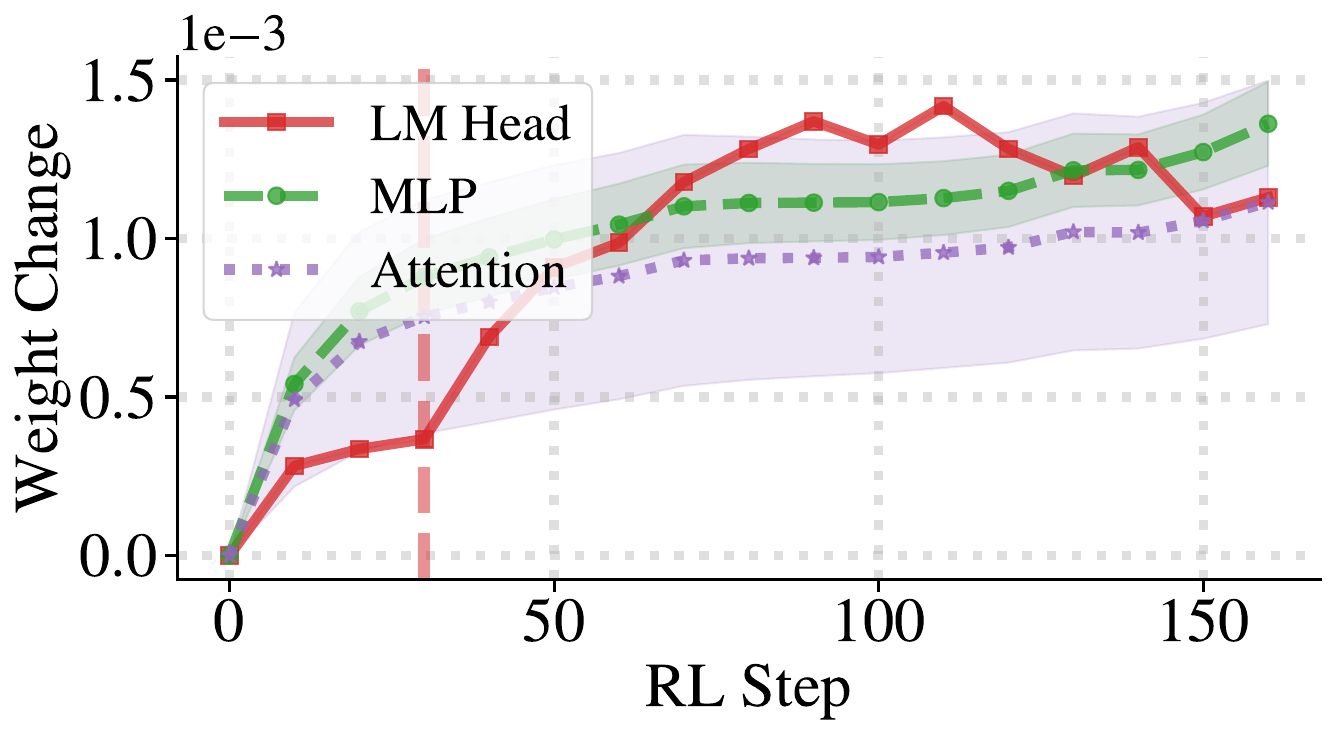}\\[-2pt]
    \end{tabular}
    \caption{\textbf{Illustration of the Disproportionate Weight Divergence (DWD) phenomenon} on math task across different LLMs. Relative weight change is defined as $\|W_t - W_{t-\Delta}\|_F / \|W_{\text{ref}}\|_F$ ($W_{\text{ref}}$: initial pretrained weight; $\Delta$: profiling interval, set to 10 RL steps); red dashed lines mark the onset of performance degradation for GRPO w/ Naive Reuse. Observation: \textit{Sample reuse accelerates early convergence, followed by a severe performance degradation synchronized with an abrupt, isolated surge in the LM Head — while intermediate layers (Attention and MLP) remain stable.}}
    \label{fig:empirical_dwd_more_llm}
\end{figure}

\newpage
\section{Stability of DGG across Random Seeds}
\label{app:seed_stability}

To verify that DGG's improvements are robust to randomness rather than 
artifacts of a specific run, we repeat the Math500 experiments on 
Qwen3-4B-Instruct and Qwen2.5-7B-Instruct with three different random 
seeds. Figure~\ref{fig:seed_stability} reports the mean accuracy with 
shaded standard deviation across seeds. DGG consistently outperforms the 
single-use baseline and avoids the collapse of Naive Reuse on every seed, 
with tight standard-deviation bands throughout training---confirming that 
DGG's gains are stable and reproducible.

\begin{figure}[htbp]
    \centering\scriptsize\renewcommand\arraystretch{0.5}
    \setlength{\tabcolsep}{10pt}
    \begin{tabular}{cc}
    \includegraphics[width=0.4\linewidth]{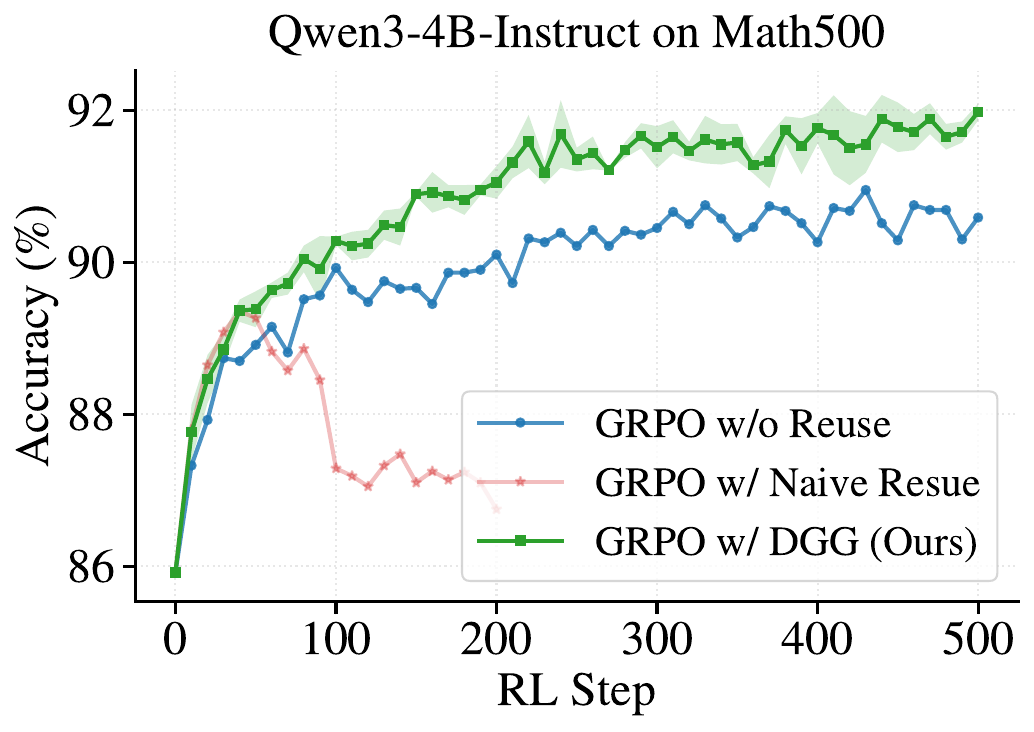}&
    \includegraphics[width=0.4\linewidth]{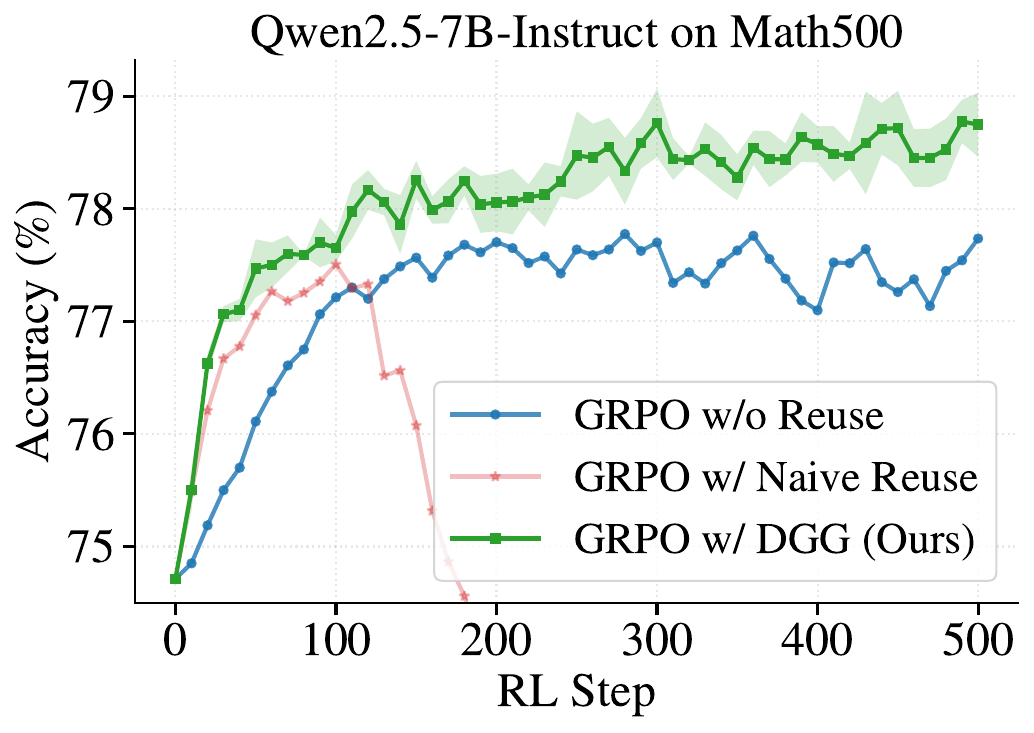}\\[-2pt]
    \end{tabular}
\caption{\textbf{Stability of DGG across three random seeds} on 
    Qwen3-4B-Instruct and Qwen2.5-7B-Instruct (Math500). Solid lines show 
    the mean accuracy and shaded regions denote one standard deviation 
    across seeds. Observation: \emph{DGG consistently outperforms the 
    single-use baseline and avoids Naive Reuse's collapse across all 
    seeds, confirming the reproducibility of its gains.}}
\label{fig:seed_stability}
\end{figure}

\section{More Experimental Details}
\label{app:more_details}
\subsection{Training Details.}
Mathematical reasoning experiments are conducted on 32 NVIDIA H800 80GB 
GPUs, and agentic experiments (ALFWorld, WebShop, and search-augmented 
QA) are conducted on 16 NVIDIA H800 80GB GPUs. All methods are configured with identical hyperparameters within each 
task. The learning rate is $1\mathrm{e}{-6}$ and the sampling temperature 
is $0.7$ throughout. For mathematical reasoning, the maximum prompt 
length is $2048$ tokens, the maximum response length is $4096$ tokens, 
the group size is $16$, the prompt batch size is 512, and the KL coefficient is $0$. For ALFWorld, 
the maximum prompt length is $8192$ tokens, the maximum response length 
is $8192$ tokens, the group size is $8$, the KL coefficient is $0.001$, the prompt batch size is 16, 
and the maximum interaction steps is $50$. For WebShop, the maximum 
prompt length is $4096$ tokens, the maximum response length is $4096$ 
tokens, the group size is $8$, the KL coefficient is $0.001$, the prompt batch size is 32, and the 
maximum interaction steps is $30$. For search-augmented QA, we use E5~\cite{wang2022text} as the retriever, following~\cite{feng2026groupingroup}. The maximum 
prompt length is $2048$ tokens, the maximum response length is $8192$ 
tokens, the group size is $8$, the KL coefficient is $0.001$, the prompt batch size is 64, and the 
maximum interaction steps is $10$.

\subsection{Implementation Details of DGG.} 
\label{app:our_details}
At each gradient step, we compute the \texttt{lm\_head} gradient energy 
$g_t = \|G^{\text{lm}}_t\|_F^2$ and its step-wise increment 
$\Delta g_t = g_t - g_{t-1}$. The running statistics $(\mu_t, \sigma_t)$ 
are maintained as the sample mean and standard deviation of 
$\Delta g_t$ over a sliding window of the most recent $W = 20$ past 
increments (excluding the current step), and the Z-score 
$z_t = (\Delta g_t - \mu_t)/(\sigma_t + \varepsilon)$ with 
$\varepsilon = 10^{-8}$ is evaluated against the threshold 
$\tau$. Gating is activated only when (i) the window is fully 
populated ($t \geq W$) and (ii) the current reuse index satisfies 
$k > 1$; under these conditions, $z_t > \tau$ triggers gradient discard 
and reuse termination as described in Section~\ref{sec:dgg}. All other 
steps follow standard GRPO. The full procedure is summarized in 
Algorithm~\ref{alg:dgg}. We set the maximum reuse $K = 4$ and select the anomaly threshold $\tau$ from the candidate set $\{0.1, 0.5, 1.0\}$.

\section{Relative Weight Change throughout Full RL Training}
\label{app:full_weight_change}

The main paper reports the relative weight change during the first 
$200$ RL steps to clearly visualize the onset of the DWD phenomenon 
under sample reuse. To rule out the possibility that the abrupt 
\texttt{lm\_head} surge is caused by the increased number of gradient updates 
rather than sample reuse itself, we extend the profiling on Math500 to 
the full training trajectory under the GRPO w/o Reuse regime. As shown 
in Figure~\ref{fig:full_weight_change}, the relative weight change of 
all components---including the \texttt{lm\_head}---grows smoothly and remains 
stable throughout the entire training process, with no abrupt surge. 
This serves as a controlled comparison to the sample-reuse setting 
(Figure~\ref{fig:empirical_dwd_qwen3-4b-instruct}): under matched gradient-update budgets 
but without sample reuse, the \texttt{lm\_head} surge does not occur, confirming 
that sample reuse---not the number of updates---is the direct cause of 
the abrupt weight divergence underlying DWD.

\begin{figure}[h]
    \centering\scriptsize\renewcommand\arraystretch{0.5}
    \setlength{\tabcolsep}{10pt}
    \begin{tabular}{cc}
    \includegraphics[width=0.42\linewidth]{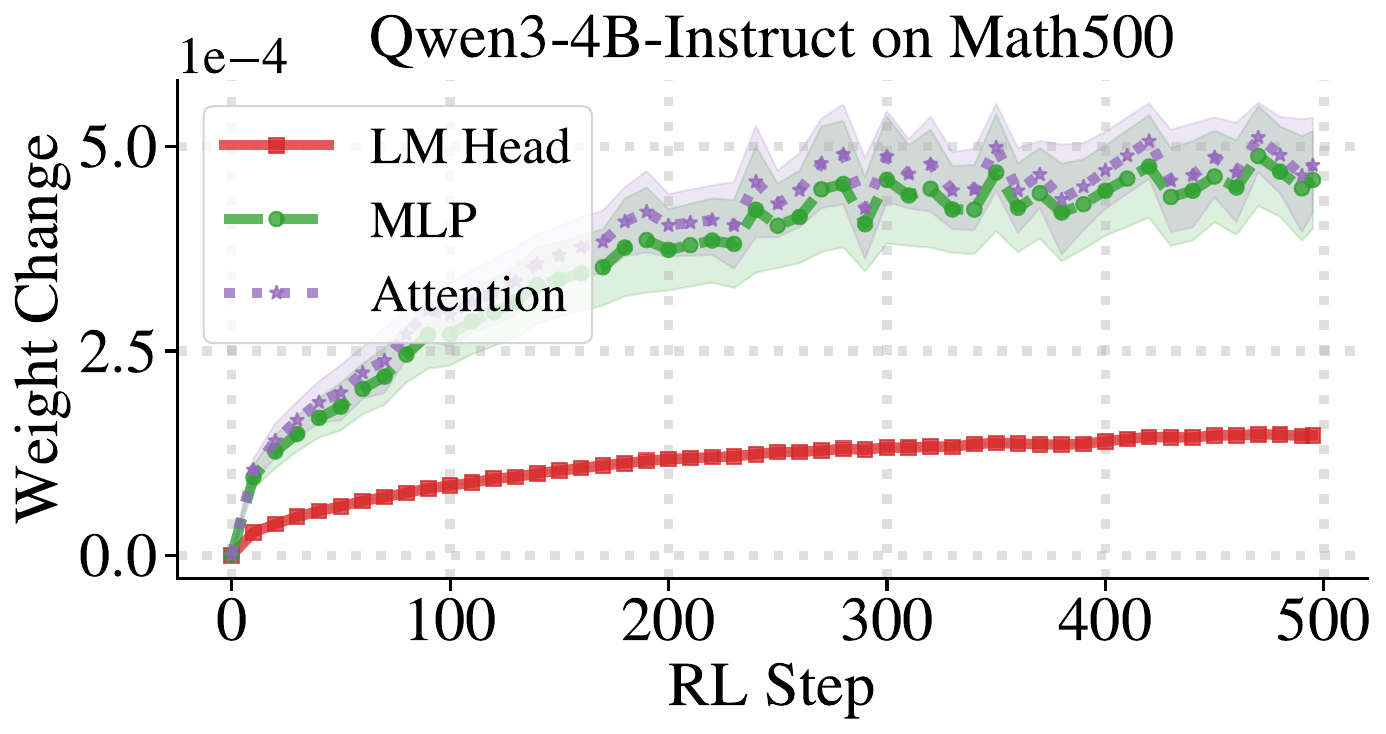}&
    \includegraphics[width=0.38\linewidth]{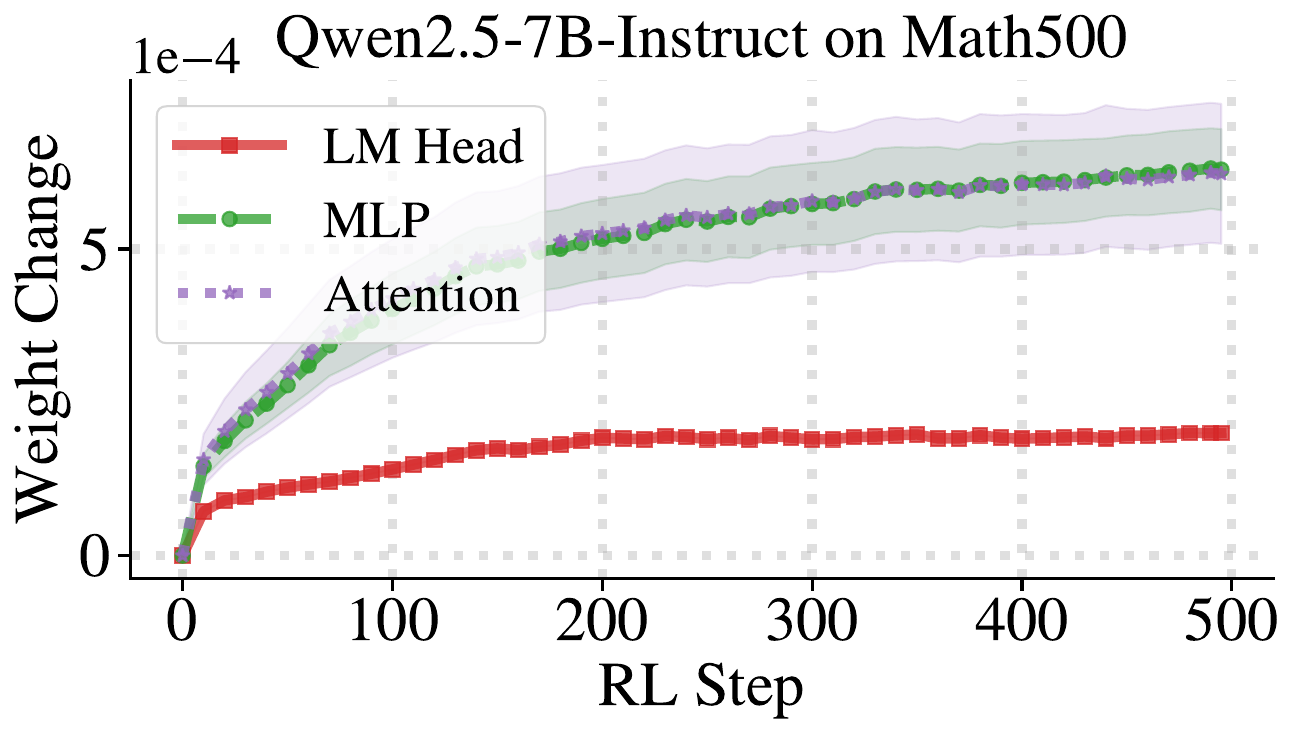}\\[-2pt]
    \end{tabular}
\caption{\textbf{Relative weight change throughout the full RL 
    training process} on Qwen3-4B-Instruct and Qwen2.5-7B-Instruct 
    (Math500), under the GRPO w/o Reuse regime. Observation: 
    \emph{Without sample reuse, the relative weight change of all 
    components---including the \texttt{lm\_head}---grows smoothly throughout 
    training, confirming that the abrupt \texttt{lm\_head} surge observed under 
    sample reuse is caused by reuse itself rather than by the increased 
    number of gradient updates.}}
\label{fig:full_weight_change}
\end{figure}


\end{document}